\documentclass[10pt,letterpaper]{article}

\usepackage{fullpage} 

\usepackage{authblk}

\usepackage[utf8]{inputenc}
\usepackage[T1]{fontenc}
\usepackage{lmodern}
\usepackage{microtype}

\usepackage{amsmath,amssymb,amsfonts,amsthm,mathtools}
\usepackage{bm}
\PassOptionsToPackage{numbers,sort&compress}{natbib}
\usepackage{natbib}

\usepackage{booktabs}
\usepackage{graphicx}
\usepackage{subcaption}
\usepackage{algorithm}
\usepackage[noend]{algpseudocode}

\usepackage{enumitem}
\usepackage{xcolor}

\usepackage{url}
\usepackage{hyperref}

\newcommand{\R}{\mathbb{R}}
\newcommand{\E}{\mathbb{E}}
\newcommand{\Var}{\mathrm{Var}}
\newcommand{\Cov}{\mathrm{Cov}}

\newcommand{\Pcal}{\mathcal{P}}

\newcommand{\Hv}{\mathcal{H}_Y}
\newcommand{\Xcal}{\mathcal{X}}
\newcommand{\Ycal}{\mathcal{Y}}
\newcommand{\Zcal}{\mathcal{Z}}

\newcommand{\1}{\mathbf{1}}

\newcommand{\dto}{\overset{d}{\to}}
\newcommand{\iid}{\stackrel{\mathrm{i.i.d.}}{\sim}}

\newcommand{\tr}{\mathrm{tr}}
\newcommand{\te}{\mathrm{te}}
\newcommand{\dr}{\mathrm{dr}}

\newcommand{\op}{\mathrm{op}}

\DeclareMathOperator*{\argmax}{arg\,max}

\newtheorem{theorem}{Theorem}[section]
\newtheorem{proposition}[theorem]{Proposition}
\newtheorem{corollary}[theorem]{Corollary}
\newtheorem{lemma}[theorem]{Lemma}

\theoremstyle{definition}

\newtheorem{remark}[theorem]{Remark}
\newtheorem{assumption}[theorem]{Assumption}

\begin{document}

\title{Semiparametric Efficient Test for Interpretable
Distributional Treatment Effects}
\author[1]{Houssam Zenati}
\author[1]{Arthur Gretton}
\affil[1]{Gatsby Computational Neuroscience Unit, University College London}
\date{}
\newcommand{\codeurl}{\url{https://github.com/houssamzenati/interpretable-dr-me-test}}

\maketitle

\begin{abstract}
Distributional treatment effects can be invisible to means: a treatment may
preserve average outcomes while changing tails, modes, dispersion, or rare-event
probabilities. Kernel tests can detect discrepancies between interventional
outcome laws, but global tests do not reveal where the laws differ. We propose
DR-ME, to our knowledge the first semiparametrically efficient finite-location
test for interpretable distributional treatment effects. DR-ME evaluates an
interventional kernel witness at learned outcome locations, returning
causal-discrepancy coordinates rather than only a global rejection. From
observational data, we derive orthogonal doubly robust kernel features whose
centered oracle form is the canonical gradient of this finite witness. For fixed
locations, we characterize the local testing limit: DR-ME is chi-square
calibrated under the null, has noncentral chi-square local power, and uses the
covariance whitening that optimizes local signal-to-noise for discrepancies
visible through the selected coordinates. This efficient local-power geometry yields a
principled location-learning criterion, with sample splitting preserving
post-selection validity. Experiments show near-nominal type-I error, competitive power against global
doubly robust kernel tests, and interpretable learned locations that localize distributional
effects in a semi-synthetic medical-imaging study.
\end{abstract}

\section{Introduction}

Average treatment effects are often too coarse for causal questions involving risk, heterogeneity, or structured outcomes. A treatment may leave the mean outcome nearly unchanged while altering dispersion, tails, multimodality, or the probability of rare but consequential events. This is the motivation behind distributional treatment-effect analysis, including nonparametric policy effects \citep{rothe2010nonparametric} and inference on interventional distributions \citep{chernozhukov2013inference}. The issue is even sharper when outcomes are images, sequences, graphs, embeddings, or high-dimensional measurements \citep{gartner2003survey}: reducing such outcomes to a small number of hand-chosen scalar summaries can obscure the effect of interest. We therefore study tests of interventional outcome distributions, rather than tests that only address average effects.

Kernel mean embeddings provide a natural language for this task. They represent probability laws as elements of a reproducing kernel Hilbert space \citep{smola2007hilbert}, so that distributional discrepancies can be studied without first specifying a scalar outcome summary. With characteristic kernels, equality of embeddings is equivalent to equality of probability laws \citep{sriperumbudur2011universality}. The associated RKHS distance is the maximum mean discrepancy, a central tool in kernel two-sample testing \citep{gretton2012kernel}. In causal problems, counterfactual mean embeddings extend this representation to potential-outcome laws \citep{muandet2021counterfactual}, and counterfactual policy mean embeddings extend it to policy-induced outcome distributions with doubly robust estimation \citep{zenati2025cpme}.

Global kernel discrepancies are powerful omnibus tools, but they are often hard to interpret. A global rejection says that two interventional outcome laws differ, but not where the difference is expressed. Recent doubly robust kernel tests make global RKHS discrepancies available for observational causal studies \citep{martinez2023,fawkes2024doubly, zenati2025kteadaptive}, but their evidence remains fundamentally global. In ordinary two-sample testing, finite-location mean-embedding tests address this limitation by evaluating a witness function at a small number of spatial or frequency locations \citep{Chwialkowski2015,jitkrittum2016interpretable}. These locations act as interpretable distributional features: they indicate where two laws are most distinguishable after accounting for sampling variability. In observational causal problems, however, the corresponding oracle contrast involves both potential outcomes, so naive treated-versus-control finite-location features are confounded and invalid.

Our work is also related to semiparametric tests beyond scalar parameters. Restricted score tests target function-valued risk minimizers by testing risk
derivatives over restricted direction classes \citep{hudson2021restricted}. Global MMD tests for unknown functions compare distributions of estimated functions and rely on higher-order pathwise differentiability \citep{luedtke2019omnibus}. Hilbert-valued one-step theory gives efficient inference for RKHS-valued parameters, including counterfactual mean embeddings \citep{luedtke2024one}. These works address different inferential objects: restricted score tests, global MMD tests, or Hilbert-valued confidence sets. Here we target a finite-location projection of an interventional kernel witness. This finite projection gives interpretable outcome-space coordinates, admits a first-order canonical gradient, and leads to an efficient finite-dimensional local testing geometry that can be used to learn where to test.

We develop DR-ME, a semiparametrically efficient finite-location test for causal
distributional testing. The locations are interpretable causal-discrepancy coordinates: each asks
whether the interventional laws differ through their kernel similarity to a selected outcome point. The challenge is that these coordinates are features of
interventional laws, not of observed treated and control samples. Under
selection on observables and positivity
\citep{rubin1974estimating,imbens2015causal,rosenbaum1983central,hernanrobins2020causal},
they are identified through propensity scores and outcome regressions, but
identification alone is not enough for nuisance-robust inference. We therefore
derive an orthogonal augmented inverse-propensity feature. For fixed locations,
it has the usual doubly robust structure
\citep{robins1994estimation,hahn1998role,bang2005doubly}; with cross-fitting,
its empirical mean has a first-order expansion whose leading term is the
canonical gradient of the finite-location signal, while nuisance errors enter
through second-order remainders \citep{chernozhukov2018double}.

One of our main contributions is a semiparametric local-efficiency theory for
this finite-location causal test. For fixed locations, we study
quadratic-mean-differentiable local submodels through the finite-location null,
following the local asymptotic testing perspective of
\citet{neyman1933mostefficient}, \citet{lecam2000asymptotics}, and
\citet{vandervaart1998asymptotic}. The canonical-gradient statistic attains the
efficient finite-signal Gaussian experiment. Consequently, the Hotelling
statistic \citep{hotelling1931generalization} has a chi-square null limit and a
noncentral chi-square local limit, with power governed by the interventional
drift whitened by the canonical-gradient covariance. This whitening is not an
accidental consequence of plugging a doubly robust score into a conventional
Hotelling statistic; it is the covariance geometry of the efficient finite-signal
experiment.

This geometry also determines how locations should be learned. The relevant
objective is not the raw witness magnitude, but its magnitude after whitening by
the canonical-gradient covariance. We therefore learn locations by maximizing a
ridge-stabilized empirical local-power criterion. Locations are selected on an
auxiliary split and tested on an independent final split; conditional on the
learning split, the selected locations are fixed, so the fixed-location null
theory applies after selection
\citep{fithian2014optimal,kuchibhotla2022postselection}. Experiments on synthetic
and semi-synthetic structured outcomes show calibrated type-I error, competitive
power relative to global doubly robust kernel tests
\citep{martinez2023,fawkes2024doubly,zenati2025cpme}, and learned locations that
localize distributional shifts.

Our contributions are fourfold. First, we formulate interpretable finite-location
testing for distributional treatment effects by evaluating the interventional
RKHS witness at selected outcome locations. Second, we derive the canonical
gradient of this finite signal; its augmented inverse-propensity form yields an
orthogonal doubly robust feature, whose cross-fitted version defines the DR-ME
Hotelling statistic. We show that DR-ME attains the efficient finite-signal
Gaussian experiment, giving chi-square calibration, noncentral local power, and
canonical-gradient covariance whitening. Third, we use this geometry to learn
locations and prove uniform consistency of the empirical criterion. Fourth, we
validate calibration, power, and covariance-whitened location learning in
simulations, and use a semi-synthetic medical-imaging study to show that learned
locations can localize distributional causal effects nearly invisible to mean
contrasts.

Section~\ref{sec:finite_scores} defines the finite-location causal witness and derives the doubly robust observed-data feature whose centered oracle form is the canonical gradient. Section~\ref{sec:theory} develops the fixed-location local testing theory, including chi-square calibration, noncentral chi-square local power, and the efficient covariance geometry underlying the Hotelling statistic. Section~\ref{sec:learning} uses this geometry to construct the location-learning criterion and proves uniform consistency of its empirical version under sample splitting. Section~\ref{sec:experiments} evaluates calibration, power, covariance-whitened location learning, and image-space interpretability. 

\section{Finite-location causal witnesses and observed-data scores}
\label{sec:finite_scores}

We use the potential-outcomes notation \citep{rubin1974estimating,imbens2015causal}. Let \(A\in\{0,1\}\) be a binary treatment, let \(Y(0),Y(1)\in\Ycal\) be potential outcomes, and let \(k_Y\) be a bounded positive definite kernel on \(\Ycal\), with RKHS \(\Hv\) and feature map \(\varphi_Y(y)=k_Y(\cdot,y)\). For each arm \(a\), define the interventional mean embedding
\(\chi(a):=\E[\varphi_Y(Y(a))]\), following counterfactual and policy mean embedding constructions \citep{muandet2021counterfactual,zenati2025cpme}. The interventional embedding difference is \(\Delta:=\chi(1)-\chi(0)\), with witness function
\begin{equation} \label{eq:witness}
 w(y)
:=
\langle \Delta,\varphi_Y(y)\rangle_{\Hv}
=
\E[k_Y(y,Y(1))]-\E[k_Y(y,Y(0))].   
\end{equation}
Thus \(w(y)\) is the kernel discrepancy between the two interventional outcome laws near location \(y\). Moreover, \(\|\Delta\|_{\Hv}^2\) is the squared maximum mean discrepancy between \(P_{Y(1)}\) and \(P_{Y(0)}\) \citep{gretton2012kernel}; if \(k_Y\) is characteristic, then \(\Delta=0\) is equivalent to equality of the two interventional outcome laws \citep{sriperumbudur2011universality}.

To obtain an interpretable finite signal, fix \(J\ge1\) outcome locations \(V=(v_1,\ldots,v_J)\in\Ycal^J\), and write \(k_V(y):=(k_Y(v_1,y),\ldots,k_Y(v_J,y))^\top\in\R^J\). The finite-location causal witness is
\begin{equation} \label{eq:causal_witness}
 \mu_V
:=
\bigl(w(v_1),\ldots,w(v_J)\bigr)^\top
=
\E[k_V(Y(1))]-\E[k_V(Y(0))]
\in\R^J .   
\end{equation}
This is a finite-dimensional projection of the RKHS discrepancy \(\Delta\), with coordinates indexed by interpretable outcome locations. The global null \(H_0:\Delta=0\) implies the finite-location null \(H_{0,V}:\mu_V=0\) for every \(V\). Conversely, a fixed finite \(V\) need not characterize every global alternative. The procedure is therefore calibrated under the global causal null, with power against alternatives visible through the selected witness coordinates. This follows the finite-location mean-embedding testing principle of \citet{Chwialkowski2015} and \citet{jitkrittum2016interpretable}, but for interventional rather than ordinary two-sample distributions.

If samples from \(P_{Y(1)}\) and \(P_{Y(0)}\) were directly available, the problem would reduce to a \(J\)-dimensional mean test based on the oracle contrasts \(k_V(Y^{(1)})-k_V(Y^{(0)})\). In observational data, these contrasts are unavailable because each unit reveals only one potential outcome. We therefore construct an observed-data pseudo-feature whose mean is \(\mu_V\) and whose centered form is the canonical gradient of this finite-location signal.

\paragraph{Identification.}
Let \(Z=(X,A,Y)\sim P_0\). We assume consistency, \(Y=Y(a)\) when \(A=a\); conditional exchangeability, \(Y(a)\perp A\mid X\); and positivity, \(\pi_0(a\mid X):=P_0(A=a\mid X)\ge \varepsilon>0\) almost surely, for \(a\in\{0,1\}\) \citep{rosenbaum1983central,hernanrobins2020causal}. For fixed \(V\), define
\begin{equation} \label{eq:cme_nuisance}
 m_a(x;V):=\E[k_V(Y)\mid A=a,X=x],
\qquad
\mu_{a,V}:=\E[k_V(Y(a))].   
\end{equation}
Then \(\mu_{a,V}=\E[m_a(X;V)]\), and hence $\mu_V=\E\bigl[m_1(X;V)-m_0(X;V)\bigr]$. Thus \(\mu_V\) is identified by a regression contrast, but the feasible plug-in
contrast is not orthogonal to nuisance error: regression estimation error induces
first-order bias. We therefore use an augmented score whose moment
is locally insensitive to nuisance perturbations, so that first-stage errors enter
only through higher-order remainders \citep{chernozhukov2018double}.

\paragraph{Orthogonal observed-data feature.}
Let \(\eta=(\pi,r_0,r_1)\) be a nuisance tuple, where \(\pi(a\mid x)>0\) is a propensity model and \(r_a(x;V)\in\R^J\) is a candidate regression for \(m_a(x;V)\). Define the arm-specific augmented feature
\[
\phi_V^a(Z;\eta)
:=
\frac{\1\{A=a\}}{\pi(a\mid X)}
\bigl(k_V(Y)-r_a(X;V)\bigr)
+
r_a(X;V),
\qquad a\in\{0,1\},
\]
and the doubly robust contrast \(z_V^{\dr}(Z;\eta):=\phi_V^1(Z;\eta)-\phi_V^0(Z;\eta)\). This is the vector-valued analogue of augmented inverse-propensity scores  \citep{robins1994estimation,hahn1998role,bang2005doubly}.

\begin{proposition}[Doubly robust feature and canonical gradient]
\label{prop:dr_canonical_gradient}
Assume the identification conditions above. For any nuisance tuple \(\eta=(\pi,r_0,r_1)\) with \(\pi(a\mid X)>0\) almost surely,
\[
\E\{\phi_V^a(Z;\eta)\}-\mu_{a,V}
=
\E\left[
\frac{\pi_0(a\mid X)-\pi(a\mid X)}{\pi(a\mid X)}
\{m_a(X;V)-r_a(X;V)\}
\right].
\]
Consequently, \(\E\{\phi_V^a(Z;\eta)\}=\mu_{a,V}\) if either \(\pi(a\mid\cdot)=\pi_0(a\mid\cdot)\) or \(r_a(\cdot;V)=m_a(\cdot;V)\). Let \(\eta_0=(\pi_0,m_0,m_1)\). Then \(\E\{z_V^{\dr}(Z;\eta_0)\}=\mu_V\), and
\[
\psi_V(Z):=z_V^{\dr}(Z;\eta_0)-\mu_V
\]
is the canonical gradient of \(\mu_V\) in the nonparametric model.
\end{proposition}


For the full RKHS-valued counterfactual mean embedding, the corresponding Hilbert-valued pathwise differentiability and efficient influence function are established by \citet{luedtke2024one}; here we use the finite-dimensional projection induced by fixed locations. The proof is given in Appendix~\ref{app:finite_scores_proofs}. Proposition~\ref{prop:dr_canonical_gradient} has two roles. First, its bias identity gives the orthogonality mechanism: the target-moment error is a product of propensity and outcome-regression errors. Second, the centered augmented contrast is the canonical gradient of the finite witness, which is the link to local testing efficiency. The statistic below is built from this canonical-gradient representation.
\paragraph{DR-ME statistic.}
Let \(Z_1,\ldots,Z_n\iid P_0\). For each \(i\), let \(\hat\eta_{-i}\) be a nuisance estimate trained on data independent of \(Z_i\); this notation covers both \(K\)-fold cross-fitting and the independent nuisance split used after location learning. Define
\[
\hat z_{i,V}^{\dr}:=z_V^{\dr}(Z_i;\hat\eta_{-i}),
\qquad
\bar z_{n,V}^{\dr}:=\frac1n\sum_{i=1}^n \hat z_{i,V}^{\dr},
\qquad
S_{n,V}^{\dr}
:=
\frac{1}{n-1}\sum_{i=1}^n
(\hat z_{i,V}^{\dr}-\bar z_{n,V}^{\dr})
(\hat z_{i,V}^{\dr}-\bar z_{n,V}^{\dr})^\top .
\]
The fixed-location doubly robust mean-embedding statistic, abbreviated DR-ME, is
\begin{equation}
  \hat\lambda_{n,V}^{\dr}
:=
n\,\bar z_{n,V}^{\dr\top}
(S_{n,V}^{\dr}+\gamma_n I_J)^{-1}
\bar z_{n,V}^{\dr},
\qquad
\gamma_n\downarrow0.  
\end{equation}
This is the observed-data analogue of the finite-location Hotelling statistic.

For the following result, assume cross-fitted or independent-split nuisances, estimated propensities bounded away from zero, and fitted regressions uniformly bounded with probability tending to one. In \(K\)-fold notation, set \(\alpha_{a,k}:=\|\hat\pi_k(a\mid\cdot)-\pi_0(a\mid\cdot)\|_{L_2(P_X)}\) and \(\beta_{a,k}(V):=\|\hat m_{a,k}(\cdot;V)-m_a(\cdot;V)\|_{L_2(P_X;\R^J)}\). We assume \(\sum_{a=0}^1\{\alpha_{a,k}+\beta_{a,k}(V)\}=o_p(1)\) and \(\sum_{a=0}^1\alpha_{a,k}\beta_{a,k}(V)=o_p(n^{-1/2})\). These are the standard consistency and product-rate conditions for orthogonal cross-fitted inference \citep{chernozhukov2018double,kennedy2022semiparametric}.

\begin{theorem}[Fixed-location first-order representation]
\label{thm:first_order}
Fix \(V\in\Ycal^J\), and let \(\Sigma_V:=\Var_{P_0}\{\psi_V(Z)\}\). Under the bounded-kernel, identification, and nuisance conditions above,
\[
\sqrt n\bigl(\bar z_{n,V}^{\dr}-\mu_V\bigr)
=
\frac1{\sqrt n}\sum_{i=1}^n \psi_V(Z_i)+o_p(1),
\qquad
S_{n,V}^{\dr}\to_p\Sigma_V.
\]
If, in addition, \(\Sigma_V\) is positive definite and \(\gamma_n\downarrow0\), then
\((S_{n,V}^{\dr}+\gamma_n I_J)^{-1}\to_p\Sigma_V^{-1}\).
\end{theorem}
The proof is given in Appendix~\ref{app:finite_scores_proofs}. The theorem reduces the feasible fixed-location statistic to the empirical mean of the canonical gradient. Thus \(\Sigma_V\) is not merely a sample covariance used for numerical normalization; it is the observed-data covariance that enters the efficient local testing experiment developed next.

\section{Efficient local testing geometry}
\label{sec:theory}

We now study the fixed-location test
\(H_{0,V}:\mu_V=0\) against \(H_{1,V}:\mu_V\neq0\), in the local regime where
alternatives approach the null at \(n^{-1/2}\) rate. This is the regime in
which asymptotic power is nondegenerate: fixed alternatives are detected with
probability tending to one, while faster local alternatives are invisible to
regular tests. Appendix~\ref{app:background} recalls the local-asymptotic
testing tools used below, including local asymptotic normality (LAN),
contiguity, regular procedures, and the efficiency interpretation.

Let \(\Pcal\) be the nonparametric model for the observed-data law on
\(\Zcal=\Xcal\times\{0,1\}\times\Ycal\), and fix \(P_0\in\Pcal\) with
\(\mu_V(P_0)=0\). Let \(T_{P_0}\) denote the tangent space at \(P_0\)
\citep{bickel1993efficient,vandervaart1998asymptotic}. For a score direction
\(g\in T_{P_0}\), let \(t\mapsto P_{t,g}\) be a quadratic-mean differentiable
regular path through \(P_0\) with score \(g\), and define the contiguous local
alternatives
\[
P_{n,h,g}:=P_{h/\sqrt n,g}^{\otimes n},
\qquad h\in\R .
\]
Quadratic-mean differentiability is the smoothness condition on this local
submodel: the square-root likelihood is differentiable in \(L_2(P_0)\) with
derivative \(g\). For i.i.d. experiments, this implies the LAN expansion and
contiguity of \(P_{n,h,g}\) with respect to \(P_0^{\otimes n}\)
\citep[Theorem~7.2]{vandervaart1998asymptotic}; no other local likelihood
regularity is used below. We use the same path to measure how the finite witness
moves under a local alternative. Since Proposition~\ref{prop:dr_canonical_gradient}
identifies \(\psi_V\) as the canonical gradient of \(\mu_V\), define the local
drift
\begin{equation}
    \eta_V(g)
:=
\left.\frac{d}{dt}\right|_{t=0}\mu_V(P_{t,g})
=
\E_{P_0}\{\psi_V(Z)g(Z)\}.
\end{equation}
Its covariance-whitened magnitude is
\begin{equation}
 \lambda_V(g)
:=
\eta_V(g)^\top\Sigma_V^{-1}\eta_V(g),
\qquad
\Sigma_V:=\Var_{P_0}\{\psi_V(Z)\}.   
\end{equation}
Thus \(\eta_V(g)\) records how the selected witness coordinates move under the
local alternative, while \(\lambda_V(g)\) measures the size of that movement
after normalization by the canonical-gradient covariance. If \(\eta_V(g)=0\),
the selected locations are locally blind to \(g\); if \(\eta_V(g)\neq0\), the
local alternative produces a first-order shift in the finite witness coordinates.

\begin{theorem}[Efficient local testing geometry]
\label{thm:efficient_local_testing}
Fix \(V\in\Ycal^J\) and let \(P_0\in\Pcal\) satisfy \(\mu_V(P_0)=0\). Let
\(t\mapsto P_{t,g}\) be a quadratic-mean differentiable regular path through
\(P_0\) with score \(g\in T_{P_0}\). Assume \(\Sigma_V\succ0\),
\(\gamma_n\downarrow0\), and that the first-order representation and covariance
consistency in Theorem~\ref{thm:first_order} hold under \(P_0^{\otimes n}\).
Then, under \(P_{n,h,g}\),
\[
\sqrt n\,\bar z_{n,V}^{\dr}
\dto
N\bigl(h\eta_V(g),\Sigma_V\bigr),
\qquad
\hat\lambda_{n,V}^{\dr}
\dto
\chi^2_J\!\left(h^2\lambda_V(g)\right).
\]
\end{theorem}

The proof is given in Appendix~\ref{app:local_proofs}. The theorem identifies
the finite-signal Gaussian experiment induced by the finite-location causal
null. Since \(\psi_V\) is the canonical gradient, the H\'ajek--Le Cam
convolution theorem \citep[Theorem~25.20]{vandervaart1998asymptotic} implies
that any regular estimator \(T_n\) of \(\mu_V\) with limit law
\(\sqrt n\{T_n-\mu_V(P_0)\}\dto L\) has \(L=N(0,\Sigma_V)*M\) for some noise
law \(M\), while \(\bar z_{n,V}^{\dr}\) attains the no-extra-noise case
\(M=\delta_0\). Thus \(\Sigma_V\) is the efficient covariance of the
finite-signal limit experiment. For any contrast \(a^\top\mu_V\), the squared
local signal-to-noise ratio along \(g\) is
\((a^\top\eta_V(g))^2/(a^\top\Sigma_Va)\), whose supremum over \(a\neq0\) is
\(\lambda_V(g)\). The Hotelling statistic is the corresponding omnibus
quadratic test in the whitened Gaussian shift experiment, with local power
governed by \(h^2\lambda_V(g)\). Appendix~\ref{app:local_results} gives additional
details on this efficiency interpretation.

\begin{remark}[Position relative to global MMD tests]
\label{rem:testing_efficiency_positioning}
The efficiency statement in Theorem~\ref{thm:efficient_local_testing} is a
local testing statement for the regular finite signal \(\mu_V\). It is not a
claim of optimality over all distributional alternatives. This differs from
global MMD-type tests, which target a squared RKHS discrepancy such as
\(\|\Delta(P)\|_{\Hv}^2\). At the global null \(\Delta(P_0)=0\), the
first-order derivative of this squared norm vanishes, so the ordinary Gaussian
Wald geometry is replaced by a second-order, degenerate \(U\)-statistic or
Gaussian-chaos null theory, as in the higher-order pathwise differentiability
analysis of \citet{luedtke2019omnibus}. Recent doubly robust kernel tests for
causal distributional effects similarly target global RKHS discrepancies and
provide omnibus validity and power guarantees
\citep{martinez2023,fawkes2024doubly,zenati2025cpme}. Our construction instead
targets the finite projection \(\mu_V=L_V\Delta\), which remains first-order
regular under the null. The tradeoff is explicit: a fixed finite \(V\) need not
characterize every global alternative, but it yields interpretable witness
coordinates, standard \(\chi^2_J\) calibration after sample splitting, and a
covariance-whitened local-power criterion for learning where to test.
\end{remark}

\begin{remark}[Directional and omnibus benchmarks within the finite signal]
\label{rem:directional_omnibus}
Within the finite-signal Gaussian experiment, if the local direction \(g\) were
known, the Neyman--Pearson linear test projects along
\(\Sigma_V^{-1}\eta_V(g)\). Equivalently, for scalar contrasts \(a^\top\mu_V\),
the squared local signal-to-noise ratio is
\[
\frac{(a^\top\eta_V(g))^2}{a^\top\Sigma_Va},
\]
whose supremum over \(a\neq0\) is
\(\eta_V(g)^\top\Sigma_V^{-1}\eta_V(g)\). When \(g\) is unknown, no uniformly
most powerful test exists over all drift directions in the multivariate
Gaussian shift experiment. The Hotelling statistic is the standard omnibus
quadratic statistic in the whitened finite-signal experiment.
\end{remark}

The noncentrality in Theorem~\ref{thm:efficient_local_testing} also explains how
locations should be chosen. The relevant signal is not the raw Euclidean size of
\(\mu_V\), but its size after whitening by the canonical-gradient covariance.
\begin{proposition}[Local-power criterion for locations]
\label{cor:local_power_criterion}
For fixed \(V\), define
\[
\Lambda_V(P)
:=
\mu_V(P)^\top\Sigma_V(P)^{-1}\mu_V(P),
\]
where \(\Sigma_V(P)\) is the covariance of the canonical gradient of \(\mu_V\)
at \(P\). Suppose \(\Sigma_V(P)\) is nonsingular in a neighborhood of \(P_0\)
and continuous along \(t\mapsto P_{t,g}\). Then, along \(P_{h/\sqrt n,g}\),
\[
\mu_V(P_{h/\sqrt n,g})
=
\frac{h}{\sqrt n}\eta_V(g)+o(n^{-1/2}),
\qquad
n\,\Lambda_V(P_{h/\sqrt n,g})
\to
h^2\lambda_V(g).
\]
\end{proposition}
The proof is given in Appendix~\ref{app:local_proofs}. Proposition~\ref{cor:local_power_criterion}
links the population criterion for choosing \(V\) to the noncentrality parameter
in Theorem~\ref{thm:efficient_local_testing}. It also explains why an
unwhitened witness norm is not the right learning objective: a large discrepancy
in a high-variance direction may be weak for testing, while a smaller
discrepancy in a low-noise direction may be more informative.

The unregularized quantity \(\Lambda_V(P)\) is the asymptotic local-power
object. In practice, our location learning uses a ridge-stabilized proxy,
\begin{equation}
   \mathcal P_\tau(V)
:=
\mu_V^\top(\Sigma_V+\tau I_J)^{-1}\mu_V,
\qquad \tau>0, 
\end{equation}
which is numerically stable and approaches \(\Lambda_V(P)\) as
\(\tau\downarrow0\) whenever the covariance is uniformly nonsingular.

Finally, the fixed-location theory gives a direct route to valid testing when
the locations are learned on data independent of the final test split.

\begin{corollary}[Post-selection null law for split-sample locations]
\label{cor:split_calibration}
Assume the global null \(H_0:\Delta=0\). Let \(\hat V\in\mathcal V\) be
selected using data independent of the final testing split, and compute
\(\hat\lambda_{n_{\te},\hat V}^{\dr}\) on that final split with
\(\gamma_{n_{\te}}\downarrow0\), using nuisance estimates independent of each
test observation. Suppose that, conditionally on the learning split, the
fixed-location first-order representation and covariance consistency of
Theorem~\ref{thm:first_order} hold at \(V=\hat V\), with nondegenerate
covariance; alternatively, suppose these conditions hold uniformly over
\(V\in\mathcal V\). Then, conditionally on the learning split,
\[
\hat\lambda_{n_{\te},\hat V}^{\dr}
\dto
\chi^2_J .
\]
\end{corollary}
Corollary~\ref{cor:split_calibration} defines the split-sample DR-ME test.
After learning \(\hat V\) on data independent of \(I_{\te}\), compute
\[
p_{\te}^{\dr}
:=
1-F_{\chi^2_J}\!\left(\hat\lambda_{n_{\te},\hat V}^{\dr}\right),
\]
and reject the global null \(H_0:\Delta=0\) at level \(\alpha\) when
\[
\hat\lambda_{n_{\te},\hat V}^{\dr}>\chi^2_{J,1-\alpha},
\qquad\text{equivalently}\qquad
p_{\te}^{\dr}\le\alpha .
\]
The reason this remains valid after learning is simple: \(H_0:\Delta=0\)
implies \(\mu_V=0\) for every \(V\). Conditional on the learning split,
\(\hat V\) is fixed relative to the final test data, so the fixed-location
\(\chi^2_J\) calibration applies at \(V=\hat V\). For a prespecified \(V\), the
same rule with \(\hat V\) replaced by \(V\) gives the fixed-location DR-ME test
of \(H_{0,V}:\mu_V=0\), whose local power is governed by
Theorem~\ref{thm:efficient_local_testing}.

\section{Learning locations by local power}
\label{sec:learning}

Theorem~\ref{thm:efficient_local_testing} shows that fixed-location local power is governed by the witness after whitening by the efficient observed-data covariance. We therefore learn locations by maximizing
\[
\mathcal P_\tau(V)
:=
\mu_V^\top(\Sigma_V+\tau I_J)^{-1}\mu_V,
\qquad \tau>0,
\]
a ridge-stabilized version of the local-power criterion in Corollary~\ref{cor:local_power_criterion}. This follows the ME-test feature-learning principle \citep{jitkrittum2016interpretable}, but uses the causal observed-data covariance geometry.

We split the sample as \(I_\eta\cup I_{\tr}\cup I_{\te}\). Nuisances \(\hat\eta\) are fitted on \(I_\eta\), locations are learned on \(I_{\tr}\), and the final DR-ME test is computed on \(I_{\te}\). For \(i\in I_{\tr}\), let \(\hat z_{i,V}^{\dr}:=z_V^{\dr}(Z_i;\hat\eta)\), with empirical mean and covariance \(\bar z_{\tr,V}^{\dr}\) and \(S_{\tr,V}^{\dr}\). The training objective is
\begin{equation}
   \hat{\mathcal P}_{\tau,\tr}(V)
:=
\bar z_{\tr,V}^{\dr\top}
(S_{\tr,V}^{\dr}+\tau I_J)^{-1}
\bar z_{\tr,V}^{\dr}. 
\end{equation}

\begin{figure}[t]
\centering
\includegraphics[width=\linewidth]{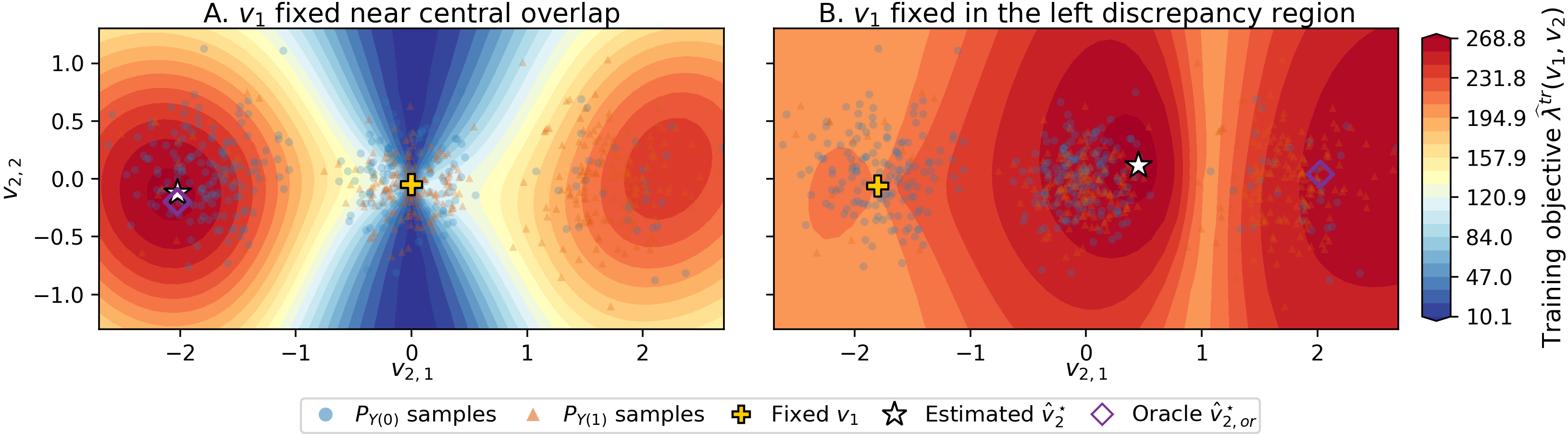}
\caption{Training objective for learned causal locations. Each panel fixes \(v_1\) and plots \(v_2\mapsto \hat{\mathcal P}_{\tau,\tr}(v_1,v_2)\). Blue/orange points show samples from \(P_{Y(0)}\)/\(P_{Y(1)}\); the gold marker is \(v_1\), the white star is the feasible maximizer, and the purple diamond is the oracle-nuisance maximizer on the same grid.}
\label{fig:contour_estimated_with_oracle_marker}
\end{figure}

Figure~\ref{fig:contour_estimated_with_oracle_marker} illustrates this objective in a representative two-location example. For the learned output \(\hat V\), define
\[
\varepsilon_{\tr}(\hat V)
:=
\sup_{V\in\mathcal V}\hat{\mathcal P}_{\tau,\tr}(V)
-
\hat{\mathcal P}_{\tau,\tr}(\hat V).
\]
This separates statistical error from numerical optimization error. It is zero for exhaustive finite-dictionary search, which is a natural set for non-Euclidean structured outcomes such as DNA or protein sequences \citep{scholkopf2004kernel,vert2005kernels,vert2006classification}. In Euclidean classes, the assumptions below make \(V\mapsto\hat{\mathcal P}_{\tau,\tr}(V)\) Lipschitz, so a \(\delta\)-net gives \(\varepsilon_{\tr}(\hat V)\le \operatorname{Lip}(\hat{\mathcal P}_{\tau,\tr})\delta\); certified Lipschitz or branch-and-bound methods can also control the gap \citep{piyavskii1972,shubert1972,horst1996global}. We keep \(\varepsilon_{\tr}(\hat V)\) explicit because plain gradient ascent generally certifies stationarity, not global optimality \citep{nocedal2006numerical}.

Conditional on \((I_\eta,I_{\tr})\), the selected \(\hat V\) is fixed relative to \(I_{\te}\). Hence the final statistic computed on \(I_{\te}\) is a fixed-location statistic at \(V=\hat V\), and Corollary~\ref{cor:split_calibration} gives post-selection calibration under the global null. The learning step affects power, not null calibration.

\begin{theorem}[Uniform consistency of the learned criterion]
\label{thm:location_learning}
Assume \(\Ycal\subset\R^d\) and let \(\mathcal V\subset[-R,R]^{Jd}\) be compact. Let
\[
\Pi(\mathcal V)
:=
\{v\in\R^d:\ v=v_j \text{ for some } V=(v_1,\ldots,v_J)\in\mathcal V\}.
\]
Suppose \(k_Y(v,y)\) is uniformly bounded and Lipschitz in \(v\) over \(\Pi(\mathcal V)\), the estimated propensity is bounded away from zero with probability tending to one, and the fitted regressions \(\hat m_a(x;v)\) are uniformly bounded and Lipschitz in \(v\). Define
$\rho_{n_\eta}
:=
\sum_{a=0}^1
\left[
\|\hat\pi(a\mid\cdot)-\pi_0(a\mid\cdot)\|_{L_2(P_X)}
+
\sup_{v\in\Pi(\mathcal V)}
\|\hat m_a(\cdot;v)-m_a(\cdot;v)\|_{L_2(P_X)}
\right].$
If \(n_\eta,n_{\tr}\to\infty\) and \(\rho_{n_\eta}=o_p(1)\), then, for fixed \(J\),
\[
\Delta_{\tr}
:=
\sup_{V\in\mathcal V}
\left|
\hat{\mathcal P}_{\tau,\tr}(V)-\mathcal P_\tau(V)
\right|
=
O_p\!\left(
\sqrt{\frac{Jd\log n_{\tr}}{n_{\tr}}}
+
\rho_{n_\eta}
\right).
\]
Consequently, for any output \(\hat V\) for which \(\varepsilon_{\tr}(\hat V)=O_p(e_{\tr})\) and any
\(V_\tau^\star\in\argmax_{V\in\mathcal V}\mathcal P_\tau(V)\),
\[
\mathcal P_\tau(V_\tau^\star)-\mathcal P_\tau(\hat V)
=
O_p\!\left(
\sqrt{\frac{Jd\log n_{\tr}}{n_{\tr}}}
+
\rho_{n_\eta}
+
e_{\tr}
\right).
\]
\end{theorem}

The proof is given in Appendix~\ref{app:learning_proofs}. The first term is the statistical price of optimizing over a \(Jd\)-dimensional Euclidean class, the second is the nuisance-induced discrepancy, and the third is the empirical optimization gap. The theorem gives near-optimality within the chosen search class and for the ridge-stabilized criterion; it does not claim recovery of globally optimal locations over all of \(\Ycal^J\).

For structured outcomes where Euclidean optimization is inappropriate, Appendix~\ref{app:finite_dictionary} gives a finite-dictionary analogue. If \(\mathcal C=\{c_1,\ldots,c_M\}\subset\Ycal\) and \(\mathcal V_J(\mathcal C)\subset\mathcal C^J\), the Euclidean complexity term is replaced by $\sqrt{\frac{\log|\mathcal V_J(\mathcal C)|}{n_{\tr}}}
\le
\sqrt{\frac{J\log M}{n_{\tr}}}$. With exhaustive dictionary search, \(\varepsilon_{\tr}(\hat V)=0\). Matrix formulas and implementation details are given in Appendix~\ref{app:implementation}.

\section{Experiments}
\label{sec:experiments}

We evaluate three empirical claims in the main text. First, the proposed learned finite-location test is calibrated under observational confounding and is competitive with global doubly robust kernel tests, with higher power on localized distributional alternatives. Second, the efficient covariance geometry is important for learning informative locations, especially in structured or high-dimensional outcome spaces. Third, we illustrate the interpretability of the learned locations on a semi-synthetic medical imaging task based on MedMNIST \citep{medmnistv2}. Additional baselines, local noncentral-\(\chi^2\) diagnostics, runtime results are deferred to Appendix~\ref{app:additional_experiments} and further validate the practical benefits of our approach.

\begin{figure}[t]
\centering
\includegraphics[width=\linewidth]{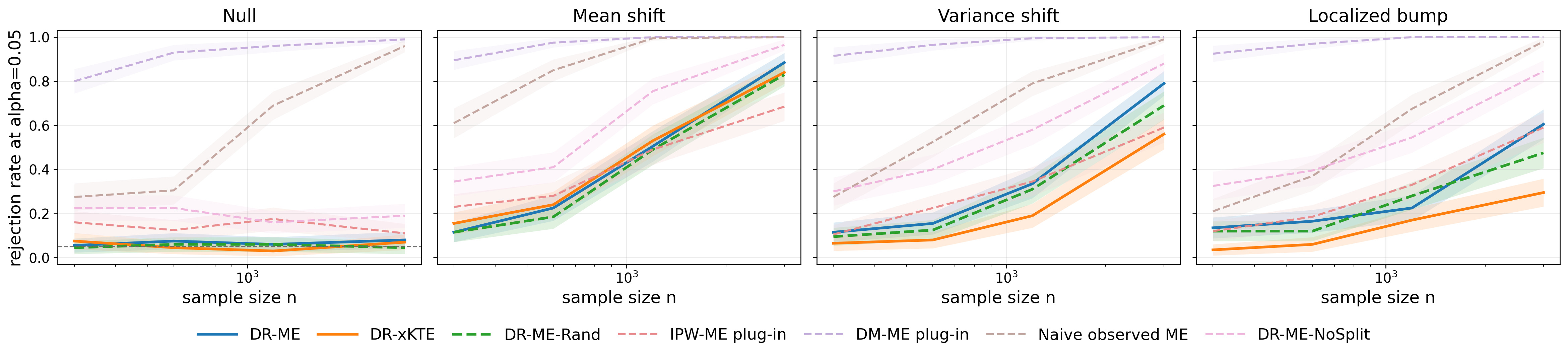}
\caption{
Calibration and power under observational confounding. Left: confounded null \(P_{Y(0)}=P_{Y(1)}\). Right: mean-shift, variance-shift, and localized-bump alternatives. DR-ME is the most powerful test among calibrated ones (DR-ME, DR-xKTE, DR-ME-Rand) and even outperforms the global test DR-xKTE. DR-ME-Rand is calibrated but weaker, showing the benefits of learned locations. IPW-ME, DM-ME, Naive observed ME, and DR-ME-NoSplit over-reject under the null, so their high alternative rejection rates are not valid power.
}
\label{fig:power_all_baselines_main}
\end{figure}

All synthetic experiments use a confounded observational design with nonlinear outcome structure, three independent splits for nuisance fitting, location learning, and final testing, and \(200\) Monte Carlo replications at nominal level \(\alpha=0.05\). Nuisances are estimated from \(I_\eta\), locations are learned from \(I_{\tr}\), and all reported tests are evaluated only on \(I_{\te}\). Unless stated otherwise, DR-ME uses a Gaussian outcome kernel, \(M=80\) candidate locations, and \(J=2\) selected locations. Full data-generating processes, nuisance models, bandwidth choices, and implementation details are given in Appendix~\ref{app:additional_experiments}.

\paragraph{Calibration and power.}
Figure~\ref{fig:power_all_baselines_main} compares DR-ME with finite-location diagnostics and the global kernel baseline. DR-ME-Rand uses the same statistic with random locations, IPW-ME uses only inverse-propensity weighting, DM-ME uses only the fitted regression contrast, Naive observed ME ignores the confounding \(X\), and DR-ME-NoSplit learns and tests on the same data. DR-xKTE \citep{martinez2023} is the calibrated global doubly robust kernel competitor. The left panel shows that calibration requires both orthogonalization and sample splitting: naive, plug-in, and no-split variants over-reject under the null in our confounded setting. The right panels show that, among calibrated methods (DR-ME, DR-xKTE, DR-ME-Rand), learning locations improves over random locations and even outperforms the global DR-xKTE baseline when the discrepancy is localized.

\paragraph{Covariance geometry.}
We next isolate the role of covariance whitening in location learning. The final
test is fixed across variants: all use the same split-sample DR-ME Hotelling
statistic on \(I_{\te}\); only the training-split location criterion changes. We
compare full whitening,
\(\hat\mu_V^\top(\hat\Sigma_V+\tau I)^{-1}\hat\mu_V\), with raw witness
maximization, \(\|\hat\mu_V\|^2\), and random locations. Outcomes are
non-scalar, \(Y\in\R^{d_Y}\), and the alternative places rare localized mass in
two sparse regions, so most dictionary locations are uninformative. This tests
whether whitening selects high signal-to-noise, nonredundant witness coordinates
rather than large or random witness values. Table~\ref{tab:covariance_geometry_main} shows that the difference is power, not
calibration: all three rules are close to nominal level under the null, but full
whitening is strongest at every dimension, with large gains when informative
locations are rare. This supports the criterion
\(\mu_V^\top\Sigma_V^{-1}\mu_V\): covariance whitening selects finite-location
coordinates with favorable observed-data signal-to-noise. Additional selection
diagnostics are deferred to Appendix~\ref{app:additional_experiments}.

\begin{table}[t]
\centering
\caption{
Non-scalar outcome two-bump ablation. Entries are rejection rates at level \(\alpha=0.05\). 
All methods use the same final DR-ME statistic; only the location-learning rule differs. 
Full covariance whitening remains calibrated under the null and gives the strongest power.
}
\label{tab:covariance_geometry_main}
\begin{tabular}{lcccccccc}
\toprule
& \multicolumn{4}{c}{Null} & \multicolumn{4}{c}{Two-bump alternative} \\
\cmidrule(lr){2-5}\cmidrule(lr){6-9}
Method 
& \(d_Y=5\) & \(10\) & \(25\) & \(50\)
& \(d_Y=5\) & \(10\) & \(25\) & \(50\) \\
\midrule
Full whitening & 0.070 & 0.025 & 0.050 & 0.070 & 0.940 & 0.945 & 0.760 & 0.660 \\
Raw witness    & 0.055 & 0.035 & 0.070 & 0.055 & 0.135 & 0.210 & 0.380 & 0.500 \\
Random         & 0.070 & 0.035 & 0.060 & 0.025 & 0.100 & 0.190 & 0.295 & 0.280 \\
\bottomrule
\end{tabular}
\end{table}

\paragraph{Interpretable image outcomes on OCTMNIST.}
We include a qualitative semi-synthetic OCTMNIST experiment to illustrate localization on structured outcomes \citep{medmnistv2}. Synthetic covariates \(X\) drive both observational treatment assignment and heterogeneity, while OCTMNIST images are used only as image outcomes. Potential outcomes are full \(28\times28\) OCT images \(Y(a)=B+R(a)\), where \(B\) is sampled from normal OCT images and the residuals \(R(0),R(1)\) are built from a synthetic DME-derived fluid-like template. The residual construction is mean-matched, with \(\E\{R(1)-R(0)\mid X\}=0\). Thus the treatment effect is primarily distributional: the treated law contains rare severe localized fluid-like deviations, while first-moment residual evidence is removed. Here, DR-ME learns with gradient based optimization a single image-space location on \(I_{\tr}\). Figure~\ref{fig:octmnist_mean_matched_main} shows that the learned location is not explained by an average residual shift: the oracle residual contrast is null, but \(v^\star\) concentrates on the central fluid-like region. This illustrates the intended interpretability of finite-location testing: beyond rejecting a global distributional null, the method returns an outcome-space coordinate where the causal discrepancy is visible.

\begin{figure}[h]
\centering
\includegraphics[height=3cm]{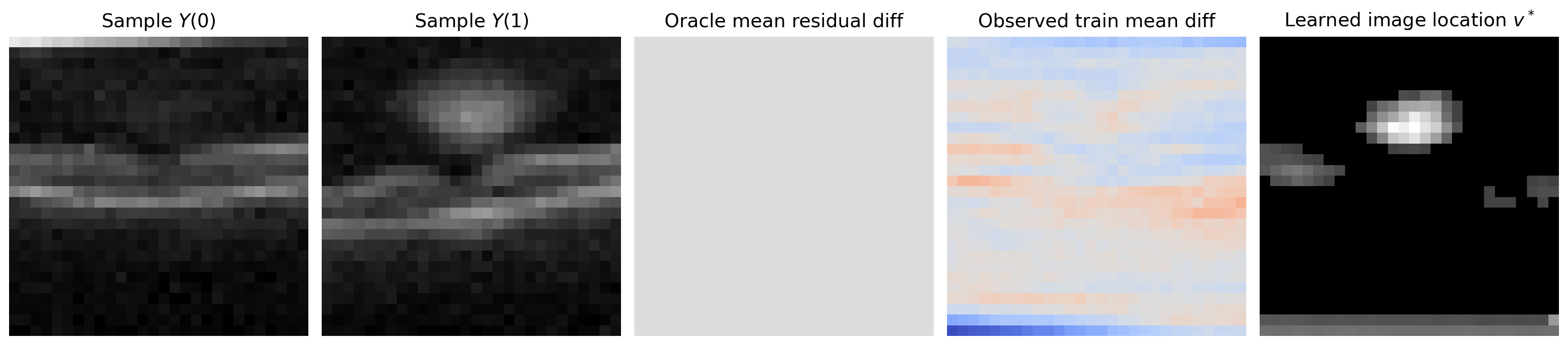}
\caption{
OCTMNIST image-location experiment. From left to right: sampled potential images \(Y(0)\) and \(Y(1)\), oracle residual mean difference, observed treated-versus-control training mean difference, and the learned image location \(v^\star\). The oracle residual mean contrast is blank by construction, while the observed training contrast can be nonzero due to confounding. The learned DR-ME location, selected on \(I_{\tr}\) and evaluated on \(I_{\te}\), localizes the retinal region where the interventional image laws differ in distribution.
}
\label{fig:octmnist_mean_matched_main}
\end{figure}

\section{Discussion}

This paper constructs a causal DR-ME Hotelling test for finite-location projections of interventional distributional discrepancies. For fixed locations, the target is a regular observed-data parameter whose canonical-gradient covariance gives the efficient local testing geometry. The resulting noncentrality \(\eta_V(g)^\top\Sigma_V^{-1}\eta_V(g)\) controls local power and motivates the location-learning criterion. DR-ME is therefore complementary to global kernel tests: global tests target omnibus sensitivity, while DR-ME targets interpretable, locally efficient witness coordinates, with power depending on whether the learned locations capture the discrepancy. Natural extensions are threefold. First, the current theory treats point treatments under observed confounding; longitudinal, adaptive, missing-data, instrumental-variable, or proximal settings would require new canonical gradients and local experiments. Second, fixed kernels could be replaced by deep kernels or learned feature maps, when representation learning is separated from final testing to preserve calibration as we do. Third, location learning should exploit structure: finite dictionaries for non-Euclidean outcomes such as DNA sequences, and constrained or regularized search classes for images and other structured outputs.

\subsubsection*{Acknowledgements}
Houssam Zenati and Arthur Gretton are supported by the Gatsby Charitable Foundation.

\bibliographystyle{plainnat}
\bibliography{references}

@article{piyavskii1972,
  author  = {Piyavskii, S. A.},
  title   = {An Algorithm for Finding the Absolute Extremum of a Function},
  journal = {USSR Computational Mathematics and Mathematical Physics},
  volume  = {12},
  number  = {4},
  pages   = {57--67},
  year    = {1972}
}

@article{shubert1972,
  author  = {Shubert, Bruno O.},
  title   = {A Sequential Method Seeking the Global Maximum of a Function},
  journal = {SIAM Journal on Numerical Analysis},
  volume  = {9},
  number  = {3},
  pages   = {379--388},
  year    = {1972}
}

@book{horst1996global,
  author    = {Horst, Reiner and Tuy, Hoang},
  title     = {Global Optimization: Deterministic Approaches},
  edition   = {3},
  publisher = {Springer},
  address   = {Berlin},
  year      = {1996}
}

@book{nocedal2006numerical,
  author    = {Nocedal, Jorge and Wright, Stephen J.},
  title     = {Numerical Optimization},
  edition   = {2},
  publisher = {Springer},
  address   = {New York},
  year      = {2006}
}

@article{medmnistv2,
    title={MedMNIST v2-A large-scale lightweight benchmark for 2D and 3D biomedical image classification},
    author={Yang, Jiancheng and Shi, Rui and Wei, Donglai and Liu, Zequan and Zhao, Lin and Ke, Bilian and Pfister, Hanspeter and Ni, Bingbing},
    journal={Scientific Data},
    volume={10},
    number={1},
    pages={41},
    year={2023},
    publisher={Nature Publishing Group UK London}
}

@article{kermany2018identifying,
  title   = {Identifying Medical Diagnoses and Treatable Diseases by Image-Based Deep Learning},
  author  = {Kermany, Daniel S. and Goldbaum, Michael and Cai, Wenjia and Valentim, Carolina C. S. and Liang, Huiying and Baxter, Sally L. and McKeown, Alex and Yang, Ge and Wu, Xiaokang and Yan, Fangbing and Dong, Justin and Prasadha, Made K. and Pei, Jacqueline and Ting, Magdalene Y. L. and Zhu, Jie and Li, Christina and Hewett, Sierra and Dong, Jason and Ziyar, Ian and Shi, Alexander and Zhang, Runze and Zheng, Lianghong and Hou, Rui and Shi, William and Fu, Xiaogang and Duan, Yaou and Huu, Viet Anh Nguyen and Wen, Cindy and Zhang, Edward D. and Zhang, Chenyang L. and Li, Oulan and Wang, Xiaobo and Singer, Michael A. and Sun, Xiaodong and Xu, Jie and Tafreshi, Ali and Lewis, M. Anthony and Xia, Huimin and Zhang, Kang},
  journal = {Cell},
  volume  = {172},
  number  = {5},
  pages   = {1122--1131.e9},
  year    = {2018},
}

@book{scholkopf2004kernel,
  editor    = {Sch{\"o}lkopf, Bernhard and Tsuda, Koji and Vert, Jean-Philippe},
  title     = {Kernel Methods in Computational Biology},
  publisher = {MIT Press},
  address   = {Cambridge, MA},
  year      = {2004},
  isbn      = {9780262195096}
}

@inproceedings{vert2005kernels,
  author    = {Vert, Jean-Philippe and Thurman, Robert and Noble, William S.},
  title     = {Kernels for Gene Regulatory Regions},
  booktitle = {Advances in Neural Information Processing Systems 18},
  editor    = {Weiss, Yair and Sch{\"o}lkopf, Bernhard and Platt, John},
  pages     = {1401--1408},
  publisher = {MIT Press},
  address   = {Cambridge, MA},
  year      = {2005}
}

@inproceedings{vert2006classification,
  author    = {Vert, Jean-Philippe},
  title     = {Classification of Biological Sequences with Kernel Methods},
  booktitle = {Grammatical Inference: Algorithms and Applications},
  editor    = {Sakakibara, Yasubumi and Kobayashi, Satoshi and Sato, Kengo and Nishino, Tetsuro and Tomita, Etsuji},
  series    = {Lecture Notes in Computer Science},
  volume    = {4201},
  pages     = {7--18},
  publisher = {Springer},
  address   = {Berlin, Heidelberg},
  year      = {2006},
  doi       = {10.1007/11872436_2}
}

@article{luedtke2019omnibus,
  title={An Omnibus Non-Parametric Test of Equality in Distribution for Unknown Functions},
  author={Luedtke, Alexander R. and Carone, Marco and van der Laan, Mark J.},
  journal={Journal of the Royal Statistical Society: Series B},
  volume={81},
  number={1},
  pages={75--99},
  year={2019}
}

@misc{hudson2021restricted,
  title={Inference on Function-Valued Parameters Using a Restricted Score Test},
  author={Hudson, Aaron and Carone, Marco and Shojaie, Ali},
  year={2021},
  eprint={2105.06646},
  archivePrefix={arXiv},
  primaryClass={stat.ME}
}

@misc{zenati2025kteadaptive,
  title         = {Kernel Treatment Effects with Adaptively Collected Data},
  author        = {Zenati, Houssam and Bozkurt, Bariscan and Gretton, Arthur},
  year          = {2025},
  eprint        = {2510.10245},
  archivePrefix = {arXiv},
  primaryClass  = {stat.ML},
  url           = {https://arxiv.org/abs/2510.10245}
}

@inproceedings{martinez2023,
 author = {Martinez Taboada, Diego and Ramdas, Aaditya and Kennedy, Edward},
 booktitle = {Advances in Neural Information Processing Systems},
 pages = {59924--59952},
 title = {An Efficient Doubly-Robust Test for the Kernel Treatment Effect},
 volume = {36},
 year = {2023}
}

@article{
fawkes2024doubly,
title={Doubly Robust Kernel Statistics for Testing Distributional Treatment Effects},
author={Jake Fawkes and Robert Hu and Robin J. Evans and Dino Sejdinovic},
journal={Transactions on Machine Learning Research},
year={2024},
}

@inproceedings{jitkrittum2016interpretable,
  title     = {Interpretable Distribution Features with Maximum Testing Power},
  author    = {Jitkrittum, Wittawat and Szab{\'o}, Zolt{\'a}n and Chwialkowski, Kacper P. and Gretton, Arthur},
  booktitle = {Advances in Neural Information Processing Systems 29},
  pages     = {181--189},
  year      = {2016}
}

@inproceedings{Chwialkowski2015,
 author = {Chwialkowski, Kacper P and Ramdas, Aaditya and Sejdinovic, Dino and Gretton, Arthur},
 booktitle = {Advances in Neural Information Processing Systems},
 title = {Fast Two-Sample Testing with Analytic Representations of Probability Measures},
 volume = {28},
 year = {2015}
}

@misc{fithian2014optimal,
  title={Optimal Inference After Model Selection},
  author={Fithian, William and Sun, Dennis and Taylor, Jonathan},
  year={2014},
  eprint={1410.2597},
  archivePrefix={arXiv},
  primaryClass={math.ST}
}

@article{kuchibhotla2022postselection,
  title={Post-Selection Inference},
  author={Kuchibhotla, Arun Kumar and Kolassa, John E. and Kuffner, Todd A.},
  journal={Annual Review of Statistics and Its Application},
  volume={9},
  pages={505--527},
  year={2022},
  doi={10.1146/annurev-statistics-100421-044639}
}

@inproceedings{smola2007hilbert,
  title={A Hilbert space embedding for distributions},
  author={Smola, Alex and Gretton, Arthur and Song, Le and Sch{\"o}lkopf, Bernhard},
  booktitle={International conference on algorithmic learning theory},
  pages={13--31},
  year={2007},
  organization={Springer}
}

@article{rothe2010nonparametric,
  title={Nonparametric estimation of distributional policy effects},
  author={Rothe, Christoph},
  journal={Journal of Econometrics},
  volume={155},
  number={1},
  pages={56--70},
  year={2010},
}

@article{chernozhukov2013inference,
  title={Inference on counterfactual distributions},
  author={Chernozhukov, Victor and Fern{\'a}ndez-Val, Iv{\'a}n and Melly, Blaise},
  journal={Econometrica},
  volume={81},
  number={6},
  pages={2205--2268},
  year={2013},
}

@article{muandet2021counterfactual,
  title={Counterfactual mean embeddings},
  author={Muandet, Krikamol and Kanagawa, Motonobu and Saengkyongam, Sorawit and Marukatat, Sanparith},
  journal={Journal of Machine Learning Research},
  volume={22},
  number={162},
  pages={1--71},
  year={2021}
}

@article{gretton2012kernel,
  title={A kernel two-sample test},
  author={Gretton, Arthur and Borgwardt, Karsten M and Rasch, Malte J and Sch{\"o}lkopf, Bernhard and Smola, Alexander},
  journal={The Journal of Machine Learning Research},
  volume={13},
  number={1},
  pages={723--773},
  year={2012},
}

@article{gartner2003survey,
  title={A survey of kernels for structured data},
  author={G{\"a}rtner, Thomas},
  journal={ACM SIGKDD explorations newsletter},
  volume={5},
  number={1},
  pages={49--58},
  year={2003},
}

@article{sriperumbudur2011universality,
  title={Universality, Characteristic Kernels and RKHS Embedding of Measures.},
  author={Sriperumbudur, Bharath K and Fukumizu, Kenji and Lanckriet, Gert RG},
  journal={Journal of Machine Learning Research},
  volume={12},
  number={7},
  year={2011}
}

@book{vandervaart1998asymptotic,
  title     = {Asymptotic Statistics},
  author    = {{van der Vaart}, A. W.},
  year      = {1998},
  publisher = {Cambridge University Press},
  series    = {Cambridge Series in Statistical and Probabilistic Mathematics},
  volume    = {3},
  doi       = {10.1017/CBO9780511802256}
}

@book{lecam2000asymptotics,
  title     = {Asymptotics in Statistics: Some Basic Concepts},
  author    = {Le Cam, Lucien M. and Yang, Grace Lo},
  year      = {2000},
  edition   = {2},
  publisher = {Springer},
  series    = {Springer Series in Statistics},
  doi       = {10.1007/978-1-4612-1166-2}
}

@article{neyman1933mostefficient,
  author  = {Neyman, Jerzy and Pearson, Egon S.},
  title   = {On the Problem of the Most Efficient Tests of Statistical Hypotheses},
  journal = {Philosophical Transactions of the Royal Society of London. Series A, Containing Papers of a Mathematical or Physical Character},
  volume  = {231},
  number  = {694--706},
  pages   = {289--337},
  year    = {1933},
}

@article{rubin1974estimating,
  title={Estimating Causal Effects of Treatments in Randomized and Nonrandomized Studies},
  author={Rubin, Donald B.},
  journal={Journal of Educational Psychology},
  volume={66},
  number={5},
  pages={688--701},
  year={1974},
  doi={10.1037/h0037350}
}

@book{imbens2015causal,
  title={Causal Inference for Statistics, Social, and Biomedical Sciences: An Introduction},
  author={Imbens, Guido W. and Rubin, Donald B.},
  publisher={Cambridge University Press},
  year={2015},
  doi={10.1017/CBO9781139025751}
}

@article{rosenbaum1983central,
  title={The Central Role of the Propensity Score in Observational Studies for Causal Effects},
  author={Rosenbaum, Paul R. and Rubin, Donald B.},
  journal={Biometrika},
  volume={70},
  number={1},
  pages={41--55},
  year={1983},
  doi={10.1093/biomet/70.1.41}
}

@book{hernanrobins2020causal,
  title={Causal Inference: What If},
  author={Hern{\'a}n, Miguel A. and Robins, James M.},
  publisher={Chapman \& Hall/CRC},
  year={2020}
}

@article{robins1994estimation,
  title={Estimation of Regression Coefficients When Some Regressors Are Not Always Observed},
  author={Robins, James M. and Rotnitzky, Andrea and Zhao, Lue Ping},
  journal={Journal of the American Statistical Association},
  volume={89},
  number={427},
  pages={846--866},
  year={1994},
  doi={10.1080/01621459.1994.10476818}
}

@article{bang2005doubly,
  title={Doubly Robust Estimation in Missing Data and Causal Inference Models},
  author={Bang, Heejung and Robins, James M.},
  journal={Biometrics},
  volume={61},
  number={4},
  pages={962--973},
  year={2005},
  doi={10.1111/j.1541-0420.2005.00377.x}
}

@article{hahn1998role,
  title={On the Role of the Propensity Score in Efficient Semiparametric Estimation of Average Treatment Effects},
  author={Hahn, Jinyong},
  journal={Econometrica},
  volume={66},
  number={2},
  pages={315--331},
  year={1998},
  doi={10.2307/2998560}
}

@article{chernozhukov2018double,
  title={Double/Debiased Machine Learning for Treatment and Structural Parameters},
  author={Chernozhukov, Victor and Chetverikov, Denis and Demirer, Mert and Duflo, Esther and Hansen, Christian and Newey, Whitney and Robins, James},
  journal={The Econometrics Journal},
  volume={21},
  number={1},
  pages={C1--C68},
  year={2018},
  doi={10.1111/ectj.12097}
}

@book{bickel1993efficient,
  title={Efficient and Adaptive Estimation for Semiparametric Models},
  author={Bickel, Peter J. and Klaassen, Chris A. J. and Ritov, Ya'acov and Wellner, Jon A.},
  publisher={Johns Hopkins University Press},
  year={1993}
}

@misc{kennedy2022semiparametric,
  title={Semiparametric Doubly Robust Targeted Double Machine Learning: A Review},
  author={Kennedy, Edward H.},
  year={2022},
  eprint={2203.06469},
  archivePrefix={arXiv},
  primaryClass={stat.ME}
}

@article{hotelling1931generalization,
  title={The Generalization of Student's Ratio},
  author={Hotelling, Harold},
  journal={The Annals of Mathematical Statistics},
  volume={2},
  number={3},
  pages={360--378},
  year={1931},
  doi={10.1214/aoms/1177732979}
}

@article{luedtke2024one,
  title={One-step estimation of differentiable hilbert-valued parameters},
  author={Luedtke, Alex and Chung, Incheoul},
  journal={The Annals of Statistics},
  volume={52},
  number={4},
  pages={1534--1563},
  year={2024},
  publisher={Institute of Mathematical Statistics}
}

@inproceedings{
zenati2025cpme,
title={Doubly-Robust Estimation of Counterfactual Policy Mean Embeddings},
author={Houssam Zenati and Bariscan Bozkurt and Arthur Gretton},
booktitle={The Thirty-ninth Annual Conference on Neural Information Processing Systems},
year={2025},
url={https://openreview.net/forum?id=0GDlX9JFf2}
}


\newpage
\appendix

\section*{Appendix}

\paragraph{Appendix organization.}
Appendix~\ref{app:notation_assumptions} collects notation and the full assumptions used in the main text. Appendix~\ref{app:background} gives a compact review of the local asymptotic tools used for the testing-efficiency statements. Appendix~\ref{app:finite_scores_proofs} proves the identification, canonical-gradient, and first-order representation results from Section~\ref{sec:finite_scores}.  Appendix~\ref{app:local_proofs} proves the efficient local testing results from Section~\ref{sec:theory}. Appendix~\ref{app:local_results} records auxiliary local-testing facts. Appendix~\ref{app:learning_proofs} proves the Euclidean location-learning result from Section~\ref{sec:learning}, and Appendix~\ref{app:finite_dictionary} gives the finite-dictionary variant for structured outcomes. Appendix~\ref{app:implementation} gives implementation details. Appendix~\ref{app:additional_experiments} contains additional experiments.

An anonymized implementation is included in the supplementary material.

\section{Notation and full assumptions}
\label{app:notation_assumptions}

\paragraph{Observed data and potential outcomes.}
We observe \(Z=(X,A,Y)\sim P_0\), with binary treatment \(A\in\{0,1\}\). The potential outcomes are \(Y(0),Y(1)\in\Ycal\). Unless a split-sample construction is explicitly used, \(Z_1,\ldots,Z_n\iid P_0\).

\paragraph{Kernel and finite-location notation.}
The outcome kernel \(k_Y:\Ycal\times\Ycal\to\R\) has RKHS \(\Hv\) and feature map \(\varphi_Y(y)=k_Y(\cdot,y)\). Throughout, \(k_Y\) is measurable, positive definite, and bounded: \(\sup_y k_Y(y,y)\le\kappa^2<\infty\). For \(V=(v_1,\ldots,v_J)\in\Ycal^J\), write
\[
k_V(y):=k_Y(V,y):=(k_Y(v_1,y),\ldots,k_Y(v_J,y))^\top\in\R^J .
\]

\paragraph{Interventional embeddings and witness coordinates.}
For \(a\in\{0,1\}\), let \(\chi(a):=\E[\varphi_Y(Y(a))]\in\Hv\), \(\Delta:=\chi(1)-\chi(0)\), and
\[
w(y):=\langle\Delta,\varphi_Y(y)\rangle_{\Hv}
=
\E[k_Y(y,Y(1))]-\E[k_Y(y,Y(0))].
\]
The finite-location arm means and contrast are
\[
\mu_{a,V}:=\E[k_V(Y(a))],
\qquad
\mu_V:=\mu_{1,V}-\mu_{0,V}
=(w(v_1),\ldots,w(v_J))^\top\in\R^J .
\]

\paragraph{Observed-data nuisance functions.}
For \(a\in\{0,1\}\), define
\[
\pi_0(a\mid x):=P_0(A=a\mid X=x),
\qquad
m_a(x;V):=\E[k_V(Y)\mid A=a,X=x].
\]
A generic nuisance tuple is \(\eta=(\pi,r_0,r_1)\), and the true tuple is \(\eta_0=(\pi_0,m_0,m_1)\).

\paragraph{Orthogonal pseudo-features.}
For \(a\in\{0,1\}\), define
\[
\phi_V^a(Z;\eta)
:=
\frac{\1\{A=a\}}{\pi(a\mid X)}
\{k_V(Y)-r_a(X;V)\}
+r_a(X;V),
\qquad
z_V^{\dr}(Z;\eta):=\phi_V^1(Z;\eta)-\phi_V^0(Z;\eta).
\]
At the truth,
\[
\psi_V(Z):=z_V^{\dr}(Z;\eta_0)-\mu_V,
\qquad
\Sigma_V:=\Var_{P_0}\{\psi_V(Z)\}.
\]
Proposition~\ref{prop:dr_canonical_gradient} shows that \(\psi_V\) is the canonical gradient of \(\mu_V\) in the observed-data model.

\paragraph{Fixed-location statistic.}
For fixed \(V\), write the feasible pseudo-feature as
\[
\hat z_{i,V}^{\dr}:=z_V^{\dr}(Z_i;\hat\eta_{-i}),
\]
where \(\hat\eta_{-i}\) is trained on data independent of \(Z_i\). This covers both \(K\)-fold cross-fitting and the independent test split used after location learning. Define
\[
\bar z_{n,V}^{\dr}:=\frac1n\sum_{i=1}^n\hat z_{i,V}^{\dr},
\qquad
S_{n,V}^{\dr}:=
\frac{1}{n-1}\sum_{i=1}^n
(\hat z_{i,V}^{\dr}-\bar z_{n,V}^{\dr})
(\hat z_{i,V}^{\dr}-\bar z_{n,V}^{\dr})^\top .
\]
The fixed-location Hotelling statistic is
\[
\hat\lambda_{n,V}^{\dr}
:=
n\,\bar z_{n,V}^{\dr\top}
(S_{n,V}^{\dr}+\gamma_n I_J)^{-1}
\bar z_{n,V}^{\dr},
\qquad
\gamma_n\downarrow0 .
\]

\paragraph{Local asymptotic notation.}
Let \(\Pcal\) be a semiparametric model for the observed-data law. For a fixed \(V\), let \(P_0\in\Pcal\) satisfy \(\mu_V(P_0)=0\), and let \(T_{P_0}\subset L_2^0(P_0)\) be the tangent space. For a score direction \(g\in T_{P_0}\), let \(t\mapsto P_{t,g}\) be a regular quadratic-mean differentiable path through \(P_0\) with score \(g\), and define \(P_{n,h,g}:=P_{h/\sqrt n,g}^{\otimes n}\). The local drift and whitened local signal are
\[
\eta_V(g):=
\left.\frac{d}{dt}\right|_{t=0}\mu_V(P_{t,g})
=
\E_{P_0}\{\psi_V(Z)g(Z)\},
\qquad
\lambda_V(g):=\eta_V(g)^\top\Sigma_V^{-1}\eta_V(g).
\]

\paragraph{Location learning notation.}
The full sample is split as \(\{1,\ldots,n\}=I_\eta\cup I_{\tr}\cup I_{\te}\), with \(n_\eta=|I_\eta|\), \(n_{\tr}=|I_{\tr}|\), and \(n_{\te}=|I_{\te}|\). Nuisances are fitted on \(I_\eta\), locations are learned on \(I_{\tr}\), and the final test is computed on \(I_{\te}\). The population and empirical ridge-stabilized criteria are
\[
\mathcal P_\tau(V):=\mu_V^\top(\Sigma_V+\tau I_J)^{-1}\mu_V,
\qquad
\hat{\mathcal P}_{\tau,\tr}(V):=
\bar z_{\tr,V}^{\dr\top}
(S_{\tr,V}^{\dr}+\tau I_J)^{-1}
\bar z_{\tr,V}^{\dr}.
\]
For a learned location set \(\hat V\), define its empirical optimization gap
\[
\varepsilon_{\tr}(\hat V)
:=
\sup_{V\in\mathcal V}\hat{\mathcal P}_{\tau,\tr}(V)
-
\hat{\mathcal P}_{\tau,\tr}(\hat V),
\]
and the uniform learning error
\[
\Delta_{\tr}:=
\sup_{V\in\mathcal V}
\left|
\hat{\mathcal P}_{\tau,\tr}(V)-\mathcal P_\tau(V)
\right|.
\]

\paragraph{Norms.}
Unless stated otherwise, \(\|\cdot\|\) is the Euclidean norm, \(\|\cdot\|_F\) is the Frobenius norm, and \(\|\cdot\|_{L_2(P_X;\R^J)}\) is the \(L_2(P_X)\) norm for \(\R^J\)-valued functions.

\subsection{Observed-data identification}
\label{app:identification_assumptions}

\begin{assumption}[Observed-data identification]
\label{ass:identification_full}
For each \(a\in\{0,1\}\), the following hold.
\begin{enumerate}[label=(\roman*),leftmargin=1.1cm]
    \item \emph{Consistency}: \(Y=Y(a)\) whenever \(A=a\).
    \item \emph{Conditional exchangeability}: \(Y(a)\perp A\mid X\).
    \item \emph{Positivity}: there exists \(\varepsilon>0\) such that \(\pi_0(a\mid X)\ge\varepsilon\) almost surely.
\end{enumerate}
\end{assumption}

Under Assumption~\ref{ass:identification_full}, \(\mu_{a,V}=\E[m_a(X;V)]\) and hence \(\mu_V=\E[m_1(X;V)-m_0(X;V)]\).

\subsection{Fixed-location nuisance and covariance conditions}
\label{app:first_stage_assumptions}

Theorem~\ref{thm:first_order} uses the following fixed-location conditions. They are stated for \(K\)-fold cross-fitting; for an independent nuisance split, remove the maxima over folds and use the single nuisance estimate \(\hat\eta\).

\begin{assumption}[Fixed-location first-stage conditions]
\label{ass:fixedV_nuisance_full}
Fix \(V\in\Ycal^J\). The validation folds \(\mathcal I_1,\ldots,\mathcal I_K\) satisfy \(\min_k|\mathcal I_k|/n\ge c_K>0\), with fixed \(K\). For \(i\in\mathcal I_k\), the nuisance estimate \(\hat\eta_k=(\hat\pi_k,\hat m_{0,k},\hat m_{1,k})\) is trained outside fold \(k\).

There exist constants \(\underline\varepsilon>0\) and \(M<\infty\) such that, with probability tending to one,
\[
\inf_{k,a,x}\hat\pi_k(a\mid x)\ge\underline\varepsilon,
\qquad
\sup_{k,a,x}\|\hat m_{a,k}(x;V)\|\le M .
\]
Moreover,
\[
\max_k\sum_{a=0}^1
\left\{
\|\hat\pi_k(a\mid\cdot)-\pi_0(a\mid\cdot)\|_{L_2(P_X)}
+
\|\hat m_{a,k}(\cdot;V)-m_a(\cdot;V)\|_{L_2(P_X;\R^J)}
\right\}
=o_p(1),
\]
and
\[
\max_k\sum_{a=0}^1
\|\hat\pi_k(a\mid\cdot)-\pi_0(a\mid\cdot)\|_{L_2(P_X)}
\,
\|\hat m_{a,k}(\cdot;V)-m_a(\cdot;V)\|_{L_2(P_X;\R^J)}
=o_p(n^{-1/2}).
\]
\end{assumption}

\begin{assumption}[Nondegenerate fixed-location covariance]
\label{ass:nondegenerate_full}
For the fixed location set \(V\), the efficient covariance \(\Sigma_V=\Var_{P_0}\{\psi_V(Z)\}\) is positive definite.
\end{assumption}

Assumption~\ref{ass:nondegenerate_full} is only needed for inverse-covariance whitening and chi-square limits. If \(\Sigma_V\) is singular, one may work in the nonzero eigenspace, but we do not pursue that extension.

\subsection{Regular local paths}
\label{app:local_assumptions}

The local testing results are formulated along regular quadratic-mean differentiable paths. This is the primitive local smoothness condition; the LAN expansion used in the proofs follows from it.

\begin{assumption}[Regular QMD local path]
\label{ass:qmd_path_full}
For each fixed \(g\in T_{P_0}\), there exists a path \(t\mapsto P_{t,g}\subset\Pcal\) through \(P_0\) with densities \(p_{t,g}\) relative to a dominating measure \(\nu\), such that
\[
\int
\left(
\sqrt{p_{t,g}}
-
\sqrt{p_0}
-
\frac{t}{2}g\sqrt{p_0}
\right)^2
d\nu
=
o(t^2).
\]
\end{assumption}

Under Assumption~\ref{ass:qmd_path_full}, the standard LAN expansion holds for \(P_{n,h,g}=P_{h/\sqrt n,g}^{\otimes n}\):
\[
\log\frac{dP_{n,h,g}}{dP_0^{\otimes n}}
=
h\,\mathbb G_n g
-
\frac12 h^2\|g\|_{L_2(P_0)}^2
+
o_{P_0}(1),
\qquad
\mathbb G_n g:=\frac1{\sqrt n}\sum_{i=1}^n\{g(Z_i)-\E_{P_0}g(Z)\}.
\]
It also implies contiguity of \(P_{n,h,g}\) with respect to \(P_0^{\otimes n}\). Thus \(o_{P_0}(1)\) remainders in the fixed-location first-order expansion transfer to \(o_{P_{n,h,g}}(1)\) remainders along the local path.

For Corollary~\ref{cor:local_power_criterion}, we additionally assume that \(P\mapsto\Sigma_V(P)\) is continuous along \(t\mapsto P_{t,g}\) at \(t=0\), and that \(\Sigma_V(P_{t,g})\) remains nonsingular for \(t\) in a neighborhood of zero.

\subsection{Euclidean location-learning conditions}
\label{app:learning_assumptions}

Theorem~\ref{thm:location_learning} assumes \(\Ycal\subset\R^d\) and a compact search class \(\mathcal V\subset[-R,R]^{Jd}\). Let
\[
\Pi(\mathcal V)
:=
\{v\in\R^d:\ v=v_j \text{ for some } V=(v_1,\ldots,v_J)\in\mathcal V\}.
\]

\begin{assumption}[Euclidean location regularity]
\label{ass:loclearn_euclid_full}
There exist constants \(B_k,L_k,B_m,L_m<\infty\) such that the following hold.
\begin{enumerate}[label=(\roman*),leftmargin=1.1cm]
    \item For all \(v,v'\in\Pi(\mathcal V)\) and \(y\in\Ycal\),
    \[
    |k_Y(v,y)|\le B_k,
    \qquad
    |k_Y(v,y)-k_Y(v',y)|\le L_k\|v-v'\|.
    \]
    \item With probability tending to one, for all \(a\in\{0,1\}\), \(x\in\Xcal\), and \(v,v'\in\Pi(\mathcal V)\),
    \[
    |\hat m_a(x;v)|\le B_m,
    \qquad
    |\hat m_a(x;v)-\hat m_a(x;v')|\le L_m\|v-v'\|.
    \]
    \item The estimated propensity is uniformly bounded away from zero with probability tending to one:
    \[
    \inf_{a,x}\hat\pi(a\mid x)\ge\underline\varepsilon
    \]
    for some \(\underline\varepsilon>0\).
\end{enumerate}
\end{assumption}

The true regressions \(m_a(\cdot;v)\) inherit the boundedness and Lipschitz properties from the kernel by Jensen's inequality, so no separate smoothness assumption on \(m_a\) is needed for the learning theorem. The nuisance error entering Theorem~\ref{thm:location_learning} is
\[
\rho_{n_\eta}
:=
\sum_{a=0}^1
\left[
\|\hat\pi(a\mid\cdot)-\pi_0(a\mid\cdot)\|_{L_2(P_X)}
+
\sup_{v\in\Pi(\mathcal V)}
\|\hat m_a(\cdot;v)-m_a(\cdot;v)\|_{L_2(P_X)}
\right].
\]
The learning theorem only requires \(\rho_{n_\eta}=o_p(1)\), because it controls uniform consistency of the learning criterion rather than a root-\(n\) expansion of the final test statistic.

\section{Background on local asymptotic testing for fixed locations}
\label{app:background}

This appendix reviews the local asymptotic tools used in Sections~\ref{sec:finite_scores} and~\ref{sec:theory}. Once the location set \(V\) is fixed, the target is the finite-dimensional parameter \(\mu_V(P)\in\R^J\). Proposition~\ref{prop:dr_canonical_gradient} identifies its observed-data canonical gradient \(\psi_V\), and Theorem~\ref{thm:efficient_local_testing} studies the local testing problem generated by this gradient.

The relevant objects are the local drift \(\eta_V(g)\), the efficient covariance \(\Sigma_V\), and the whitened local signal
\[
\lambda_V(g):=\eta_V(g)^\top\Sigma_V^{-1}\eta_V(g).
\]
The role of this appendix is to explain why these three quantities govern both regular estimation of \(\mu_V\) and first-order local power of tests at the selected locations.

\subsection{QMD paths and the local Gaussian approximation}

The local theory is pathwise. We do not assume a finite-dimensional parametric model for the whole data-generating law. Instead, for each score direction \(g\in T_{P_0}\), Assumption~\ref{ass:qmd_path_full} postulates a regular quadratic-mean differentiable path \(t\mapsto P_{t,g}\) through \(P_0\). Quadratic mean differentiability means that, for densities \(p_{t,g}\) relative to a dominating measure \(\nu\),
\[
\int
\left(
\sqrt{p_{t,g}}-\sqrt{p_0}-\frac{t}{2}g\sqrt{p_0}
\right)^2d\nu=o(t^2).
\]
This condition is the standard local smoothness assumption behind Le Cam's asymptotic theory \citep{lecam2000asymptotics,vandervaart1998asymptotic,bickel1993efficient}.

For the contiguous alternatives \(P_{n,h,g}:=P_{h/\sqrt n,g}^{\otimes n}\), QMD implies the LAN expansion \citep[Theorem~7.2]{vandervaart1998asymptotic}
\[
\log\frac{dP_{n,h,g}}{dP_0^{\otimes n}}
=
h\,\mathbb G_n g
-
\frac12 h^2\|g\|_{L_2(P_0)}^2
+
o_{P_0}(1),
\qquad
\mathbb G_n g:=\frac1{\sqrt n}\sum_{i=1}^n g(Z_i),
\]
where \(g\in L_2^0(P_0)\). Thus, along each regular path, the original experiment is locally approximated by a Gaussian shift experiment. This is the only LAN input used in the paper. No least favorable submodel needs to be constructed for the main results; least favorable paths are useful for interpretation, but the proofs use arbitrary regular QMD paths and tangent-space projection.

\subsection{Le Cam's third lemma and the local drift}

Le Cam's third lemma describes how an asymptotically linear statistic shifts under a contiguous alternative. If
\[
T_n=\frac1{\sqrt n}\sum_{i=1}^n \phi(Z_i)+o_{P_0}(1),
\qquad
\E_{P_0}\phi(Z)=0,
\]
then, under \(P_{n,h,g}\),
\[
T_n\dto
N\!\left(
h\,\E_{P_0}\{\phi(Z)g(Z)\},
\Var_{P_0}\{\phi(Z)\}
\right).
\]
The covariance is unchanged to first order; the mean shifts in the direction correlated with the score \(g\).

For the fixed-location statistic, \(\phi=\psi_V\). Since \(\psi_V\) is the canonical gradient of \(\mu_V\), the pathwise derivative of \(\mu_V\) along \(g\) is
\[
\eta_V(g)
:=
\left.\frac{d}{dt}\right|_{t=0}\mu_V(P_{t,g})
=
\E_{P_0}\{\psi_V(Z)g(Z)\}.
\]
Thus \(\eta_V(g)\) is not an auxiliary definition. It is exactly the mean shift of the canonical-gradient statistic. If \(\eta_V(g)=0\), the selected locations are locally blind to direction \(g\); if \(\eta_V(g)\neq0\), the local alternative produces a first-order shift in the witness coordinates.

\subsection{Regularity, convolution, and the efficient finite-signal experiment}

In the nonparametric observed-data model, the canonical gradient of the finite
signal \(\mu_V\) is unique. Thus the efficiency statement is not a nontrivial
comparison among different regular influence functions for the same target.
Rather, it is a statement about the efficient limit experiment for regular
estimation of \(\mu_V\).

By the H\'ajek--Le Cam convolution theorem
\citep[Theorem~25.20]{vandervaart1998asymptotic}, if \(T_n\) is any regular
estimator of \(\mu_V\) and
\[
\sqrt n\{T_n-\mu_V(P_0)\}\dto L
\quad\text{under }P_0^{\otimes n},
\]
then
\[
L = N(0,\Sigma_V)*M
\]
for some probability law \(M\). The canonical-gradient estimator
\(\bar z_{n,V}^{\dr}\) attains the no-extra-noise case \(M=\delta_0\), since
Theorem~\ref{thm:first_order} gives
\[
\sqrt n(\bar z_{n,V}^{\dr}-\mu_V)
=
\frac1{\sqrt n}\sum_{i=1}^n\psi_V(Z_i)+o_p(1).
\]
Hence \(N(0,\Sigma_V)\) is the efficient Gaussian limit for the finite witness.
The covariance \(\Sigma_V\) is therefore not one possible normalization among
many; it is the covariance of the efficient finite-signal experiment.

\subsection{From efficient estimation to local testing geometry}

Under Theorem~\ref{thm:first_order} and the local path assumptions,
\[
\sqrt n\,\bar z_{n,V}^{\dr}
\dto
N(h\eta_V(g),\Sigma_V)
\qquad
\text{under } P_{n,h,g}.
\]
After whitening, the efficient finite-signal experiment is
\[
\Sigma_V^{-1/2}\sqrt n\,\bar z_{n,V}^{\dr}
\dto
N(h\Sigma_V^{-1/2}\eta_V(g),I_J).
\]

For any scalar contrast \(a^\top\mu_V\), the efficient squared local
signal-to-noise ratio along \(g\) is
\[
\frac{(a^\top\eta_V(g))^2}{a^\top\Sigma_Va}.
\]
If the local direction \(g\) were known, optimizing this quantity over
\(a\neq0\) gives the Rayleigh quotient
\[
\max_{a\neq0}
\frac{(a^\top\eta_V(g))^2}{a^\top\Sigma_Va}
=
\eta_V(g)^\top\Sigma_V^{-1}\eta_V(g)
=
\lambda_V(g),
\]
with optimizer proportional to \(\Sigma_V^{-1}\eta_V(g)\). This is the
directional Neyman--Pearson benchmark in the finite-signal Gaussian experiment.

When the direction is unknown, no uniformly most powerful test exists over all
drift directions in the multivariate Gaussian shift experiment. The natural
omnibus reduction is the whitened problem \(G\sim N(m,I_J)\), with
\(H_0:m=0\) and \(H_1:m\neq0\). In this limit experiment, the likelihood-ratio,
Wald, and score statistics all reduce to \(\|G\|^2\). Transported back to the
finite witness, this is the Hotelling statistic based on \(\Sigma_V^{-1}\),
with noncentrality \(h^2\lambda_V(g)\). This is why the population
location-learning criterion is
\[
\Lambda_V(P)
:=
\mu_V(P)^\top\Sigma_V(P)^{-1}\mu_V(P),
\]
and why the practical criterion uses
\(\mathcal P_\tau(V)=\mu_V^\top(\Sigma_V+\tau I_J)^{-1}\mu_V\).

\subsection{Relation to existing semiparametric kernel inference}

This paper uses kernel embeddings, but the local problem is different from global MMD testing with unknown nuisance functions. In the unknown-function MMD setting of \citet{luedtke2019omnibus}, the squared MMD has a degenerate first-order derivative under the null, and valid testing requires a second-order U-statistic analysis. Here, for fixed \(V\), the target \(\mu_V\in\R^J\) is first-order pathwise differentiable. The null limit is therefore a finite-dimensional chi-square law rather than a degenerate infinite weighted chi-square limit.

The paper is also related to Hilbert-valued semiparametric one-step estimation. General Hilbert-valued theory shows that, when an efficient influence function exists, one-step estimators can achieve root-\(n\) Hilbert-norm inference; counterfactual kernel mean embeddings are a key example \citep{luedtke2024one}. Our construction can be viewed as applying bounded linear evaluation maps to the interventional RKHS discrepancy, yielding a finite-dimensional parameter with canonical gradient \(\psi_V\). The contribution here is not Hilbert-valued estimation per se, but the testing geometry induced by this finite projection and its use for interpretable location learning.

\subsection{Why the orthogonal DR geometry is the relevant geometry}

The regression contrast \(m_1(X;V)-m_0(X;V)\) has mean \(\mu_V\). If \(m_0\) and \(m_1\) were known, its covariance could be smaller than \(\Sigma_V\). That comparison, however, corresponds to an oracle problem in which nuisance regressions are given. It is not the observed-data semiparametric problem faced by a regular procedure using flexible nuisance estimates.

A feasible plug-in regression contrast is first-order sensitive to regression error. Without stronger nuisance conditions, it does not provide the stable local expansion needed for regular local testing. By contrast, the orthogonal pseudo-feature \(z_V^{\dr}(Z;\eta)\) is built from the canonical gradient. Theorem~\ref{thm:first_order} gives
\[
\sqrt n(\bar z_{n,V}^{\dr}-\mu_V)
=
\frac1{\sqrt n}\sum_{i=1}^n\psi_V(Z_i)+o_p(1),
\qquad
S_{n,V}^{\dr}\to_p\Sigma_V.
\]
Thus the fixed-location test enters the local Gaussian experiment with the efficient observed-data score \(\psi_V\). Thus the fixed-location test enters the efficient finite-signal Gaussian experiment with drift \(\eta_V(g)\) and covariance \(\Sigma_V\). This is the covariance geometry used both for local testing and for location learning.

\section{Proofs for Section~\ref{sec:finite_scores}}
\label{app:finite_scores_proofs}

Throughout this appendix, fix \(V=(v_1,\ldots,v_J)\in\Ycal^J\) and write
\[
k_V(y):=k_Y(V,y)=\bigl(k_Y(v_1,y),\ldots,k_Y(v_J,y)\bigr)^\top\in\R^J .
\]
We use \(\|\cdot\|\) for the Euclidean norm on \(\R^J\). Since \(k_Y\) is bounded and \(J\) is fixed,
\[
B_V:=\sup_{y\in\Ycal}\|k_V(y)\|<\infty .
\]
Consequently, \(\|m_a(x;V)\|\le B_V\) for \(a\in\{0,1\}\) and all \(x\). We prove the cross-fitted case. The independent split case is the same argument with a single validation split independent of the nuisance-fitting sample.

\subsection{Identification}

\begin{proof}[Proof of the identification claim]
Fix \(a\in\{0,1\}\). By iterated expectation,
\[
\mu_{a,V}
=
\E\bigl[k_V(Y(a))\bigr]
=
\E\!\left[\E\{k_V(Y(a))\mid X\}\right].
\]
Conditional exchangeability gives
\[
\E\{k_V(Y(a))\mid X\}
=
\E\{k_V(Y(a))\mid A=a,X\},
\]
and consistency gives \(Y=Y(a)\) on \(\{A=a\}\). Hence
\[
\E\{k_V(Y(a))\mid A=a,X\}
=
\E\{k_V(Y)\mid A=a,X\}
=
m_a(X;V).
\]
Therefore \(\mu_{a,V}=\E[m_a(X;V)]\), and subtracting the two arm-specific identities yields
\[
\mu_V
=
\mu_{1,V}-\mu_{0,V}
=
\E\{m_1(X;V)-m_0(X;V)\}.
\]
\end{proof}

\subsection{A basic doubly robust identity}

\begin{lemma}[Doubly robust algebra]
\label{lem:dr_basic}
Fix \(a\in\{0,1\}\), let \(\eta=(\pi,r_0,r_1)\), and define
\[
D_a(Z;\eta):=\phi_V^a(Z;\eta)-\phi_V^a(Z;\eta_0),
\qquad
\eta_0=(\pi_0,m_0,m_1).
\]
Then
\begin{align}
D_a(Z;\eta)
&=
\left(1-\frac{\1\{A=a\}}{\pi_0(a\mid X)}\right)
\{r_a(X;V)-m_a(X;V)\}
\notag\\
&\quad
+
\1\{A=a\}
\left\{
\frac{1}{\pi(a\mid X)}
-
\frac{1}{\pi_0(a\mid X)}
\right\}
\{k_V(Y)-r_a(X;V)\},
\label{eq:dr_basic_decomp}
\end{align}
and
\begin{equation}
\label{eq:dr_basic_mean}
P_0D_a(\cdot;\eta)
=
P_0\left[
\frac{\pi_0(a\mid X)-\pi(a\mid X)}{\pi(a\mid X)}
\{m_a(X;V)-r_a(X;V)\}
\right].
\end{equation}
Moreover, on any event on which \(\pi(a\mid X)\ge \underline\pi>0\) almost surely and
\(\sup_x\|r_a(x;V)\|\le M\),
\begin{equation}
\label{eq:dr_basic_l2}
P_0\|D_a(\cdot;\eta)\|^2
\le
C\left\{
\|r_a(\cdot;V)-m_a(\cdot;V)\|_{L_2(P_X;\R^J)}^2
+
\|\pi(a\mid\cdot)-\pi_0(a\mid\cdot)\|_{L_2(P_X)}^2
\right\},
\end{equation}
where \(C<\infty\) depends only on the true positivity constant, \(\underline\pi\), \(B_V\), and \(M\).
\end{lemma}

\begin{proof}
Write \(I_a:=\1\{A=a\}\), \(\pi_a(X):=\pi(a\mid X)\), and \(\pi_{0,a}(X):=\pi_0(a\mid X)\). Then
\begin{align*}
D_a(Z;\eta)
&=
\frac{I_a}{\pi_a(X)}\{k_V(Y)-r_a(X;V)\}+r_a(X;V)
\\
&\quad
-
\frac{I_a}{\pi_{0,a}(X)}\{k_V(Y)-m_a(X;V)\}-m_a(X;V)
\\
&=
I_a\left\{\frac{1}{\pi_a(X)}-\frac{1}{\pi_{0,a}(X)}\right\}
\{k_V(Y)-r_a(X;V)\}
\\
&\quad
+
\left(1-\frac{I_a}{\pi_{0,a}(X)}\right)
\{r_a(X;V)-m_a(X;V)\},
\end{align*}
which proves \eqref{eq:dr_basic_decomp}.

Taking conditional expectation given \(X\), the second term in \eqref{eq:dr_basic_decomp} has mean zero. Also,
\[
\E\!\left[I_a\{k_V(Y)-r_a(X;V)\}\mid X\right]
=
\pi_{0,a}(X)\{m_a(X;V)-r_a(X;V)\}.
\]
Thus
\[
\E\{D_a(Z;\eta)\mid X\}
=
\frac{\pi_{0,a}(X)-\pi_a(X)}{\pi_a(X)}
\{m_a(X;V)-r_a(X;V)\},
\]
and \eqref{eq:dr_basic_mean} follows.

For the \(L_2\) bound, use \((u+v)^2\le 2u^2+2v^2\) in \eqref{eq:dr_basic_decomp}. True positivity implies
\[
\E\!\left[\left.\left|1-\frac{I_a}{\pi_{0,a}(X)}\right|^2\right|X\right]\le C,
\]
so the first squared term is bounded by
\[
C\,\|r_a(\cdot;V)-m_a(\cdot;V)\|_{L_2(P_X;\R^J)}^2.
\]
On the stated event,
\[
\left|
\frac{1}{\pi_a(X)}-\frac{1}{\pi_{0,a}(X)}
\right|
\le
C|\pi_a(X)-\pi_{0,a}(X)|,
\qquad
\|k_V(Y)-r_a(X;V)\|\le B_V+M.
\]
Hence the second squared term is bounded by
\[
C\,\|\pi(a\mid\cdot)-\pi_0(a\mid\cdot)\|_{L_2(P_X)}^2.
\]
Combining the two bounds proves \eqref{eq:dr_basic_l2}.
\end{proof}

\subsection{Proof of Proposition~\ref{prop:dr_canonical_gradient}}

\begin{proof}[Proof of Proposition~\ref{prop:dr_canonical_gradient}]
The bias identity is exactly \eqref{eq:dr_basic_mean}. Since
\[
\E\{\phi_V^a(Z;\eta)\}-\mu_{a,V}
=
P_0D_a(\cdot;\eta),
\]
we obtain
\[
\E\{\phi_V^a(Z;\eta)\}-\mu_{a,V}
=
\E\left[
\frac{\pi_0(a\mid X)-\pi(a\mid X)}{\pi(a\mid X)}
\{m_a(X;V)-r_a(X;V)\}
\right].
\]
Therefore \(\E\{\phi_V^a(Z;\eta)\}=\mu_{a,V}\) if either \(\pi(a\mid\cdot)=\pi_0(a\mid\cdot)\) or \(r_a(\cdot;V)=m_a(\cdot;V)\). Subtracting the two arm-specific identities gives the corresponding double-robustness statement for \(z_V^{\dr}\).

It remains to identify the canonical gradient. We prove the arm-specific statement; the contrast follows by linearity. Consider a regular path \(t\mapsto P_t\) through \(P_0\) with score \(s\in T_{P_0}\). Write
\[
m_{a,t}(x;V):=\E_t[k_V(Y)\mid A=a,X=x],
\qquad
\mu_{a,V}(P_t):=\E_t[m_{a,t}(X;V)].
\]
Let \(s_X(X):=\E[s(Z)\mid X]\). Standard conditional-score calculus gives
\[
\left.\frac{d}{dt}\right|_{t=0}\mu_{a,V}(P_t)
=
\E\!\left[m_a(X;V)s_X(X)\right]
+
\E\!\left[
\E\{(k_V(Y)-m_a(X;V))s(Z)\mid A=a,X\}
\right].
\]
The first term can be written as
\[
\E\!\left[m_a(X;V)s_X(X)\right]
=
\E\!\left[\{m_a(X;V)-\mu_{a,V}\}s(Z)\right],
\]
because \(\E[s(Z)]=0\). The second term can be represented in observed-data form as
\[
\E\!\left[
\frac{\1\{A=a\}}{\pi_0(a\mid X)}
\{k_V(Y)-m_a(X;V)\}s(Z)
\right].
\]
Hence
\[
\left.\frac{d}{dt}\right|_{t=0}\mu_{a,V}(P_t)
=
\E\!\left[
\left\{
\frac{\1\{A=a\}}{\pi_0(a\mid X)}
\{k_V(Y)-m_a(X;V)\}
+
m_a(X;V)-\mu_{a,V}
\right\}
s(Z)
\right].
\]
Thus the arm-specific canonical gradient is
\[
\phi_V^a(Z;\eta_0)-\mu_{a,V}.
\]
Since \(\mu_V=\mu_{1,V}-\mu_{0,V}\), the canonical gradient of \(\mu_V\) is
\[
\{\phi_V^1(Z;\eta_0)-\mu_{1,V}\}
-
\{\phi_V^0(Z;\eta_0)-\mu_{0,V}\}
=
z_V^{\dr}(Z;\eta_0)-\mu_V
=
\psi_V(Z).
\]
\end{proof}

\subsection{Proof of Theorem~\ref{thm:first_order}}

\begin{proof}[Proof of Theorem~\ref{thm:first_order}]
Let \(\mathcal I_1,\ldots,\mathcal I_K\) be the validation folds, let \(n_k:=|\mathcal I_k|\), and let \(P_{n,k}\) denote the empirical measure on fold \(k\). For \(i\in\mathcal I_k\), write
\[
\hat z_i:=z_V^{\dr}(Z_i;\hat\eta_k),
\qquad
z_0(Z_i):=z_V^{\dr}(Z_i;\eta_0),
\qquad
\psi_V(Z_i)=z_0(Z_i)-\mu_V.
\]
Because \(K\) is fixed, all \(o_p(\cdot)\) statements below may be read uniformly over \(k\).

Let \(\mathcal E_n\) be the event on which the fitted propensities are uniformly bounded away from zero and the fitted regressions are uniformly bounded, as required by Assumption~\ref{ass:fixedV_nuisance_full}. Then \(P_0(\mathcal E_n)\to1\). On this event, define
\[
d_{a,k}(Z):=\phi_V^a(Z;\hat\eta_k)-\phi_V^a(Z;\eta_0),
\qquad
d_k(Z):=d_{1,k}(Z)-d_{0,k}(Z),
\]
and
\[
\alpha_{a,k}:=
\|\hat\pi_k(a\mid\cdot)-\pi_0(a\mid\cdot)\|_{L_2(P_X)},
\qquad
\beta_{a,k}:=
\|\hat m_{a,k}(\cdot;V)-m_a(\cdot;V)\|_{L_2(P_X;\R^J)}.
\]
By Lemma~\ref{lem:dr_basic} and Cauchy--Schwarz,
\begin{align}
\|P_0d_k\|
&\le
C\sum_{a=0}^1\alpha_{a,k}\beta_{a,k},
\label{eq:dkmean_bound}
\\
P_0\|d_k\|^2
&\le
C\sum_{a=0}^1(\alpha_{a,k}^2+\beta_{a,k}^2).
\label{eq:dkL2_bound}
\end{align}
The assumed nuisance consistency and product-rate conditions therefore imply, uniformly over \(k\),
\[
\|P_0d_k\|=o_p(n^{-1/2}),
\qquad
P_0\|d_k\|^2=o_p(1).
\]

\paragraph{First-order expansion.}
Since
\[
\bar z_{n,V}^{\dr}
=
\sum_{k=1}^K\frac{n_k}{n}P_{n,k}(z_0+d_k),
\]
we have
\[
\sqrt n(\bar z_{n,V}^{\dr}-\mu_V)
=
\frac1{\sqrt n}\sum_{i=1}^n\psi_V(Z_i)
+
R_{1n}
+
R_{2n},
\]
where
\[
R_{1n}:=
\sqrt n\sum_{k=1}^K\frac{n_k}{n}(P_{n,k}-P_0)d_k,
\qquad
R_{2n}:=
\sqrt n\sum_{k=1}^K\frac{n_k}{n}P_0d_k.
\]
For \(R_{2n}\), \eqref{eq:dkmean_bound} and the product-rate condition give
\[
\|R_{2n}\|
\le
C\sqrt n\sum_{k=1}^K\frac{n_k}{n}\sum_{a=0}^1\alpha_{a,k}\beta_{a,k}
=
o_p(1).
\]

For \(R_{1n}\), condition on the training data used to construct \(\hat\eta_k\). Then \(d_k\) is fixed and independent of the validation observations in fold \(k\), so
\[
\E_0\!\left[
\left.
\left\|
\sqrt n\,\frac{n_k}{n}(P_{n,k}-P_0)d_k
\right\|^2
\right|\hat\eta_k
\right]
\le
\frac{n_k}{n}P_0\|d_k\|^2
\le
P_0\|d_k\|^2
=
o_p(1).
\]
A conditional Markov inequality gives each fold contribution as \(o_p(1)\). Since \(K\) is fixed, \(R_{1n}=o_p(1)\). Hence
\[
\sqrt n(\bar z_{n,V}^{\dr}-\mu_V)
=
\frac1{\sqrt n}\sum_{i=1}^n\psi_V(Z_i)+o_p(1).
\]

\paragraph{Covariance consistency.}
Let
\[
M_n:=\frac1n\sum_{i=1}^n\hat z_i\hat z_i^\top,
\qquad
M_0:=P_0[z_0z_0^\top].
\]
The ordinary law of large numbers gives
\[
\frac1n\sum_{i=1}^n z_0(Z_i)z_0(Z_i)^\top\to_p M_0,
\]
because \(z_0\) is bounded under the bounded-kernel and positivity assumptions. It remains to show that replacing \(z_0(Z_i)\) by \(\hat z_i\) is negligible. Define
\[
A_n
:=
\frac1n\sum_{i=1}^n\|\hat z_i-z_0(Z_i)\|^2
=
\sum_{k=1}^K\frac{n_k}{n}P_{n,k}\|d_k\|^2.
\]
Conditionally on the nuisance fits,
\[
\E_0[A_n\mid \hat\eta_1,\ldots,\hat\eta_K]
=
\sum_{k=1}^K\frac{n_k}{n}P_0\|d_k\|^2
=
o_p(1),
\]
so \(A_n=o_p(1)\) by conditional Markov. Therefore,
\begin{align*}
\left\|
M_n-\frac1n\sum_{i=1}^n z_0(Z_i)z_0(Z_i)^\top
\right\|_F
&\le
A_n
+
2A_n^{1/2}
\left(\frac1n\sum_{i=1}^n\|z_0(Z_i)\|^2\right)^{1/2}
\\
&=
o_p(1).
\end{align*}
Thus \(M_n\to_p M_0\). Since the first-order expansion implies \(\bar z_{n,V}^{\dr}\to_p\mu_V\),
\[
S_{n,V}^{\dr}
=
\frac{n}{n-1}
\left\{
M_n-\bar z_{n,V}^{\dr}\bar z_{n,V}^{\dr\top}
\right\}
\to_p
M_0-\mu_V\mu_V^\top
=
\Var_{P_0}\{z_V^{\dr}(Z;\eta_0)\}
=
\Sigma_V.
\]
If \(\Sigma_V\) is positive definite and \(\gamma_n\downarrow0\), then
\[
S_{n,V}^{\dr}+\gamma_n I_J\to_p\Sigma_V,
\]
and continuity of matrix inversion at positive definite matrices gives
\[
(S_{n,V}^{\dr}+\gamma_n I_J)^{-1}\to_p\Sigma_V^{-1}.
\]
\end{proof}




\section{Proofs for Section~\ref{sec:theory}}
\label{app:local_proofs}

Throughout this appendix, \(V\in\Ycal^J\) is fixed and \(P_0\in\Pcal\) satisfies
\[
\mu_V(P_0)=0.
\]
Let \(t\mapsto P_{t,g}\) be a quadratic-mean differentiable regular path
through \(P_0\) with score \(g\in T_{P_0}\), and let
\[
P_{n,h,g}:=P_{h/\sqrt n,g}^{\otimes n}.
\]
By the standard QMD-to-LAN implication for i.i.d. experiments
\citep[Theorem~7.2]{vandervaart1998asymptotic},
\begin{equation}
\label{eq:LAN_local_proofs}
\ell_{n,h,g}
:=
\log\frac{dP_{n,h,g}}{dP_0^{\otimes n}}
=
h\,\mathbb G_n g
-
\frac12h^2\|g\|_{L_2(P_0)}^2
+
o_{P_0}(1),
\end{equation}
where
\[
\mathbb G_n g
:=
\frac1{\sqrt n}\sum_{i=1}^n g(Z_i),
\]
and \(P_{n,h,g}\) is contiguous with respect to \(P_0^{\otimes n}\). Since
\(g\in T_{P_0}\subset L_2^0(P_0)\), \(\E_{P_0}\{g(Z)\}=0\).

\subsection{A contiguity transfer lemma}

\begin{lemma}[Contiguity transfer]
\label{lem:contiguity_transfer}
Let \(Q_n\) be contiguous with respect to \(P_0^{\otimes n}\). If
\(R_n=o_{P_0}(1)\), then \(R_n=o_{Q_n}(1)\). If \(A_n\to_{P_0}A\), then
\(A_n\to_{Q_n}A\). In particular, if \(A\) is positive definite and
\(\gamma_n\downarrow0\), then
\[
(A_n+\gamma_n I)^{-1}\to_{Q_n} A^{-1}.
\]
\end{lemma}

\begin{proof}
For any \(\epsilon>0\),
\[
P_0^{\otimes n}(\|R_n\|>\epsilon)\to0.
\]
Contiguity implies
\[
Q_n(\|R_n\|>\epsilon)\to0,
\]
so \(R_n=o_{Q_n}(1)\). The convergence \(A_n\to A\) transfers by applying the
same argument to the events
\[
\{\|A_n-A\|>\epsilon\}.
\]
Since \(\gamma_n\downarrow0\), \(A_n+\gamma_n I\to_{Q_n}A\). The inverse
statement follows from continuity of matrix inversion at positive definite
matrices.
\end{proof}

\subsection{Proof of Theorem~\ref{thm:efficient_local_testing}}

\begin{proof}[Proof of Theorem~\ref{thm:efficient_local_testing}]
Define
\[
W_n:=\frac1{\sqrt n}\sum_{i=1}^n\psi_V(Z_i).
\]
Since \(\mu_V(P_0)=0\), the first-order representation in
Theorem~\ref{thm:first_order} gives, under \(P_0^{\otimes n}\),
\begin{equation}
\label{eq:AL_transfer_proof}
\sqrt n\,\bar z_{n,V}^{\dr}
=
W_n+r_n,
\qquad
r_n=o_{P_0}(1),
\end{equation}
and
\[
S_{n,V}^{\dr}\to_{P_0}\Sigma_V.
\]
By QMD, \(P_{n,h,g}\) is contiguous with respect to \(P_0^{\otimes n}\).
Lemma~\ref{lem:contiguity_transfer} therefore implies
\begin{equation}
\label{eq:AL_under_local}
r_n=o_{P_{n,h,g}}(1),
\qquad
S_{n,V}^{\dr}\to_{P_{n,h,g}}\Sigma_V,
\end{equation}
and, since \(\Sigma_V\succ0\) and \(\gamma_n\downarrow0\),
\[
(S_{n,V}^{\dr}+\gamma_n I_J)^{-1}\to_{P_{n,h,g}}\Sigma_V^{-1}.
\]

We next determine the local law of \(W_n\). Under \(P_0^{\otimes n}\), the
joint vector
\[
(W_n,\mathbb G_n g)
=
\left(
\frac1{\sqrt n}\sum_{i=1}^n\psi_V(Z_i),
\frac1{\sqrt n}\sum_{i=1}^n g(Z_i)
\right)
\]
converges by the multivariate central limit theorem to a centered Gaussian
vector \((W,G)\) with covariance
\[
\begin{pmatrix}
\Sigma_V & \eta_V(g) \\
\eta_V(g)^\top & \|g\|_{L_2(P_0)}^2
\end{pmatrix}.
\]
Indeed,
\[
\Cov_{P_0}\{\psi_V(Z),g(Z)\}
=
\E_{P_0}\{\psi_V(Z)g(Z)\}
=
\eta_V(g).
\]
Combining this joint central limit theorem with the LAN expansion
\eqref{eq:LAN_local_proofs}, Le Cam's third lemma gives, under \(P_{n,h,g}\),
\[
W_n\dto N(h\eta_V(g),\Sigma_V).
\]
Together with \eqref{eq:AL_under_local}, this yields
\[
\sqrt n\,\bar z_{n,V}^{\dr}
\dto
N(h\eta_V(g),\Sigma_V)
\qquad
\text{under }P_{n,h,g}.
\]

For the quadratic statistic, write
\[
\hat\lambda_{n,V}^{\dr}
=
(\sqrt n\,\bar z_{n,V}^{\dr})^\top
(S_{n,V}^{\dr}+\gamma_n I_J)^{-1}
(\sqrt n\,\bar z_{n,V}^{\dr}).
\]
By Slutsky's theorem,
\[
\hat\lambda_{n,V}^{\dr}
\dto
Y^\top\Sigma_V^{-1}Y,
\qquad
Y\sim N(h\eta_V(g),\Sigma_V).
\]
Equivalently, if \(Z\sim N(0,I_J)\), then
\[
Y=\Sigma_V^{1/2}Z+h\eta_V(g),
\]
and therefore
\[
Y^\top\Sigma_V^{-1}Y
=
\left\|
Z+h\Sigma_V^{-1/2}\eta_V(g)
\right\|^2.
\]
Thus
\[
\hat\lambda_{n,V}^{\dr}
\dto
\chi^2_J\!\left(
h^2\eta_V(g)^\top\Sigma_V^{-1}\eta_V(g)
\right)
=
\chi^2_J(h^2\lambda_V(g)).
\]
The whitened representation follows from the same argument:
\[
\Sigma_V^{-1/2}\sqrt n\,\bar z_{n,V}^{\dr}
\dto
N(h\Sigma_V^{-1/2}\eta_V(g),I_J),
\]
and
\[
\hat\lambda_{n,V}^{\dr}
=
\left\|
\Sigma_V^{-1/2}\sqrt n\,\bar z_{n,V}^{\dr}
\right\|^2
+
o_{P_{n,h,g}}(1).
\]
\end{proof}

\subsection{Proof of Proposition~\ref{cor:local_power_criterion}}

\begin{proof}[Proof of Proposition~\ref{cor:local_power_criterion}]
Let \(t_n:=h/\sqrt n\). Since \(\mu_V\) is pathwise differentiable at \(P_0\)
along \(t\mapsto P_{t,g}\), with derivative \(\eta_V(g)\), and since
\(\mu_V(P_0)=0\),
\[
\mu_V(P_{t,g})
=
t\,\eta_V(g)+o(t)
\qquad
\text{as }t\to0.
\]
Substituting \(t=t_n\) gives
\[
\mu_V(P_{h/\sqrt n,g})
=
\frac{h}{\sqrt n}\eta_V(g)+o(n^{-1/2}).
\]

Let
\[
\Sigma_n:=\Sigma_V(P_{t_n,g}).
\]
By continuity along the path and nonsingularity in a neighborhood of \(P_0\),
\[
\Sigma_n\to\Sigma_V(P_0)=\Sigma_V,
\qquad
\Sigma_n^{-1}\to\Sigma_V^{-1}.
\]
Set
\[
u_n:=\sqrt n\,\mu_V(P_{t_n,g}).
\]
Then
\[
u_n=h\eta_V(g)+o(1).
\]
Therefore
\[
n\,\Lambda_V(P_{t_n,g})
=
u_n^\top\Sigma_n^{-1}u_n
\to
h^2\,\eta_V(g)^\top\Sigma_V^{-1}\eta_V(g)
=
h^2\lambda_V(g).
\]
\end{proof}

\subsection{Proof of Corollary~\ref{cor:split_calibration}}

\begin{proof}[Proof of Corollary~\ref{cor:split_calibration}]
Under the global null \(H_0:\Delta=0\), every finite-location signal is zero:
\[
\mu_V(P_0)=0
\qquad
\text{for all }V\in\Ycal^J.
\]
In particular, conditional on the learning split,
\[
\mu_{\hat V}(P_0)=0.
\]
Conditional on the learning split, the selected \(\hat V\) is fixed relative to
the final testing data. The final statistic is therefore a fixed-location
statistic at \(V=\hat V\), computed with \(\gamma_{n_{\te}}\downarrow0\). By
the assumed conditional fixed-location first-order representation and covariance
consistency, together with the multivariate central limit theorem and Slutsky's
theorem,
\[
\hat\lambda_{n_{\te},\hat V}^{\dr}
\dto
\chi^2_J
\qquad
\text{conditionally on the learning split}.
\]
Thus the conditional rejection probability converges to \(\alpha\) at the
chi-square critical value. Since the rejection indicator is bounded and the
limiting rejection probability is nonrandom, the unconditional rejection
probability also converges to \(\alpha\).
\end{proof}

\section{Additional local-theory refinements}
\label{app:local_results}

Throughout this appendix, fix \(V\in\Ycal^J\) and let \(P_0\in\Pcal\) satisfy
\[
\mu_V(P_0)=0.
\]
Write \(T_{P_0}\) for the tangent space at \(P_0\). For \(g\in T_{P_0}\), let
\[
P_{n,h,g}:=P_{h/\sqrt n,g}^{\otimes n}
\]
denote the corresponding contiguous local alternatives. Recall that
\[
\psi_V(Z)=z_V^{\dr}(Z;\eta_0)-\mu_V
\]
is the canonical gradient of the fixed-location signal \(\mu_V\), and define
\[
\Sigma_V:=\Var_{P_0}\{\psi_V(Z)\},
\qquad
\eta_V(g):=\E_{P_0}\{\psi_V(Z)g(Z)\},
\qquad
\lambda_V(g):=\eta_V(g)^\top\Sigma_V^{-1}\eta_V(g).
\]
The main text identifies the efficient finite-signal Gaussian experiment and the
omnibus Hotelling statistic. This appendix records the corresponding
known-direction benchmark, which explains the role of
\(\lambda_V(g)\) as a directional local signal-to-noise ratio.

\subsection{Known-direction score benchmark}
\label{app:directional_benchmark}

Theorem~\ref{thm:efficient_local_testing} treats the omnibus alternative
\(\mu_V\neq0\). If the local direction \(g\) were known, the finite-signal
Gaussian experiment also gives a one-dimensional Neyman--Pearson benchmark.

\begin{corollary}[Directional efficient score test]
\label{cor:directional_score_app}
Assume the conditions of Theorem~\ref{thm:efficient_local_testing}, and suppose
that \(\eta_V(g)\neq0\). Define
\[
L_{n,V,g}
:=
\frac{
\eta_V(g)^\top
\bigl(S_{n,V}^{\dr}+\gamma_n I_J\bigr)^{-1}
\sqrt n\,\bar z_{n,V}^{\dr}
}{
\left[
\eta_V(g)^\top
\bigl(S_{n,V}^{\dr}+\gamma_n I_J\bigr)^{-1}
\eta_V(g)
\right]^{1/2}
}.
\]
Then, under \(P_0^{\otimes n}\),
\[
L_{n,V,g}\dto N(0,1),
\]
and, under \(P_{n,h,g}\),
\[
L_{n,V,g}\dto N\bigl(h\sqrt{\lambda_V(g)},1\bigr).
\]
Consequently, the one-sided level-\(\alpha\) rule that rejects for
\[
L_{n,V,g}>z_{1-\alpha}
\]
has asymptotic power
\[
1-\Phi\bigl(z_{1-\alpha}-h\sqrt{\lambda_V(g)}\bigr).
\]
Moreover, in the finite-signal Gaussian experiment
\[
Y\sim N(h\eta_V(g),\Sigma_V),
\qquad h\in\R,
\]
this directional rule is Neyman--Pearson optimal for testing \(h=0\) against any
fixed simple alternative \(h=h_1>0\).
\end{corollary}

\begin{proof}
Let
\[
A_n:=\bigl(S_{n,V}^{\dr}+\gamma_n I_J\bigr)^{-1},
\qquad
W_n:=\sqrt n\,\bar z_{n,V}^{\dr}.
\]
By Theorem~\ref{thm:efficient_local_testing},
\[
W_n\dto N\bigl(h\eta_V(g),\Sigma_V\bigr)
\qquad\text{under }P_{n,h,g},
\]
and
\[
A_n\to_p\Sigma_V^{-1}.
\]
Therefore, by Slutsky's theorem,
\[
\eta_V(g)^\top A_n W_n
=
\eta_V(g)^\top \Sigma_V^{-1}W_n
+
o_{P_{n,h,g}}(1).
\]
Hence
\[
\eta_V(g)^\top A_n W_n
\dto
N\!\left(
h\,\eta_V(g)^\top\Sigma_V^{-1}\eta_V(g),
\,
\eta_V(g)^\top\Sigma_V^{-1}\Sigma_V\Sigma_V^{-1}\eta_V(g)
\right),
\]
or equivalently,
\[
\eta_V(g)^\top A_n W_n
\dto
N\bigl(h\lambda_V(g),\lambda_V(g)\bigr).
\]
Also,
\[
\eta_V(g)^\top A_n\eta_V(g)
\to_p
\eta_V(g)^\top\Sigma_V^{-1}\eta_V(g)
=
\lambda_V(g).
\]
Dividing numerator and denominator yields
\[
L_{n,V,g}\dto N\bigl(h\sqrt{\lambda_V(g)},1\bigr).
\]
The null statement is obtained by setting \(h=0\), and the displayed power
formula follows immediately.

For the Neyman--Pearson statement, in the finite-signal Gaussian experiment
\[
Y\sim N(h\eta_V(g),\Sigma_V),
\]
the log-likelihood ratio between \(h=h_1\) and \(h=0\) is, up to an additive
constant,
\[
h_1\,\eta_V(g)^\top\Sigma_V^{-1}Y
-
\frac12 h_1^2\,\eta_V(g)^\top\Sigma_V^{-1}\eta_V(g).
\]
Thus the most powerful level-\(\alpha\) test rejects for large values of
\[
\eta_V(g)^\top\Sigma_V^{-1}Y,
\]
or equivalently for large values of its standardized version. This is the
limiting test generated by \(L_{n,V,g}\).
\end{proof}

\subsection{Interpretation}

For any scalar contrast \(a^\top\mu_V\), the squared local signal-to-noise ratio
along \(g\) is
\[
\frac{(a^\top\eta_V(g))^2}{a^\top\Sigma_Va}.
\]
The known-direction benchmark optimizes this quantity over \(a\neq0\), yielding
the Rayleigh quotient
\[
\lambda_V(g)=\eta_V(g)^\top\Sigma_V^{-1}\eta_V(g).
\]
When \(g\) is unknown, no uniformly most powerful test exists over all drift
directions in the multivariate Gaussian shift experiment. The fixed-location
DR-ME statistic instead uses the omnibus quadratic statistic in the whitened
efficient Gaussian experiment, whose local power is governed by the same
noncentrality \(h^2\lambda_V(g)\). This is why the population criterion
\[
\Lambda_V(P)=\mu_V(P)^\top\Sigma_V(P)^{-1}\mu_V(P)
\]
is the relevant target for learning informative locations.
\section{Proofs for Section~\ref{sec:learning}}
\label{app:learning_proofs}

Throughout this appendix, write \(n:=n_{\tr}\) and
\[
D_\eta
:=
\sigma\bigl((Z_i)_{i\in I_\eta},\hat\pi,\hat m_0,\hat m_1\bigr).
\]
For \(V\in\mathcal V\) and \(z=(x,a,y)\in\Zcal\), define the nuisance-fitted pseudo-feature
\[
\hat z_V(z):=z_V^{\dr}(z;\hat\eta),
\qquad
\hat\eta=(\hat\pi,\hat m_0,\hat m_1).
\]
Thus, for \(i\in I_{\tr}\), \(\hat z_{i,V}^{\dr}=\hat z_V(Z_i)\). Conditional on \(D_\eta\), the variables \(\{\hat z_V(Z_i):i\in I_{\tr}\}\) are i.i.d. for each fixed \(V\).

Define the conditional population quantities
\[
\mu_V^\eta:=\E[\hat z_V(Z)\mid D_\eta],
\qquad
M_V^\eta:=\E[\hat z_V(Z)\hat z_V(Z)^\top\mid D_\eta],
\qquad
\Sigma_V^\eta:=M_V^\eta-\mu_V^\eta\mu_V^{\eta\top},
\]
and the conditional criterion
\[
\mathcal P_\tau^\eta(V)
:=
\mu_V^{\eta\top}(\Sigma_V^\eta+\tau I_J)^{-1}\mu_V^\eta.
\]
The proof of Theorem~\ref{thm:location_learning} uses
\begin{equation}
\label{eq:criterion_decomp_main}
\sup_{V\in\mathcal V}
\bigl|
\hat{\mathcal P}_{\tau,\tr}(V)-\mathcal P_\tau(V)
\bigr|
\le
\sup_{V\in\mathcal V}
\bigl|
\hat{\mathcal P}_{\tau,\tr}(V)-\mathcal P_\tau^\eta(V)
\bigr|
+
\sup_{V\in\mathcal V}
\bigl|
\mathcal P_\tau^\eta(V)-\mathcal P_\tau(V)
\bigr|.
\end{equation}
The first term is the training-split empirical fluctuation. The second term is the nuisance-induced discrepancy between the conditional and oracle population criteria.

\subsection{Uniform bounds for pseudo-features}

Let
\[
\bar\varepsilon:=\min\{\varepsilon,\underline\varepsilon\},
\]
where \(\varepsilon\) is the true positivity constant and \(\underline\varepsilon\) is the fitted-propensity lower bound in Assumption~\ref{ass:loclearn_euclid_full}.

\begin{lemma}[Uniform boundedness and Lipschitzness]
\label{lem:pseudofeature_bounds}
Under Assumption~\ref{ass:loclearn_euclid_full}, there exist finite constants
\[
B_z:=2\sqrt J\Bigl[\bar\varepsilon^{-1}(B_k+B_m)+B_m\Bigr],
\qquad
L_z:=2\Bigl[\bar\varepsilon^{-1}(L_k+L_m)+L_m\Bigr],
\]
such that, with probability tending to one, for all \(z\in\Zcal\) and all \(V,V'\in\mathcal V\),
\[
\|\hat z_V(z)\|\le B_z,
\qquad
\|\hat z_V(z)-\hat z_{V'}(z)\|
\le
L_z\|V-V'\|.
\]
The same bounds hold for the oracle feature \(z_V^0(z):=z_V^{\dr}(z;\eta_0)\).
\end{lemma}

\begin{proof}
For \(V=(v_1,\ldots,v_J)\),
\[
\|k_V(y)\|^2
=
\sum_{j=1}^J |k_Y(v_j,y)|^2
\le
J B_k^2,
\qquad
\|\hat m_a(x;V)\|^2
\le
J B_m^2.
\]
Hence the arm-specific fitted score is bounded by
\[
\left\|
\frac{\1\{A=a\}}{\hat\pi(a\mid X)}
\{k_V(Y)-\hat m_a(X;V)\}
+
\hat m_a(X;V)
\right\|
\le
\sqrt J\Bigl[\bar\varepsilon^{-1}(B_k+B_m)+B_m\Bigr].
\]
Subtracting the two arm-specific scores gives the bound for \(\hat z_V\).

For Lipschitzness, Assumption~\ref{ass:loclearn_euclid_full} gives
\[
\|k_V(y)-k_{V'}(y)\|\le L_k\|V-V'\|,
\qquad
\|\hat m_a(x;V)-\hat m_a(x;V')\|\le L_m\|V-V'\|.
\]
Therefore each arm-specific fitted score is Lipschitz with constant
\[
\bar\varepsilon^{-1}L_k+(\bar\varepsilon^{-1}+1)L_m
\le
\bar\varepsilon^{-1}(L_k+L_m)+L_m.
\]
Subtracting the two arms gives the displayed \(L_z\). The oracle feature satisfies the same bounds because the true regressions inherit boundedness and Lipschitzness from \(k_Y\), and the true propensities are bounded below by \(\varepsilon\).
\end{proof}

\subsection{Uniform empirical control}

\begin{lemma}[Uniform concentration for bounded Lipschitz Euclidean classes]
\label{lem:euclidean_empirical}
Let \(\mathcal V\subset[-R,R]^m\), and let
\[
\mathcal F:=\{f_V:V\in\mathcal V\}
\]
be a class of measurable real-valued functions such that, conditionally on \(D_\eta\),
\[
|f_V(z)|\le B,
\qquad
|f_V(z)-f_{V'}(z)|\le L\|V-V'\|
\]
for all \(z\) and all \(V,V'\). Then
\[
\sup_{V\in\mathcal V}
\bigl|
(P_n^{\tr}-P_0)f_V
\bigr|
=
O_p\!\left(
\sqrt{\frac{m\log n}{n}}
\right),
\]
where \(P_n^{\tr}:=n^{-1}\sum_{i\in I_{\tr}}\delta_{Z_i}\), and the hidden constant depends only on \(R,B,L,m\).
\end{lemma}

\begin{proof}
Let \(\mathcal N_\delta\) be a \(\delta\)-net of \(\mathcal V\). Since \(\mathcal V\subset[-R,R]^m\),
\[
|\mathcal N_\delta|\le\left(\frac{C_0R}{\delta}\right)^m
\]
for a universal constant \(C_0\). For each \(V\), choose \(\Pi_\delta(V)\in\mathcal N_\delta\) with \(\|V-\Pi_\delta(V)\|\le\delta\). Then
\[
\sup_{V\in\mathcal V}|(P_n^{\tr}-P_0)f_V|
\le
\max_{U\in\mathcal N_\delta}|(P_n^{\tr}-P_0)f_U|
+
2L\delta.
\]
Conditionally on \(D_\eta\), Hoeffding's inequality and a union bound give
\[
\Pr\!\left(
\left.
\max_{U\in\mathcal N_\delta}|(P_n^{\tr}-P_0)f_U|>t
\,\right|\,D_\eta
\right)
\le
2|\mathcal N_\delta|\exp\!\left(-\frac{nt^2}{2B^2}\right).
\]
Taking \(\delta=n^{-1/2}\) and \(t=M\sqrt{m\log n/n}\), with \(M\) large enough, proves the result.
\end{proof}

\begin{lemma}[Uniform mean and covariance concentration]
\label{lem:mean_moment_uniform}
Under Assumption~\ref{ass:loclearn_euclid_full},
\[
\sup_{V\in\mathcal V}
\|\bar z_{\tr,V}^{\dr}-\mu_V^\eta\|
=
O_p\!\left(
\sqrt{\frac{Jd\log n}{n}}
\right),
\]
and, with
\[
\widehat M_{\tr,V}
:=
\frac1n\sum_{i\in I_{\tr}}\hat z_{i,V}^{\dr}\hat z_{i,V}^{\dr\top},
\]
\[
\sup_{V\in\mathcal V}
\|\widehat M_{\tr,V}-M_V^\eta\|_F
=
O_p\!\left(
\sqrt{\frac{Jd\log n}{n}}
\right).
\]
Consequently,
\[
\sup_{V\in\mathcal V}
\|S_{\tr,V}^{\dr}-\Sigma_V^\eta\|_F
=
O_p\!\left(
\sqrt{\frac{Jd\log n}{n}}
\right).
\]
\end{lemma}

\begin{proof}
Set \(m=Jd\). By Lemma~\ref{lem:pseudofeature_bounds}, each coordinate class
\[
\{z\mapsto \hat z_{V,j}(z):V\in\mathcal V\}
\]
is conditionally bounded and Lipschitz. Applying Lemma~\ref{lem:euclidean_empirical} coordinatewise and using fixed \(J\) yields the mean bound.

For second moments, apply the same lemma to
\[
\{z\mapsto \hat z_{V,j}(z)\hat z_{V,\ell}(z):V\in\mathcal V\}.
\]
This class is conditionally bounded by \(B_z^2\) and Lipschitz with constant \(2B_zL_z\). Applying the scalar concentration bound entrywise and using fixed \(J\) gives the second-moment bound.

Finally,
\[
S_{\tr,V}^{\dr}
=
\frac{n}{n-1}
\left(
\widehat M_{\tr,V}
-
\bar z_{\tr,V}^{\dr}\bar z_{\tr,V}^{\dr\top}
\right),
\qquad
\Sigma_V^\eta=M_V^\eta-\mu_V^\eta\mu_V^{\eta\top}.
\]
The uniform boundedness of \(\bar z_{\tr,V}^{\dr}\) and \(\mu_V^\eta\), together with
\[
\|uu^\top-vv^\top\|_F
\le
(\|u\|+\|v\|)\|u-v\|,
\]
gives the covariance bound.
\end{proof}

\subsection{Nuisance transfer}

\begin{lemma}[Uniform nuisance transfer]
\label{lem:nuisance_transfer}
Under Assumption~\ref{ass:loclearn_euclid_full},
\[
\sup_{V\in\mathcal V}\|\mu_V^\eta-\mu_V\|
=
O_p(\rho_{n_\eta}),
\qquad
\sup_{V\in\mathcal V}\|\Sigma_V^\eta-\Sigma_V\|_F
=
O_p(\rho_{n_\eta}).
\]
\end{lemma}

\begin{proof}
For \(V\in\mathcal V\), write
\[
z_V^0(Z):=z_V^{\dr}(Z;\eta_0),
\qquad
d_V(Z):=\hat z_V(Z)-z_V^0(Z).
\]
Then \(\mu_V^\eta-\mu_V=P_0d_V\). For each arm \(a\), let
\[
D_{a,V}(Z):=\phi_V^a(Z;\hat\eta)-\phi_V^a(Z;\eta_0).
\]
Lemma~\ref{lem:dr_basic} gives, uniformly in \(V\),
\[
P_0\|D_{a,V}\|^2
\le
C\left\{
\|\hat\pi(a\mid\cdot)-\pi_0(a\mid\cdot)\|_{L_2(P_X)}^2
+
\|\hat m_a(\cdot;V)-m_a(\cdot;V)\|_{L_2(P_X;\R^J)}^2
\right\}.
\]
Since
\[
\|\hat m_a(\cdot;V)-m_a(\cdot;V)\|_{L_2(P_X;\R^J)}^2
\le
J
\sup_{v\in\Pi(\mathcal V)}
\|\hat m_a(\cdot;v)-m_a(\cdot;v)\|_{L_2(P_X)}^2,
\]
and \(d_V=D_{1,V}-D_{0,V}\), we obtain
\[
\sup_{V\in\mathcal V}P_0\|d_V\|^2
\le
C\rho_{n_\eta}^2.
\]
Thus
\[
\sup_{V\in\mathcal V}\|\mu_V^\eta-\mu_V\|
\le
\sup_{V\in\mathcal V}(P_0\|d_V\|^2)^{1/2}
=
O_p(\rho_{n_\eta}).
\]

For covariance, let \(M_V^0:=P_0[z_V^0(Z)z_V^0(Z)^\top]\). Since
\[
\hat z_V\hat z_V^\top-z_V^0z_V^{0\top}
=
d_Vd_V^\top+d_Vz_V^{0\top}+z_V^0d_V^\top,
\]
and \(\sup_{V,z}\|z_V^0(z)\|\le B_z\),
\[
\sup_{V\in\mathcal V}\|M_V^\eta-M_V^0\|_F
=
O_p(\rho_{n_\eta}).
\]
Finally,
\[
\Sigma_V^\eta-\Sigma_V
=
(M_V^\eta-M_V^0)
-
(\mu_V^\eta\mu_V^{\eta\top}-\mu_V\mu_V^\top),
\]
and the outer-product term is \(O_p(\rho_{n_\eta})\) uniformly because the means are uniformly bounded. This proves the result.
\end{proof}

\subsection{Perturbation of the ridge criterion}

\begin{lemma}[Perturbation bound]
\label{lem:criterion_perturb}
Fix \(\tau>0\) and define
\[
Q_\tau(m,A):=m^\top(A+\tau I_J)^{-1}m
\]
for \(m\in\R^J\) and symmetric positive semidefinite \(A\in\R^{J\times J}\). If
\[
\|m\|\vee\|\tilde m\|\le B,
\]
then
\[
|Q_\tau(m,A)-Q_\tau(\tilde m,\tilde A)|
\le
C(B,\tau)
\bigl(
\|m-\tilde m\|
+
\|A-\tilde A\|_F
\bigr),
\]
where one may take \(C(B,\tau)=2B/\tau+B^2/\tau^2\).
\end{lemma}

\begin{proof}
Let \(B_A:=A+\tau I_J\) and \(B_{\tilde A}:=\tilde A+\tau I_J\). Since \(A,\tilde A\succeq0\),
\[
\|B_A^{-1}\|_{\op}\le\tau^{-1},
\qquad
\|B_{\tilde A}^{-1}\|_{\op}\le\tau^{-1}.
\]
Decompose
\[
Q_\tau(m,A)-Q_\tau(\tilde m,\tilde A)
=
m^\top B_A^{-1}m-\tilde m^\top B_A^{-1}\tilde m
+
\tilde m^\top(B_A^{-1}-B_{\tilde A}^{-1})\tilde m.
\]
The first term is bounded by \(2B\tau^{-1}\|m-\tilde m\|\). For the second, use the resolvent identity
\[
B_A^{-1}-B_{\tilde A}^{-1}
=
B_A^{-1}(\tilde A-A)B_{\tilde A}^{-1},
\]
which gives
\[
|\tilde m^\top(B_A^{-1}-B_{\tilde A}^{-1})\tilde m|
\le
B^2\tau^{-2}\|A-\tilde A\|_F.
\]
Combining the bounds proves the lemma.
\end{proof}

\subsection{Proof of Theorem~\ref{thm:location_learning}}

\begin{proof}[Proof of Theorem~\ref{thm:location_learning}]
By Lemma~\ref{lem:criterion_perturb},
\[
\sup_{V\in\mathcal V}
\bigl|
\hat{\mathcal P}_{\tau,\tr}(V)-\mathcal P_\tau^\eta(V)
\bigr|
\le
C
\left[
\sup_{V\in\mathcal V}
\|\bar z_{\tr,V}^{\dr}-\mu_V^\eta\|
+
\sup_{V\in\mathcal V}
\|S_{\tr,V}^{\dr}-\Sigma_V^\eta\|_F
\right].
\]
Lemma~\ref{lem:mean_moment_uniform} gives
\[
\sup_{V\in\mathcal V}
\bigl|
\hat{\mathcal P}_{\tau,\tr}(V)-\mathcal P_\tau^\eta(V)
\bigr|
=
O_p\!\left(
\sqrt{\frac{Jd\log n_{\tr}}{n_{\tr}}}
\right).
\]
Similarly, Lemmas~\ref{lem:criterion_perturb} and~\ref{lem:nuisance_transfer} imply
\[
\sup_{V\in\mathcal V}
\bigl|
\mathcal P_\tau^\eta(V)-\mathcal P_\tau(V)
\bigr|
=
O_p(\rho_{n_\eta}).
\]
Combining these two bounds with \eqref{eq:criterion_decomp_main} yields
\[
\Delta_{\tr}
:=
\sup_{V\in\mathcal V}
\left|
\hat{\mathcal P}_{\tau,\tr}(V)-\mathcal P_\tau(V)
\right|
=
O_p\!\left(
\sqrt{\frac{Jd\log n_{\tr}}{n_{\tr}}}
+
\rho_{n_\eta}
\right).
\]

It remains to prove the deterministic optimization inequality. Let
\[
V_\tau^\star\in\argmax_{V\in\mathcal V}\mathcal P_\tau(V).
\]
By definition of the empirical optimization gap,
\[
\hat{\mathcal P}_{\tau,\tr}(V_\tau^\star)
-
\hat{\mathcal P}_{\tau,\tr}(\hat V)
\le
\varepsilon_{\tr}(\hat V).
\]
Therefore
\begin{align*}
\mathcal P_\tau(V_\tau^\star)-\mathcal P_\tau(\hat V)
&=
\{\mathcal P_\tau(V_\tau^\star)-\hat{\mathcal P}_{\tau,\tr}(V_\tau^\star)\}
\\
&\quad+
\{\hat{\mathcal P}_{\tau,\tr}(V_\tau^\star)-\hat{\mathcal P}_{\tau,\tr}(\hat V)\}
\\
&\quad+
\{\hat{\mathcal P}_{\tau,\tr}(\hat V)-\mathcal P_\tau(\hat V)\}
\\
&\le
2\Delta_{\tr}+\varepsilon_{\tr}(\hat V).
\end{align*}
If \(\varepsilon_{\tr}(\hat V)=O_p(e_{\tr})\), the stochastic near-optimality statement follows immediately.
\end{proof}

\section{Finite-dictionary location learning}
\label{app:finite_dictionary}

For structured outcomes, Euclidean optimization may be inappropriate. A finite dictionary provides a direct alternative. Let
\[
\mathcal C=\{c_1,\ldots,c_M\}\subset\Ycal,
\qquad
\mathcal V_J(\mathcal C)\subset\mathcal C^J.
\]
For any output \(\hat V\in\mathcal V_J(\mathcal C)\), define its empirical dictionary optimization gap
\[
\varepsilon_{\tr}^{\mathcal C}(\hat V)
:=
\sup_{V\in\mathcal V_J(\mathcal C)}
\hat{\mathcal P}_{\tau,\tr}(V)
-
\hat{\mathcal P}_{\tau,\tr}(\hat V).
\]
For exhaustive search over the dictionary, this gap is zero.

Assume there exist constants \(B_k,B_m,\varepsilon,\tau>0\) such that:
\begin{enumerate}[label=(\roman*),leftmargin=1.2cm]
    \item for all \(c\in\mathcal C\) and \(y\in\Ycal\), \(|k_Y(c,y)|\le B_k\);
    \item for all \(a\in\{0,1\}\) and \(x\in\Xcal\), both \(\pi_0(a\mid x)\) and \(\hat\pi(a\mid x)\) are at least \(\varepsilon\);
    \item for all \(a,x,c\), \(|m_a(x;c)|\vee|\hat m_a(x;c)|\le B_m\).
\end{enumerate}
Define
\[
\rho^{\mathcal C}_{n_\eta}
:=
\sum_{a=0}^1
\left[
\|\hat\pi(a\mid\cdot)-\pi_0(a\mid\cdot)\|_{L_2(P_X)}
+
\max_{c\in\mathcal C}
\|\hat m_a(\cdot;c)-m_a(\cdot;c)\|_{L_2(P_X)}
\right].
\]

\begin{theorem}[Uniform consistency over a finite dictionary]
\label{thm:loclearn-finite}
Fix \(J\), suppose \(n_\eta,n_{\tr}\to\infty\), and assume
\(\rho^{\mathcal C}_{n_\eta}=o_p(1)\). Then
\[
\Delta_{\tr}^{\mathcal C}
:=
\sup_{V\in\mathcal V_J(\mathcal C)}
\left|
\hat{\mathcal P}_{\tau,\tr}(V)-\mathcal P_\tau(V)
\right|
=
O_p\!\left(
\sqrt{\frac{\log|\mathcal V_J(\mathcal C)|}{n_{\tr}}}
+
\rho^{\mathcal C}_{n_\eta}
\right).
\]
In particular, since \(|\mathcal V_J(\mathcal C)|\le M^J\),
\[
\Delta_{\tr}^{\mathcal C}
=
O_p\!\left(
\sqrt{\frac{J\log M}{n_{\tr}}}
+
\rho^{\mathcal C}_{n_\eta}
\right).
\]
Moreover, for any output \(\hat V\) and any
\[
V_{\tau,\mathcal C}^\star
\in
\argmax_{V\in\mathcal V_J(\mathcal C)}\mathcal P_\tau(V),
\]
\[
\mathcal P_\tau(V_{\tau,\mathcal C}^\star)-\mathcal P_\tau(\hat V)
\le
2\Delta_{\tr}^{\mathcal C}
+
\varepsilon_{\tr}^{\mathcal C}(\hat V).
\]
Thus, if \(\varepsilon_{\tr}^{\mathcal C}(\hat V)=O_p(e_{\tr}^{\mathcal C})\),
\[
\mathcal P_\tau(V_{\tau,\mathcal C}^\star)-\mathcal P_\tau(\hat V)
=
O_p\!\left(
\sqrt{\frac{J\log M}{n_{\tr}}}
+
\rho^{\mathcal C}_{n_\eta}
+
e_{\tr}^{\mathcal C}
\right).
\]
\end{theorem}

\begin{proof}
Let \(N_{\mathcal C}:=|\mathcal V_J(\mathcal C)|\). For \(i\in I_{\tr}\), define
\[
\hat z_{i,V}:=z_V^{\dr}(Z_i;\hat\eta),
\qquad
z_{i,V}^0:=z_V^{\dr}(Z_i;\eta_0).
\]
Conditionally on \(I_\eta\), the variables \(\{\hat z_{i,V}:i\in I_{\tr}\}\) are i.i.d. for each fixed \(V\).

The boundedness and positivity assumptions imply
\[
\|\hat z_{i,V}\|\vee\|z_{i,V}^0\|
\le
\sqrt J\,B_z
\]
for a constant \(B_z<\infty\) depending only on \(B_k,B_m,\varepsilon\). A union bound and Bernstein's inequality over the finite class give
\[
\sup_{V\in\mathcal V_J(\mathcal C)}
\|\bar z_{\tr,V}^{\dr}-\mu_V^{\hat\eta}\|
=
O_p\!\left(
\sqrt{\frac{\log N_{\mathcal C}}{n_{\tr}}}
\right),
\]
and the same entrywise argument for second moments gives
\[
\sup_{V\in\mathcal V_J(\mathcal C)}
\|S_{\tr,V}^{\dr}-\Sigma_V^{\hat\eta}\|_F
=
O_p\!\left(
\sqrt{\frac{\log N_{\mathcal C}}{n_{\tr}}}
\right).
\]

The nuisance transfer bound is the same as Lemma~\ref{lem:nuisance_transfer}, with \(\sup_{v\in\Pi(\mathcal V)}\) replaced by \(\max_{c\in\mathcal C}\). Hence
\[
\sup_{V\in\mathcal V_J(\mathcal C)}
\|\mu_V^{\hat\eta}-\mu_V\|
=
O_p(\rho^{\mathcal C}_{n_\eta}),
\qquad
\sup_{V\in\mathcal V_J(\mathcal C)}
\|\Sigma_V^{\hat\eta}-\Sigma_V\|_F
=
O_p(\rho^{\mathcal C}_{n_\eta}).
\]
The perturbation Lemma~\ref{lem:criterion_perturb} then yields
\[
\Delta_{\tr}^{\mathcal C}
=
O_p\!\left(
\sqrt{\frac{\log N_{\mathcal C}}{n_{\tr}}}
+
\rho^{\mathcal C}_{n_\eta}
\right).
\]
Since \(N_{\mathcal C}\le M^J\), the displayed simplified rate follows.

Finally, by definition of the dictionary optimization gap,
\[
\hat{\mathcal P}_{\tau,\tr}(V_{\tau,\mathcal C}^\star)
-
\hat{\mathcal P}_{\tau,\tr}(\hat V)
\le
\varepsilon_{\tr}^{\mathcal C}(\hat V).
\]
Adding and subtracting the empirical criterion gives
\[
\mathcal P_\tau(V_{\tau,\mathcal C}^\star)-\mathcal P_\tau(\hat V)
\le
2\Delta_{\tr}^{\mathcal C}
+
\varepsilon_{\tr}^{\mathcal C}(\hat V).
\]
The stochastic bound follows immediately.
\end{proof}

\section{Implementation details}
\label{app:implementation}

This appendix records the matrix implementation of the split-sample DR-ME test. Algorithm~\ref{alg:drme_appendix} gives the full testing pipeline. The remaining subsections define the matrices used to evaluate the training objective and the final statistic, and give the closed-form regression formulas used for the nuisance estimates.

\begin{algorithm}[h]
\caption{Split-sample DR-ME test with learned locations}
\label{alg:drme_appendix}
\resizebox{\linewidth}{!}{%
\begin{minipage}{\linewidth}
\small
\begin{algorithmic}[1]
\Require Data \(Z_{1:n}=\{(X_i,A_i,Y_i)\}_{i=1}^n\), search class \(\mathcal V\), number of locations \(J\), outcome kernel \(k_Y\), learning ridge \(\tau>0\), test ridge \(\gamma_{n_{\te}}>0\), level \(\alpha\)
\Ensure Learned locations \(\hat V\), statistic \(\hat\lambda_{n_{\te},\hat V}^{\dr}\), \(p\)-value \(p_{\te}^{\dr}\)

\State Split the sample into disjoint sets \(I_\eta\), \(I_{\tr}\), and \(I_{\te}\)
\State Fit \(\hat\eta=(\hat\pi,\hat m_0,\hat m_1)\) on \(I_\eta\)

\For{each candidate or optimizer evaluation \(V\in\mathcal V\)}
    \State Compute \(\hat z_{i,V}^{\dr}=z_V^{\dr}(Z_i;\hat\eta)\) for all \(i\in I_{\tr}\)
    \State Compute \(\bar z_{\tr,V}^{\dr}\) and \(S_{\tr,V}^{\dr}\)
    \State Compute
    \[
    \hat{\mathcal P}_{\tau,\tr}(V)
    =
    \bar z_{\tr,V}^{\dr\top}
    (S_{\tr,V}^{\dr}+\tau I_J)^{-1}
    \bar z_{\tr,V}^{\dr}
    \]
\EndFor

\State Select \(\hat V\in\mathcal V\) with small empirical gap
\[
\sup_{V\in\mathcal V}\hat{\mathcal P}_{\tau,\tr}(V)
-
\hat{\mathcal P}_{\tau,\tr}(\hat V)
\le
\varepsilon_{\tr}
\]

\State Compute \(\hat z_{i,\hat V}^{\dr}=z_{\hat V}^{\dr}(Z_i;\hat\eta)\) for all \(i\in I_{\te}\)
\State Compute \(\bar z_{\te,\hat V}^{\dr}\) and \(S_{\te,\hat V}^{\dr}\)
\State Compute
\[
\hat\lambda_{n_{\te},\hat V}^{\dr}
=
n_{\te}\,
\bar z_{\te,\hat V}^{\dr\top}
(S_{\te,\hat V}^{\dr}+\gamma_{n_{\te}} I_J)^{-1}
\bar z_{\te,\hat V}^{\dr}
\]
\State Set
\[
p_{\te}^{\dr}
=
1-F_{\chi^2_J}\!\left(\hat\lambda_{n_{\te},\hat V}^{\dr}\right)
\]
\State Reject if \(p_{\te}^{\dr}\le \alpha\)
\end{algorithmic}
\end{minipage}
}
\end{algorithm}

\subsection{Three-way split and matrix notation}

Let
\[
\{1,\dots,n\}=I_\eta\cup I_{\tr}\cup I_{\te}
\]
be the three-way split used for nuisance fitting, location learning, and final testing. For any evaluation set \(I=\{i_1,\dots,i_m\}\), let \(m=|I|\), let \(1_m\in\R^m\) be the all-ones vector, and define
\[
H_m:=I_m-\frac{1}{m}1_m1_m^\top .
\]

For \(V=(v_1,\dots,v_J)\in\Ycal^J\), define
\[
K_Y(Y_I,V)\in\R^{m\times J},
\qquad
[K_Y(Y_I,V)]_{rj}:=k_Y(v_j,Y_{i_r}).
\]
Thus the \(r\)-th row of \(K_Y(Y_I,V)\) is \(k_V(Y_{i_r})^\top\).

For \(a\in\{0,1\}\), define
\[
D_{a,I}
:=
\mathrm{diag}\!\left(
\frac{\1\{A_{i_1}=a\}}{\hat\pi(a\mid X_{i_1})},
\dots,
\frac{\1\{A_{i_m}=a\}}{\hat\pi(a\mid X_{i_m})}
\right).
\]
Let \(\hat M_{a,I}(V)\in\R^{m\times J}\) be the matrix of fitted arm-\(a\) regressions on \(I\), with
\[
[\hat M_{a,I}(V)]_{rj}
=
\hat m_a(X_{i_r};v_j).
\]

\subsection{Matrix form of the pseudo-features}

The arm-specific pseudo-feature matrices are
\[
\hat\Phi^a_I(V)
:=
D_{a,I}\bigl(K_Y(Y_I,V)-\hat M_{a,I}(V)\bigr)
+
\hat M_{a,I}(V),
\qquad a\in\{0,1\}.
\]
The doubly robust contrast matrix is
\[
\hat Z_I(V):=\hat\Phi^1_I(V)-\hat\Phi^0_I(V)\in\R^{m\times J}.
\]
Its \(r\)-th row is \(\hat z_{i_r,V}^{\dr\top}\). Therefore
\[
\bar z_I^{\dr}(V)
=
\frac{1}{m}\hat Z_I(V)^\top 1_m,
\qquad
S_I^{\dr}(V)
=
\frac{1}{m-1}\hat Z_I(V)^\top H_m \hat Z_I(V).
\]
The empirical learning criterion on \(I\) is
\[
\hat{\mathcal P}_{\tau,I}(V)
=
\bar z_I^{\dr}(V)^\top
\bigl(S_I^{\dr}(V)+\tau I_J\bigr)^{-1}
\bar z_I^{\dr}(V),
\]
and the final test statistic on \(I_{\te}\) is
\[
\hat\lambda_{n_{\te},V}^{\dr}
=
n_{\te}\,
\bar z_{I_{\te}}^{\dr}(V)^\top
\bigl(S_{I_{\te}}^{\dr}(V)+\gamma_{n_{\te}}I_J\bigr)^{-1}
\bar z_{I_{\te}}^{\dr}(V).
\]

\subsection{Outcome-regression matrices}

Suppose the nuisance regressions \(m_a(\cdot;V)\) are estimated on \(I_\eta\) by kernel ridge regression with a covariate kernel \(k_X\). The finite-feature ridge regressions used in the experiments are obtained as the corresponding linear-kernel case on the chosen feature basis.

For each arm \(a\), let
\[
I_{\eta,a}:=\{i\in I_\eta:\ A_i=a\},
\qquad
n_{\eta,a}:=|I_{\eta,a}|.
\]
Define
\[
K^X_{aa}
:=
\bigl[k_X(X_i,X_j)\bigr]_{i,j\in I_{\eta,a}},
\qquad
G^X_{I,a}
:=
\bigl[k_X(X_i,X_j)\bigr]_{i\in I,\ j\in I_{\eta,a}}.
\]
For a location set \(V\), define the arm-\(a\) target matrix
\[
U_a(V)
:=
K_Y(Y_{I_{\eta,a}},V)
\in\R^{n_{\eta,a}\times J}.
\]
With ridge parameter \(\lambda_a>0\),
\[
\hat M_{a,I}(V)
=
G^X_{I,a}
\bigl(K^X_{aa}+\lambda_a n_{\eta,a}I_{n_{\eta,a}}\bigr)^{-1}
U_a(V).
\]
Thus all \(J\) regression coordinates are fitted simultaneously. The matrices \(K^X_{aa}\), \(G^X_{I_{\tr},a}\), and \(G^X_{I_{\te},a}\) do not depend on \(V\); only \(U_a(V)\) changes with the candidate locations.

\subsection{Training and testing criteria}

On the training split,
\[
\hat Z_{\tr}(V)
=
\hat\Phi^1_{\tr}(V)-\hat\Phi^0_{\tr}(V),
\]
\[
\bar z_{\tr}^{\dr}(V)
=
\frac{1}{n_{\tr}}\hat Z_{\tr}(V)^\top 1_{n_{\tr}},
\qquad
S_{\tr}^{\dr}(V)
=
\frac{1}{n_{\tr}-1}
\hat Z_{\tr}(V)^\top H_{n_{\tr}}\hat Z_{\tr}(V),
\]
and
\[
\hat{\mathcal P}_{\tau,\tr}(V)
=
\bar z_{\tr}^{\dr}(V)^\top
\bigl(S_{\tr}^{\dr}(V)+\tau I_J\bigr)^{-1}
\bar z_{\tr}^{\dr}(V).
\]
After selecting \(\hat V\), the final statistic is computed by replacing \(I_{\tr}\) with \(I_{\te}\) and \(\tau\) with \(\gamma_{n_{\te}}\).

\subsection{Gradient-based Euclidean optimization}

When \(\Ycal\subset\R^d\) and \(k_Y(v,y)\) is differentiable in \(v\), the criterion \(V\mapsto\hat{\mathcal P}_{\tau,\tr}(V)\) is differentiable. Let
\[
R_a
:=
\bigl(K^X_{aa}+\lambda_a n_{\eta,a}I_{n_{\eta,a}}\bigr)^{-1}.
\]
Since \(\hat M_{a,I}(V)=G^X_{I,a}R_aU_a(V)\), for each location \(v_j\),
\[
\partial_{v_j}\hat M_{a,I}(V)
=
G^X_{I,a}R_a\,\partial_{v_j}U_a(V).
\]
The matrix \(\partial_{v_j}U_a(V)\) has only one nonzero column, with entries
\[
[\partial_{v_j}U_a(V)]_{rj}
=
\nabla_1 k_Y(v_j,Y_{i_r}),
\qquad i_r\in I_{\eta,a}.
\]
The same structure holds for \(\partial_{v_j}K_Y(Y_I,V)\).

Writing
\[
\mu_I(V):=\bar z_I^{\dr}(V),
\qquad
A_I(V):=S_I^{\dr}(V)+\tau I_J,
\]
the differential of the criterion is
\[
d\hat{\mathcal P}_{\tau,I}(V)
=
2\,(d\mu_I(V))^\top A_I(V)^{-1}\mu_I(V)
-
\mu_I(V)^\top A_I(V)^{-1}
(dA_I(V))
A_I(V)^{-1}\mu_I(V).
\]
These formulas give analytic gradients. Automatic differentiation can also be applied directly to the matrix expressions above.

\subsection{Finite-dictionary implementation}

For a finite dictionary
\[
\mathcal C=\{c_1,\dots,c_M\}\subset\Ycal,
\]
precompute
\[
K_Y(Y_{I_{\eta,a}},\mathcal C),
\qquad
K_Y(Y_{I_{\tr}},\mathcal C),
\qquad
K_Y(Y_{I_{\te}},\mathcal C).
\]
A candidate \(V=(c_{j_1},\dots,c_{j_J})\in\mathcal C^J\) is then evaluated by selecting columns \(j_1,\dots,j_J\) from these matrices. The formulas above apply without modification.

\subsection{Numerical precomputation}

The implementation precomputes
\[
R_a,
\qquad
G^X_{I_{\tr},a},
\qquad
G^X_{I_{\te},a},
\qquad
D_{a,\tr},
\qquad
D_{a,\te},
\qquad a\in\{0,1\}.
\]
For each new \(V\), the remaining operations are the construction or column selection of \(K_Y(Y_I,V)\) and \(U_a(V)\), followed by multiplications involving \(|I|\times n_{\eta,a}\) and \(n_{\eta,a}\times J\) matrices. In the finite-location regime, \(J\) is small, so the dominant cost is typically location search rather than the final Hotelling statistic.

\section{Additional experiments}
\label{app:additional_experiments}

\paragraph{Common synthetic setup.}
Unless stated otherwise, we generate \(X_i\sim N(0,I_5)\) and use the confounded treatment mechanism
\[
\pi_0(1\mid X_i)
=
\operatorname{clip}\{
\sigma(0.90X_{i1}-0.75X_{i2}+0.55X_{i3}-0.40X_{i4}),
0.06,0.94\},
\]
with prognostic component
\[
g(x)
=
0.90x_1+0.60\sin(x_2)+0.35(x_3^2-1)+0.25x_1x_4-0.20\cos(x_5).
\]
Each sample is split into \(I_\eta\), \(I_{\tr}\), and \(I_{\te}\), used respectively for nuisance fitting, location learning, and final testing. Propensities are estimated by ridge-regularized logistic regression, and the conditional kernel regressions \(m_a(x;c)=E[k_Y(c,Y)\mid A=a,X=x]\) are estimated by multi-output ridge regression on nonlinear features of \(X\). We use a Gaussian outcome kernel with bandwidth chosen on \(I_\eta\cup I_{\tr}\), a dictionary of \(M=80\) candidate outcome locations, \(J=2\) selected locations, and nominal level \(\alpha=0.05\). All rejection rates are based on \(200\) Monte Carlo replications; the Monte Carlo standard error at level \(0.05\) is approximately \(0.015\).

\subsection{Computation infrastructure}

All experiments were run locally on a professional laptop; no remote servers, cluster
resources, or external GPUs were used. The reported runtimes and empirical
results therefore reflect a single-machine CPU-based setup.

\begin{itemize}[label=\textbullet]
\item \textbf{Machine:} Apple MacBook Pro
\item \textbf{Model Identifier:} Mac16,1
\item \textbf{Chip:} Apple M4
\item \textbf{CPU:} 10 cores (4 performance cores, 6 efficiency cores)
\item \textbf{Memory:} 24\,GB
\item \textbf{Operating System:} macOS 15.0
\item \textbf{Kernel / Architecture:} Darwin 24.0.0, \texttt{arm64}
\item \textbf{Compute:} CPU-only execution; no external GPU or remote compute was used
\end{itemize}

\subsection{Sharp-null calibration under observational confounding.}
We first isolate type-I error under a sharp interventional null:
\[
Y_i(0)=Y_i(1)=g(X_i)+\varepsilon_i,
\qquad
\varepsilon_i\sim N(0,1).
\]
Thus \(P_{Y(0)}=P_{Y(1)}\), but the observed treated and control outcome distributions differ because treatment depends on prognostic covariates. Figure~\ref{fig:app_exp1_calibration} and Table~\ref{tab:app_exp1_calibration} show that split-sample DR-ME remains close to nominal level across sample sizes, as does the random-location DR variant. By contrast, the naive observed ME test strongly over-rejects because it ignores confounding, the no-split DR-ME ablation over-rejects because locations are learned and tested on the same data, and the plug-in DM/IPW variants show the instability expected from non-orthogonal or non-augmented procedures. This experiment confirms that calibration comes from the orthogonal score together with an independent final test split.

\begin{figure}[t]
\centering
\includegraphics[width=0.6\linewidth]{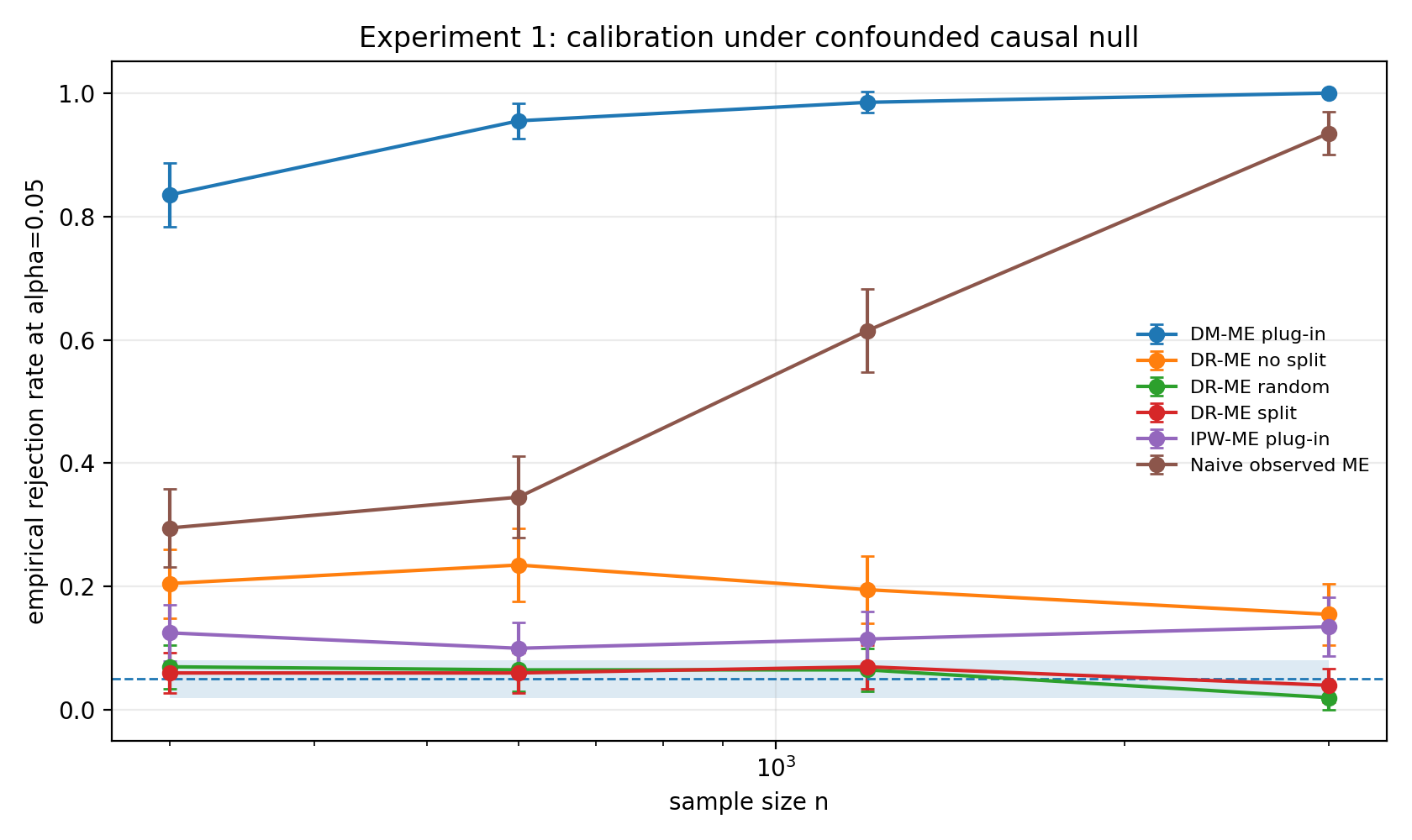}
\caption{
Sharp-null calibration under observational confounding. The dashed line is the nominal level \(\alpha=0.05\). Split DR-ME and random-location DR-ME remain calibrated; naive observed testing, no-split selection, and non-orthogonal plug-in baselines over-reject.
}
\label{fig:app_exp1_calibration}
\end{figure}

\begin{table}[t]
\centering
\caption{
Empirical rejection rates under the sharp confounded null \(Y(0)=Y(1)\). The interventional laws are equal, but observed treated and control outcome distributions differ because of confounding.
}
\label{tab:app_exp1_calibration}
\begin{tabular}{lcccc}
\toprule
Method & \(n=300\) & \(n=600\) & \(n=1200\) & \(n=3000\) \\
\midrule
DR-ME split & 0.060 & 0.060 & 0.070 & 0.040 \\
DR-ME random & 0.070 & 0.065 & 0.065 & 0.020 \\
IPW-ME plug-in & 0.125 & 0.100 & 0.115 & 0.135 \\
DM-ME plug-in & 0.835 & 0.955 & 0.985 & 1.000 \\
Naive observed ME & 0.295 & 0.345 & 0.615 & 0.935 \\
DR-ME no split & 0.205 & 0.235 & 0.195 & 0.155 \\
\bottomrule
\end{tabular}
\end{table}

\subsection{Power under confounded alternatives.}
We next use the same confounded observational design, nuisance estimators, and three-way split, but replace the sharp null by three alternatives. With \(Y_i(0)=g(X_i)+\varepsilon_{i0}\), we consider a broad mean shift,
\[
Y_i(1)=g(X_i)+\varepsilon_{i1}+0.35,
\]
a variance shift,
\[
Y_i(1)=g(X_i)+1.45\,\varepsilon_{i1},
\]
and a mean-zero mixture perturbation,
\[
Y_i(1)
=
g(X_i)+\varepsilon_{i1}
+
\frac{B_i-q}{\sqrt{q(1-q)}},
\qquad
B_i\sim\mathrm{Bernoulli}(q),\quad q=0.12 .
\]
The last alternative changes the shape and tail behavior of the interventional outcome law without changing the mean of the added component. We compare DR-ME with DR-xKTE, a global doubly robust kernel test, using the same held-out test split. We also retain finite-location diagnostics from the calibration experiment: random-location DR-ME, IPW-ME, DM-ME, naive observed ME, and the no-split DR-ME ablation.

\begin{table}[t]
\centering
\caption{Empirical rejection rates for DR-ME and DR-xKTE. Under the null, the entries estimate type-I error. Under the alternatives, they estimate power.}
\label{tab:power_odrme_drxkte}
\begin{tabular}{llcccc}
\toprule
Setting & Method & $n=300$ & $n=600$ & $n=1200$ & $n=3000$ \\
\midrule
Null & DR-ME & 0.055 & 0.075 & 0.060 & 0.080 \\
 & DR-xKTE & 0.075 & 0.045 & 0.030 & 0.070 \\
\midrule
Mean shift & DR-ME & 0.115 & 0.225 & 0.505 & 0.885 \\
 & DR-xKTE & 0.155 & 0.240 & 0.530 & 0.840 \\
\midrule
Variance shift & DR-ME & 0.115 & 0.155 & 0.335 & 0.790 \\
 & DR-xKTE & 0.065 & 0.080 & 0.190 & 0.560 \\
\midrule
Localized bump & DR-ME & 0.135 & 0.165 & 0.225 & 0.605 \\
 & DR-xKTE & 0.035 & 0.060 & 0.170 & 0.295 \\
\bottomrule
\end{tabular}
\end{table}

\begin{figure}[t]
\centering
\includegraphics[width=\linewidth]{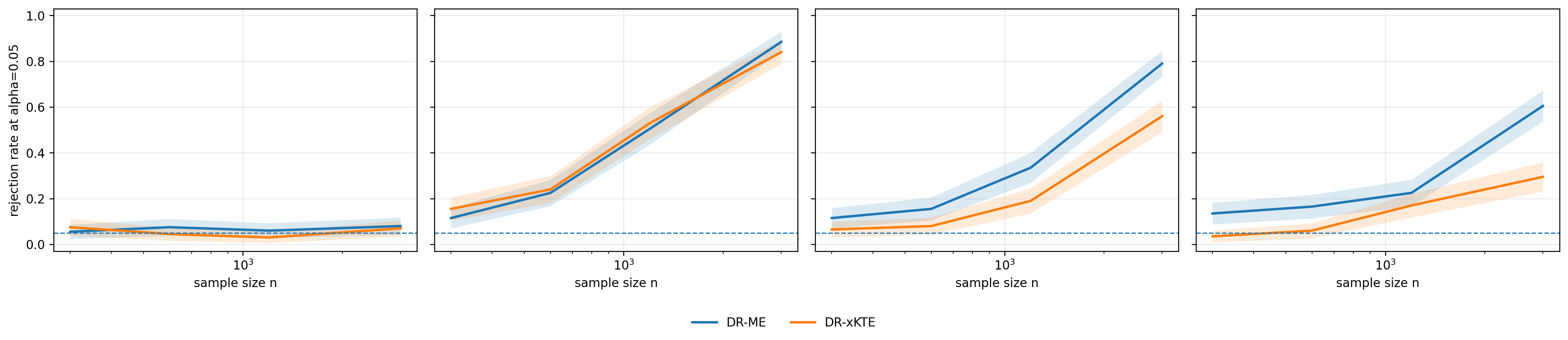}
\caption{
DR-ME versus the global DR-xKTE baseline. Both methods are close to nominal level under the confounded null. DR-ME is comparable on the mean shift and stronger on the variance-shift and localized-bump alternatives in this design.
}
\label{fig:power_odrme_drxkte}
\end{figure}

Table~\ref{tab:power_odrme_drxkte} and Figure~\ref{fig:power_odrme_drxkte} show that DR-ME and DR-xKTE have comparable calibration under the confounded null. On the broad mean-shift alternative, the two methods are similar: DR-xKTE is slightly stronger at smaller sample sizes, while DR-ME is slightly stronger at \(n=3000\). On the variance-shift and localized-bump alternatives, DR-ME is stronger at larger sample sizes. At \(n=3000\), DR-ME rejects with probabilities \(0.790\) and \(0.605\), compared with \(0.560\) and \(0.295\) for DR-xKTE. This is consistent with the finite-location objective: when a small number of whitened witness evaluations captures the discrepancy, location learning can concentrate power on informative parts of the outcome distribution rather than averaging globally.

\begin{table}[t]
\centering
\caption{Diagnostic baselines at $n=3000$. These baselines are included to interpret the power curves, not all as valid calibrated competitors.}
\label{tab:power_diagnostics_n3000}
\begin{tabular}{lcccc}
\toprule
Method & Null & Mean shift & Variance shift & Localized bump \\
\midrule
DR-ME & 0.080 & 0.885 & 0.790 & 0.605 \\
DR-xKTE & 0.070 & 0.840 & 0.560 & 0.295 \\
DR-ME-Rand & 0.045 & 0.830 & 0.690 & 0.475 \\
IPW-ME plug-in & 0.110 & 0.685 & 0.590 & 0.590 \\
DM-ME plug-in & 0.990 & 1.000 & 1.000 & 1.000 \\
Naive observed ME & 0.960 & 1.000 & 0.990 & 0.980 \\
DR-ME-NoSplit & 0.190 & 0.965 & 0.880 & 0.845 \\
\bottomrule
\end{tabular}
\end{table}

Table~\ref{tab:power_diagnostics_n3000} clarifies which rejection rates can be interpreted as calibrated power. The naive observed ME test rejects the null with probability \(0.960\), because it tests observed treated-versus-control outcome distributions rather than interventional laws. The no-split ablation also over-rejects, showing the selection bias induced by learning and testing locations on the same data. The DM plug-in baseline is not a valid calibrated comparator in this experiment, since it has near-unit null rejection. The calibrated comparison is therefore mainly among DR-ME, DR-ME-Rand, and DR-xKTE. Within that group, learned locations improve over random locations, and DR-ME gives the largest rejection rates on the variance and localized-bump alternatives.

\subsection{High-dimensional location-learning ablation}
\label{sec:exp_hd_whitening}

This experiment isolates the role of the covariance-whitened learning criterion. The final test is held fixed across methods: all variants compute the same split-sample DR-ME Hotelling statistic on \(I_{\te}\). Only the location-learning rule on \(I_{\tr}\) changes. We compare full whitening,
\[
\hat\mu_V^\top(\hat\Sigma_V+\tau I)^{-1}\hat\mu_V,
\]
raw witness maximization, \(\|\hat\mu_V\|^2\), and random dictionary locations.

Outcomes are vector-valued, \(Y\in\R^{d_Y}\), with \(d_Y\in\{5,10,25,50\}\). Under the null,
\[
Y(0)=Y(1)=\mu(X)+\varepsilon .
\]
Under the alternative, rare mass is shifted toward two sparse regions:
\[
Y(0)=\mu(X)+\varepsilon_0-\Delta v_1\mathbf 1\{B_0=1\}-\Delta v_2\mathbf 1\{B_0=2\},
\]
\[
Y(1)=\mu(X)+\varepsilon_1+\Delta v_1\mathbf 1\{B_1=1\}+\Delta v_2\mathbf 1\{B_1=2\},
\]
where \(\Pr(B_a=1)=\Pr(B_a=2)=0.04\) and \(\Delta=4\). Candidate locations are observed outcomes from the training split, so most dictionary points are background points rather than bump locations. We use \(n=3000\), \(M=300\) dictionary candidates, \(J=3\) selected locations, and \(200\) Monte Carlo replications.

The main text reports null rejection rates and power. Here we report selection diagnostics explaining the power gap.

\begin{table}[t]
\centering
\caption{Selection diagnostics under the high-dimensional two-bump alternative. "Any bump" is the probability that at least one selected location lies in a true bump region. "Distinct regions" is the average number of distinct bump regions selected. "Avg. corr." is the average absolute correlation among the selected DR features on the training split.}
\label{tab:exp_hd_selection}
\begin{tabular}{llccc}
\toprule
$d_Y$ & Method & Any bump & Distinct regions & Avg. corr. \\
\midrule
5 & DR-ME & 1.000 & 2.495 & 0.173 \\
 & Raw witness & 0.180 & 0.205 & 0.765 \\
 & Random locations & 0.200 & 0.230 & 0.384 \\
\midrule
10 & DR-ME & 1.000 & 2.240 & 0.170 \\
 & Raw witness & 0.295 & 0.315 & 0.646 \\
 & Random locations & 0.245 & 0.270 & 0.383 \\
\midrule
25 & DR-ME & 0.870 & 1.470 & 0.248 \\
 & Raw witness & 0.420 & 0.500 & 0.529 \\
 & Random locations & 0.205 & 0.210 & 0.378 \\
\midrule
50 & DR-ME & 0.890 & 1.300 & 0.274 \\
 & Raw witness & 0.750 & 1.010 & 0.434 \\
 & Random locations & 0.210 & 0.220 & 0.360 \\
\bottomrule
\end{tabular}
\end{table}

Table~\ref{tab:exp_hd_selection} shows that full whitening selects informative and less redundant coordinates. It hits at least one true bump region much more often than random locations, and it selects more distinct bump regions than the raw witness objective. It also has lower average feature correlation than raw witness maximization. Thus the gain in the main power table is not merely from finding large empirical witness values: the covariance term changes the selected finite-dimensional projection by penalizing redundant or high-noise coordinates, as predicted by the criterion \(\mu_V^\top\Sigma_V^{-1}\mu_V\).

\subsection{Local noncentral chi-square validation}

This experiment validates the local asymptotic law of the fixed-location
orthogonalized ME statistic. Unlike the previous power experiments, the goal is
not to compare against external baselines. The goal is to check whether, once
the locations are fixed independently of the test sample, the empirical
Hotelling statistic follows the predicted central and noncentral chi-square
laws.

The theoretical prediction is the following. For fixed locations $V$ and local
alternatives of size $n^{-1/2}$,
\begin{equation}
\sqrt n \, \bar z^{\mathrm{dr}}_{n,V}
\Rightarrow
N\left(h\eta_V(g),\Sigma_V\right),
\end{equation}
and therefore
\begin{equation}
\widehat\lambda^{\mathrm{dr}}_{n,V}
=
n
\bar z^{\mathrm{dr}\top}_{n,V}
\left(S^{\mathrm{dr}}_{n,V}+\gamma_n I_J\right)^{-1}
\bar z^{\mathrm{dr}}_{n,V}
\Rightarrow
\chi^2_J\left(h^2\lambda_V(g)\right),
\end{equation}
where
\begin{equation}
\lambda_V(g)=\eta_V(g)^\top\Sigma_V^{-1}\eta_V(g).
\end{equation}
Thus, under the null $h=0$, the statistic should be approximately central
$\chi^2_J$. Under local alternatives, the rejection probability should follow
the power curve of a noncentral $\chi^2_J$ distribution with noncentrality
parameter $h^2\lambda_V(g)$.

We use the same observational structure as in the previous experiments:
\begin{equation}
X_i \sim N(0,I_5),
\qquad
A_i\mid X_i \sim \operatorname{Bernoulli}(\pi_0(X_i)),
\end{equation}
with
\begin{equation}
\pi_0(1\mid X_i)
=
\operatorname{clip}
\left\{
\sigma(0.90X_{i1}-0.75X_{i2}+0.55X_{i3}-0.40X_{i4}),
0.06,
0.94
\right\}.
\end{equation}
The prognostic component is
\begin{equation}
g(x)
=
0.90x_1
+
0.60\sin(x_2)
+
0.35(x_3^2-1)
+
0.25x_1x_4
-
0.20\cos(x_5).
\end{equation}
The local path is a treatment-arm mean shift of order $n^{-1/2}$:
\begin{equation}
Y_i(0)=g(X_i)+\varepsilon_{i0},
\qquad
Y_i(1)=g(X_i)+\varepsilon_{i1}+\delta_n,
\qquad
\delta_n=\frac{h}{\sqrt n},
\end{equation}
with $\varepsilon_{i0},\varepsilon_{i1}\sim N(0,1)$. The treatment mechanism
and the covariate distribution are unchanged along the path. Hence the only
local perturbation is in the interventional outcome law under treatment.

We use a Gaussian kernel on the outcome. Since the conditional residual law is
Gaussian, the nuisance regressions
\[
m_a(x;v)=E[k_Y(v,Y)\mid A=a,X=x]
\]
are available in closed form. For $Y=g(x)+\delta+\varepsilon$,
$\varepsilon\sim N(0,\sigma_Y^2)$, and
$k_Y(v,y)=\exp\{-(v-y)^2/(2\ell_Y^2)\}$,
\begin{equation}
m_\delta(x;v)
=
\left(\frac{\ell_Y^2}{\ell_Y^2+\sigma_Y^2}\right)^{1/2}
\exp\left\{
-\frac{(v-g(x)-\delta)^2}{2(\ell_Y^2+\sigma_Y^2)}
\right\}.
\end{equation}
We use these true nuisance functions in this experiment to isolate the local
chi-square approximation from nuisance-estimation error. This is a deliberate
theory-validation experiment: first-stage estimation is not the object being
tested here.

The locations are also fixed independently of the Monte Carlo testing samples.
A large independent pilot sample under the null is used to select $J=2$
locations by maximizing the local noncentrality proxy
\begin{equation}
\eta_V^\top(\Sigma_V+\tau I_J)^{-1}\eta_V.
\end{equation}
The selected locations are then frozen for all repetitions, all sample sizes,
and all local-alternative strengths. This avoids any leakage from the test
sample into the location choice.

The pilot configuration is:
\begin{align}
n_{\mathrm{pilot}} &= 100000, \\
M &= 100, \\
J &= 2, \\
\widehat V &= (-2.6752,\; 3.9031), \\
\ell_Y &= 1.4289, \\
\widehat\lambda_V &= 0.1383.
\end{align}
We then run $2000$ Monte Carlo repetitions for each
$n\in\{500,1000,3000\}$ and
$h\in\{0,1,2,3,4,6,8\}$. The rejection threshold is the
$0.95$ quantile of $\chi^2_2$.

\begin{table}[t]
\centering
\caption{Empirical rejection rates under local alternatives,
compared with the noncentral chi-square prediction. The theoretical power is
computed from $\chi^2_2(h^2\widehat\lambda_V)$ with
$\widehat\lambda_V=0.1383$.}
\label{tab:exp4_local_power}
\begin{tabular}{cccccc}
\toprule
$h$ & $h^2\widehat\lambda_V$ & Theory & $n=500$ & $n=1000$ & $n=3000$ \\
\midrule
0 & 0.000 & 0.050 & 0.053 & 0.044 & 0.054 \\
1 & 0.138 & 0.061 & 0.063 & 0.059 & 0.049 \\
2 & 0.553 & 0.094 & 0.107 & 0.101 & 0.095 \\
3 & 1.245 & 0.155 & 0.164 & 0.166 & 0.167 \\
4 & 2.213 & 0.246 & 0.260 & 0.269 & 0.275 \\
6 & 4.979 & 0.502 & 0.555 & 0.519 & 0.524 \\
8 & 8.851 & 0.763 & 0.810 & 0.786 & 0.770 \\
\bottomrule
\end{tabular}
\end{table}

Table~\ref{tab:exp4_local_power} shows that the empirical rejection
probabilities closely follow the noncentral chi-square prediction. Under the
null, corresponding to $h=0$, the rejection rate is approximately $5\%$ at all
sample sizes. This verifies the central chi-square calibration of the
fixed-location statistic. Under local alternatives, the rejection probability
increases with $h$ and tracks the theoretical curve well. The agreement is
especially good at $n=1000$ and $n=3000$, as expected from an asymptotic local
theory.

At the smaller sample size $n=500$, the empirical power is slightly above the
asymptotic prediction for larger values of $h$, especially at $h=6$ and $h=8$.
This is not a failure of the theory. For large $h$, the finite-sample shift
$\delta_n=h/\sqrt n$ is no longer extremely small, so the experiment is moving
away from the strictly local regime. The asymptotic prediction should be most
accurate when $n$ is larger and the same $h$ corresponds to a smaller actual
shift.

\begin{figure}[t]
\centering
\includegraphics[width=0.85\linewidth]{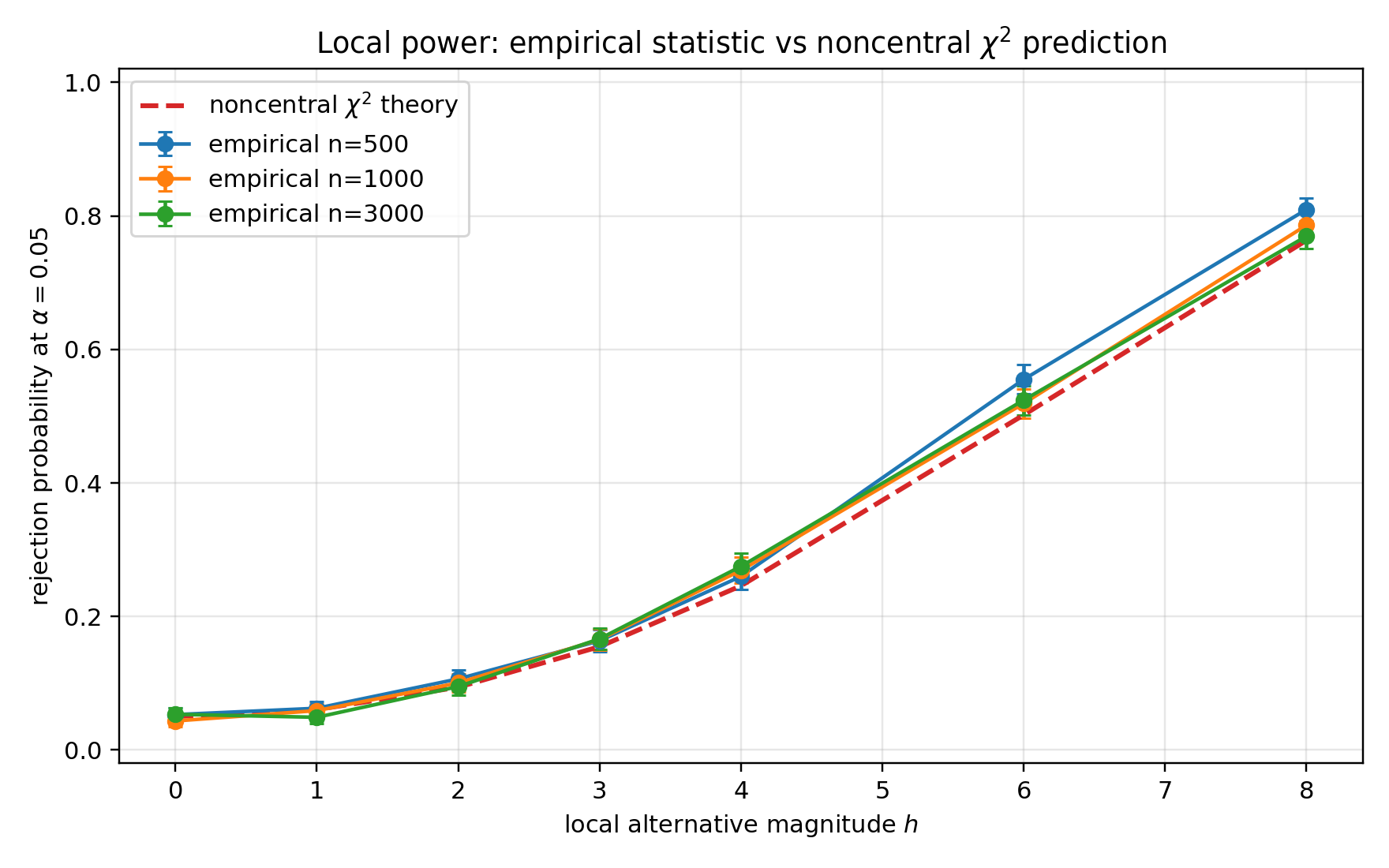}
\caption{Empirical rejection rates against the noncentral
chi-square prediction. The black theoretical curve is computed from
$\chi^2_2(h^2\widehat\lambda_V)$. The empirical curves for
$n=500,1000,3000$ closely follow the predicted local-power curve, with the
largest-sample curve giving the closest agreement.}
\label{fig:exp4_power_vs_theory}
\end{figure}

Figure~\ref{fig:exp4_power_vs_theory} is the main figure for the experiment.
It directly validates the local-power formula. The horizontal axis is the local
strength $h$, not the raw shift $\delta_n$. The raw shift changes with $n$ as
$\delta_n=h/\sqrt n$, while the theory predicts that the asymptotic power
should depend on $h$ through $h^2\lambda_V$. The empirical curves are close to
the theoretical curve, which supports the claim that $\lambda_V$ is the relevant
whitened local signal-to-noise ratio.

\begin{figure}[t]
\centering
\includegraphics[width=0.75\linewidth]{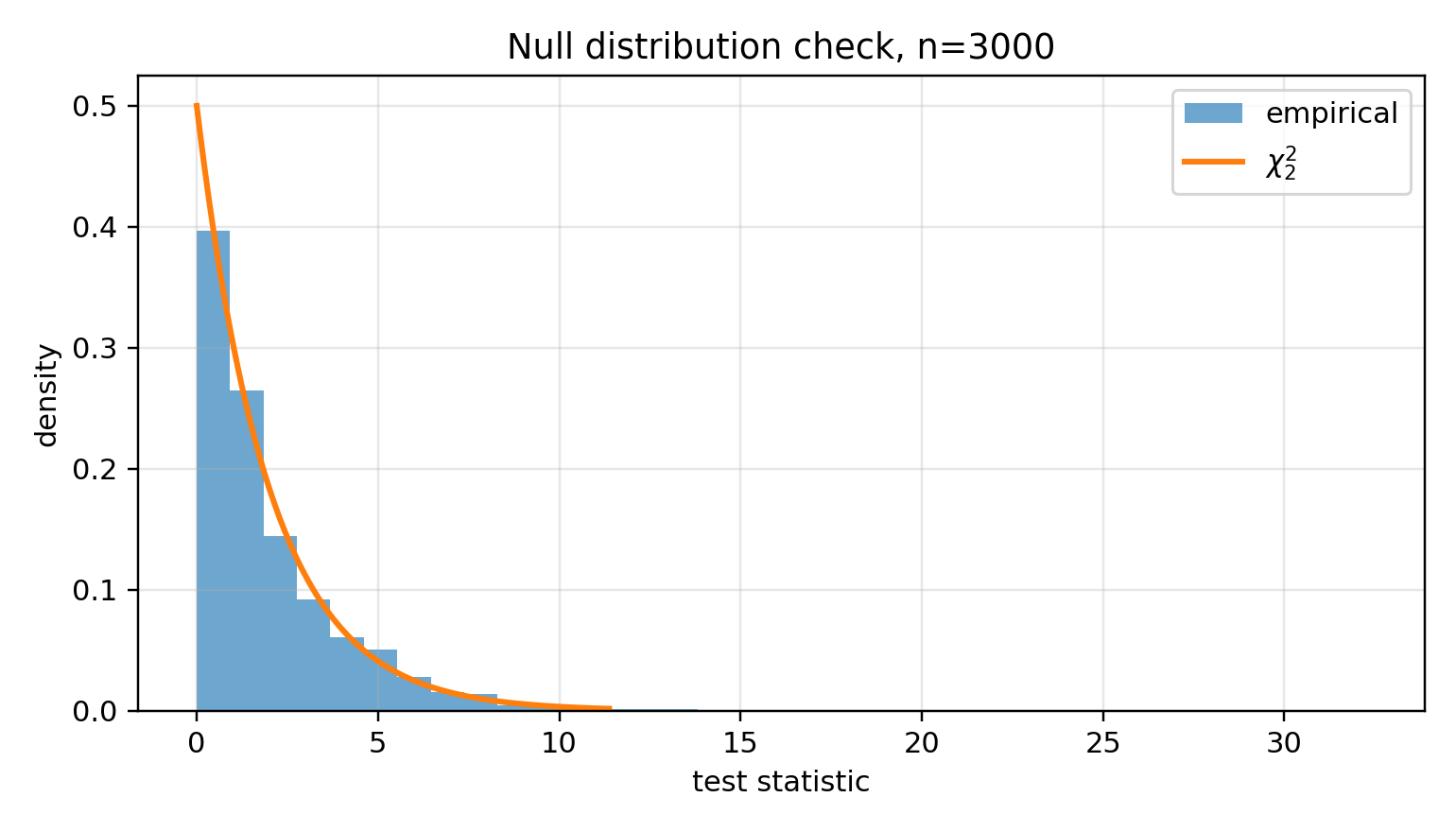}
\caption{Null distribution of the statistic at the largest sample
size. The histogram shows the Monte Carlo distribution of
$\widehat\lambda^{\mathrm{dr}}_{n,V}$ under $h=0$, while the curve shows the
central $\chi^2_2$ density.}
\label{fig:exp4_null_hist}
\end{figure}

Figure~\ref{fig:exp4_null_hist} checks the full null distribution, not only the
rejection probability at level $0.05$. The empirical histogram aligns well with
the central $\chi^2_2$ reference law. This supports the fixed-location null
calibration of the statistic and confirms that the covariance normalization is
working as intended.

\begin{figure}[t]
\centering
\includegraphics[width=0.75\linewidth]{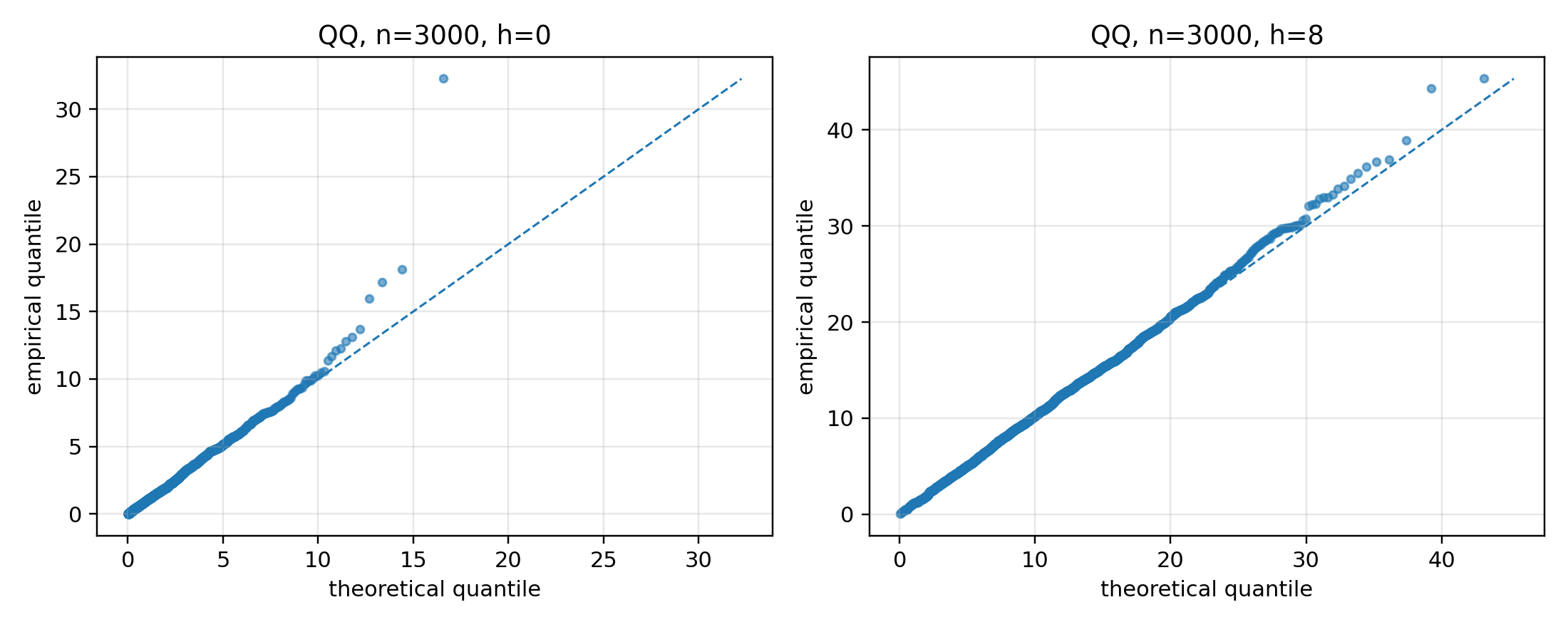}
\caption{QQ plot at the largest sample size. Empirical quantiles
of the statistic are compared with the corresponding noncentral chi-square
quantiles.}
\label{fig:exp4_qq}
\end{figure}

Figure~\ref{fig:exp4_qq} gives a stronger distributional diagnostic. A power
curve only checks one tail probability for each $h$, whereas the QQ plot checks
the entire distribution. The approximate alignment with the diagonal indicates
that the statistic is not merely producing the correct rejection probability,
but is close to the predicted noncentral chi-square distribution over a broad
range of quantiles.

\begin{figure}[t]
\centering
\includegraphics[width=0.75\linewidth]{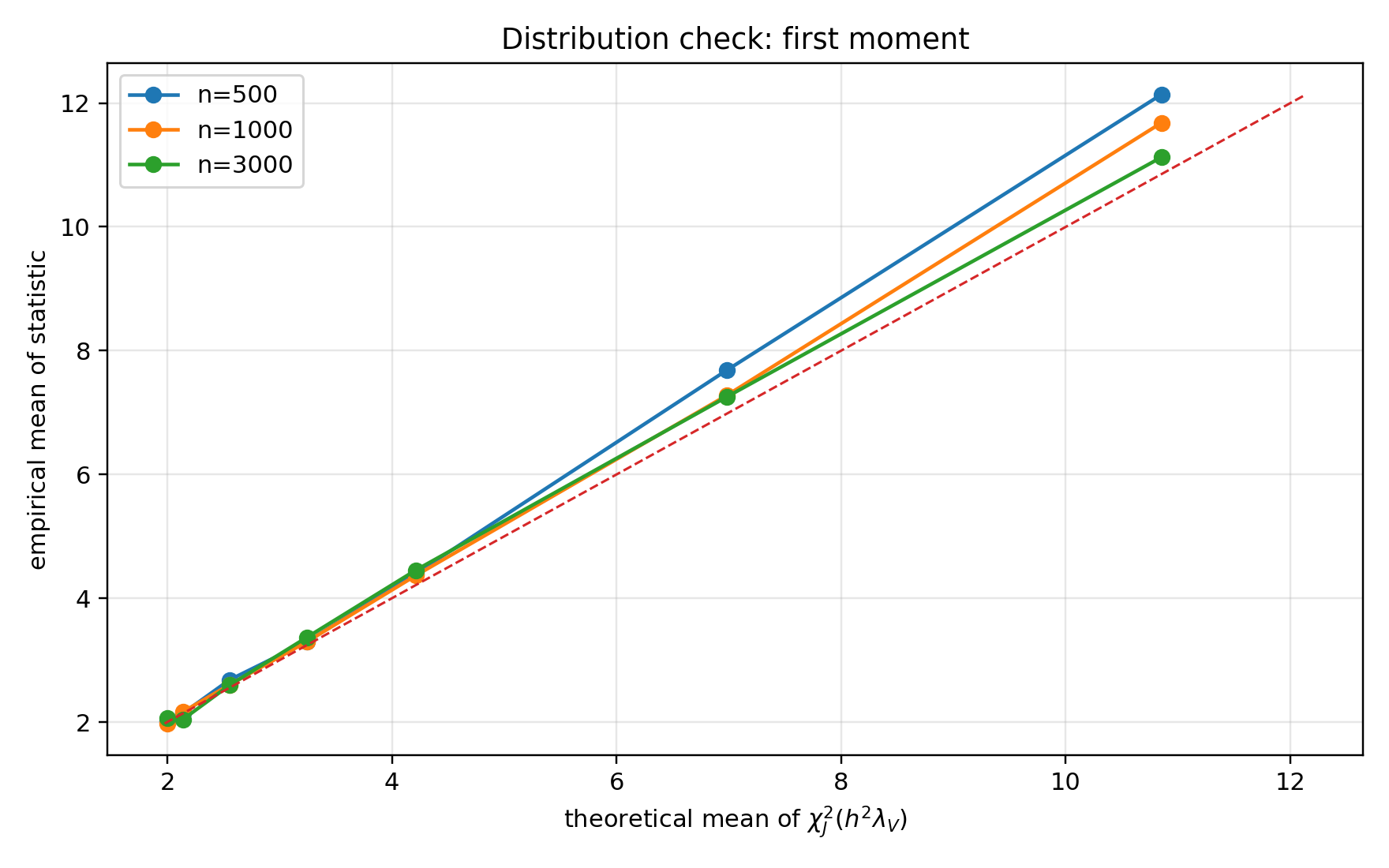}
\caption{Moment diagnostic. The empirical mean of
$\widehat\lambda^{\mathrm{dr}}_{n,V}$ is compared with the theoretical mean
$J+h^2\widehat\lambda_V$ of the noncentral $\chi^2_J$ distribution.}
\label{fig:exp4_mean_diagnostic}
\end{figure}

Figure~\ref{fig:exp4_mean_diagnostic} compares the empirical first moment of
the statistic with the theoretical mean of the limiting noncentral chi-square
law:
\begin{equation}
E\left[\chi^2_J(h^2\lambda_V)\right]
=
J+h^2\lambda_V.
\end{equation}
The empirical means follow this prediction closely. The mild upward deviation
for larger $h$ at smaller $n$ is consistent with the finite-sample behavior
already visible in the power curves: for large local strengths and moderate
$n$, the alternatives are no longer perfectly infinitesimal.

Overall, this experiment supports the local asymptotic theory in two ways.
First, under $h=0$, the statistic has the predicted central chi-square behavior.
Second, under $n^{-1/2}$ local alternatives, its rejection probability and
distributional shape are well described by the noncentral chi-square law with
noncentrality parameter $h^2\lambda_V$. This validates the main theoretical
interpretation of the statistic and justifies using the whitened quantity
$\lambda_V$ as the target criterion for location learning.

Once $V$ is fixed independently of the
test sample, the statistic behaves exactly as the local Gaussian and
noncentral chi-square theory predicts.

\subsection{Runtime comparison with a global DR kernel test}

We finally compare the computational cost of the proposed finite-location test
against a global doubly robust kernel treatment-effect test. The purpose of this
experiment is not to study power, but to quantify the practical cost of the two
ways of implementing DR-ME and to compare them with DR-xKTE. The comparison is
important because DR-ME ultimately tests a low-dimensional statistic indexed by
$J$ learned locations, whereas DR-xKTE remains a global kernel statistic over the
evaluation sample.

We use the same observational data-generating mechanism as in the previous
experiments and the same three-way split
\[
\{1,\ldots,n\}=I_\eta \cup I_{\mathrm{tr}} \cup I_{\mathrm{te}}.
\]
The nuisance split $I_\eta$ is used to fit the propensity and, for DR-ME, the
conditional kernel regressions. The training split $I_{\mathrm{tr}}$ is used to
learn locations for DR-ME, and the final test is evaluated on
$I_{\mathrm{te}}$. This split is kept fixed across implementations so that the
measured difference is due to the testing procedure rather than to different data
usage.

We compare three methods:
\begin{itemize}
    \item \textbf{DR-ME-Dict}: the finite-dictionary implementation of DR-ME.
    A dictionary $C=\{c_1,\ldots,c_M\}$ of candidate outcome locations is formed,
    all outcome-kernel evaluations $K_Y(Y,C)$ are precomputed, and locations are
    selected by greedy maximization of the ridge-stabilized local-power criterion.
    We use $J=2$ selected locations and $M=80$ dictionary candidates.

    \item \textbf{DR-ME-Grad}: the continuous-location implementation of DR-ME.
    Locations are optimized by gradient ascent in the Euclidean outcome space,
    using the differentiable version of the same ridge-stabilized criterion. We
    use three gradient steps in this runtime experiment.

    \item \textbf{DR-xKTE}: a global doubly robust kernel test for the kernel
    treatment effect. The propensity is fit on $I_\eta$ and evaluated on
    $I_{\mathrm{te}}$, as in DR-ME. The statistic is then computed on the final
    test split. We use a Cholesky-based linear solver for the internal kernel
    ridge systems, so this is not an intentionally slow implementation.
\end{itemize}

For each method, we decompose runtime into nuisance fitting, location learning
when applicable, and final testing. For DR-ME-Dict, the learning time includes
precomputing the dictionary kernel matrices and greedy location search. For
DR-ME-Grad, the learning time is the continuous gradient-ascent optimization.
For DR-xKTE, there is no location-learning phase; the method-specific cost is the
global kernel statistic computed on $I_{\mathrm{te}}$. We report averages over
five random seeds for each sample size
\[
n\in\{300,600,1200,3000,5000,10000\}.
\]

\begin{table}[t]
\centering
\caption{Total runtime in seconds. Total runtime includes nuisance fitting, location learning when applicable, and final testing. Results are averaged over five random seeds with $J=2$, $M=80$, and three gradient steps for DR-ME-Grad.}
\label{tab:runtime_total_drkte}
\begin{tabular}{rrrr}
\toprule
$n$ & DR-ME-Dict & DR-ME-Grad & DR-xKTE  \\
\midrule
300 & 0.0048 & 0.0019 & 0.0006 \\
600 & 0.0059 & 0.0032 & 0.0012 \\
1200 & 0.0132 & 0.0107 & 0.0046 \\
3000 & 0.0223 & 0.0282 & 0.0366 \\
5000 & 0.0395 & 0.0508 & 0.1252 \\
10000 & 0.2080 & 0.3631 & 0.9982 \\
\bottomrule
\end{tabular}
\end{table}

\begin{figure}[t]
\centering
\includegraphics[width=0.78\linewidth]{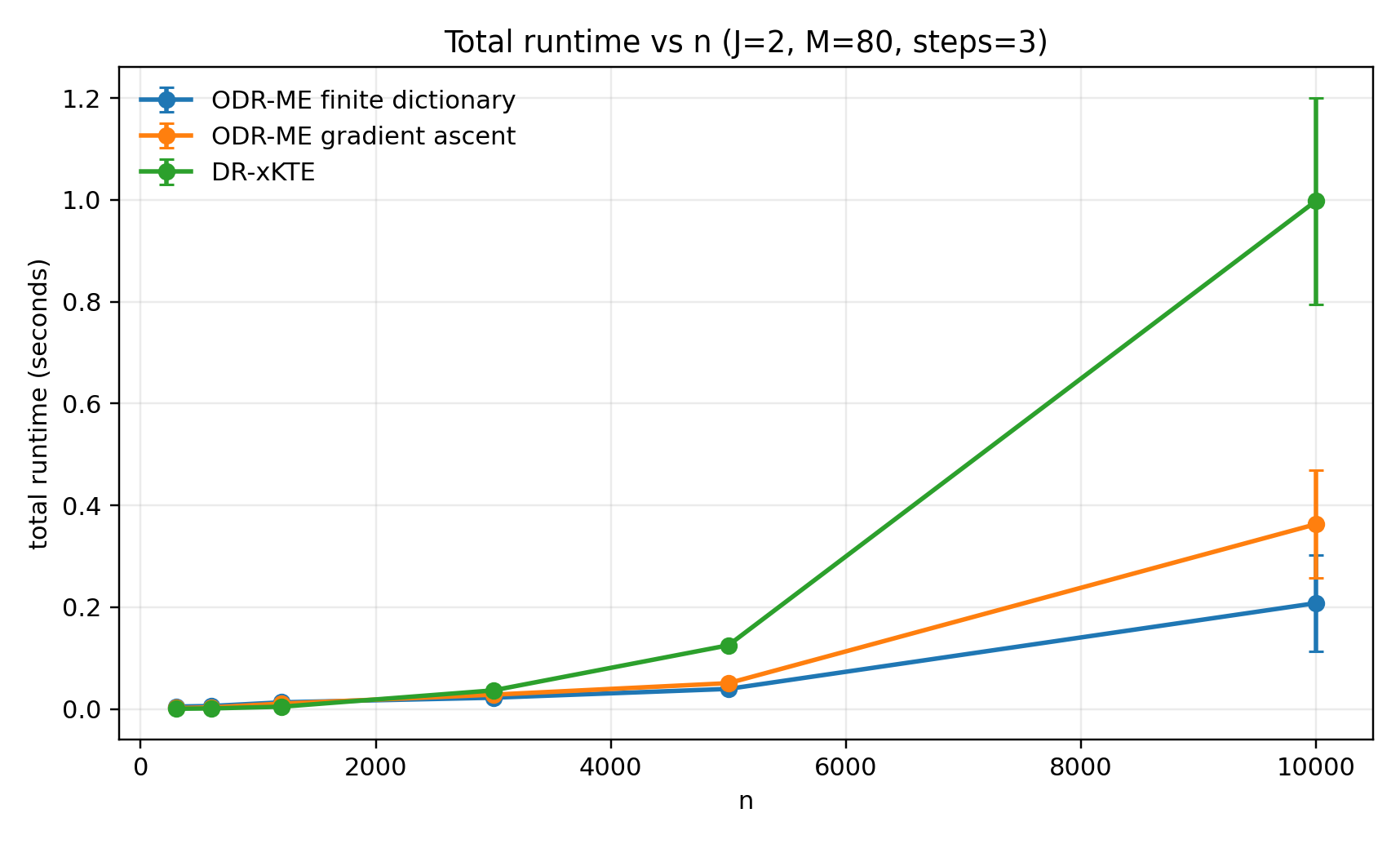}
\caption{Total runtime versus sample size. DR-xKTE is competitive at small sample
sizes, where the global kernel computation is still cheap. At larger sample
sizes, the finite-dictionary DR-ME implementation becomes faster. Continuous
gradient ascent becomes the most expensive implementation at large $n$, even
with only three gradient steps.}
\label{fig:runtime_total_drkte}
\end{figure}

Table~\ref{tab:runtime_total_drkte} and
Figure~\ref{fig:runtime_total_drkte} show that the relative cost changes with
sample size. For $n\leq 1200$, all methods are cheap, and DR-xKTE is often the
fastest method. This is expected: at small sample sizes, the overhead of
precomputing dictionary features and running location search is not yet offset by
DR-ME's low-dimensional final statistic. At $n=3000$, DR-xKTE remains slightly
faster than DR-ME-Dict in this optimized implementation. However, for
$n=5000$ and $n=10000$, DR-ME-Dict becomes faster than DR-xKTE. At $n=10000$,
DR-ME-Dict is faster than DR-xKTE.

This scaling is the relevant practical point. DR-xKTE is a global kernel test and
therefore continues to manipulate kernel matrices on the evaluation sample. In
contrast, after location learning, DR-ME computes a Hotelling statistic in only
$J=2$ dimensions. The finite-dictionary implementation also exploits the fact
that all candidate outcome-kernel evaluations can be precomputed once and reused
during greedy search.

\begin{table}[t]
\centering
\caption{Core method runtime in seconds. Core runtime removes nuisance fitting and keeps only location learning plus the final test for DR-ME, or the global kernel statistic for DR-xKTE. Results are averaged over five random seeds with $J=2$, $M=80$, and three gradient steps for DR-ME-Grad.}
\label{tab:runtime_core_drkte}
\begin{tabular}{rrrr}
\toprule
$n$ & DR-ME-Dict & DR-ME-Grad & DR-xKTE  \\
\midrule
300 & 0.0037 & 0.0007 & 0.0004 \\
600 & 0.0038 & 0.0011 & 0.0010 \\
1200 & 0.0054 & 0.0030 & 0.0043 \\
3000 & 0.0140 & 0.0200 & 0.0359 \\
5000 & 0.0296 & 0.0409 & 0.1236 \\
10000 & 0.1924 & 0.3475 & 0.9964 \\
\bottomrule
\end{tabular}
\end{table}

\begin{figure}[t]
\centering
\includegraphics[width=0.78\linewidth]{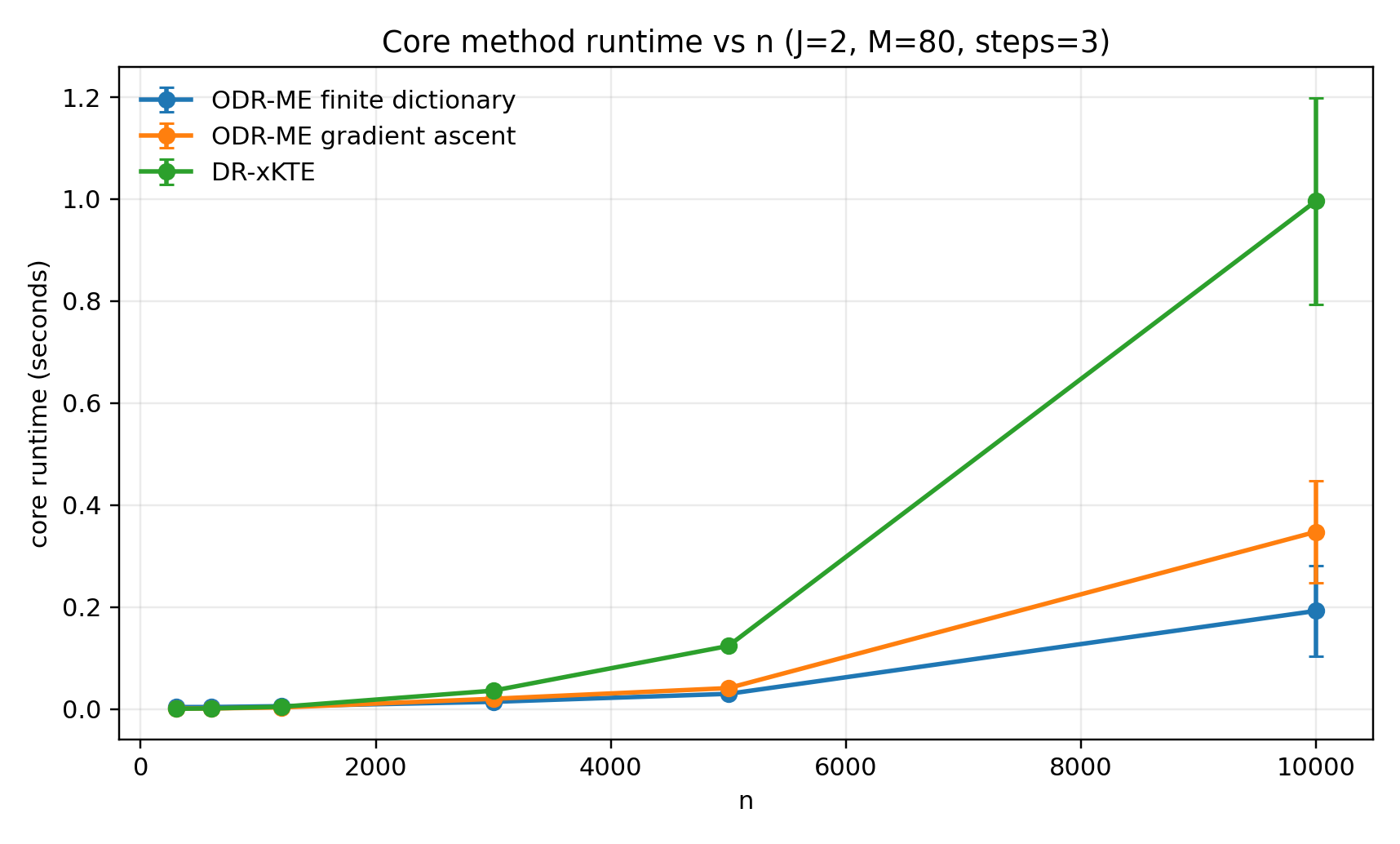}
\caption{Core method runtime versus sample size. This plot removes nuisance
fitting and focuses on the computational cost specific to the test. The finite
-dictionary implementation of DR-ME overtakes DR-xKTE at large $n$, while the
gradient implementation becomes substantially slower.}
\label{fig:runtime_core_drkte}
\end{figure}

The core-runtime comparison in Table~\ref{tab:runtime_core_drkte} and
Figure~\ref{fig:runtime_core_drkte} gives the cleanest view of method-specific
scaling. Nuisance fitting is shared in spirit across the methods and is not the
main object of the comparison. The core runtime shows that DR-ME-Dict is not
uniformly faster than DR-xKTE at every sample size, but it becomes faster in the
large-$n$ regime. This is the regime where using a small number of learned
locations matters computationally.

The continuous gradient implementation has a different profile. It avoids a
finite dictionary and therefore its cost does not depend on $M$, but each
gradient step requires evaluating the full criterion and its derivative over the
training split. With three gradient steps, it is already slower than DR-ME-Dict
at $n\geq 3000$. Increasing the number of gradient steps would further increase
this gap. Therefore, continuous gradient ascent is useful as an alternative for
low-dimensional Euclidean outcomes, but it should not be the default scalable
implementation in the current empirical suite.

\begin{table}[t]
\centering
\caption{Runtime breakdown in seconds at $n=10000$, with $J=2$, $M=80$, and three gradient steps for DR-ME-Grad. Entries are averages over five random seeds.}
\label{tab:runtime_breakdown_n10000}
\begin{tabular}{lrrrr}
\toprule
Method & Nuisance & Learning & Test & Total  \\
\midrule
DR-ME-Dict & 0.0156 & 0.1236 & 0.0688 & 0.2080 \\
DR-ME-Grad & 0.0156 & 0.3033 & 0.0441 & 0.3631 \\
DR-xKTE & 0.0017 & 0.0000 & 0.9964 & 0.9982 \\
\bottomrule
\end{tabular}
\end{table}

\begin{figure}[t]
\centering
\includegraphics[width=0.7\linewidth]{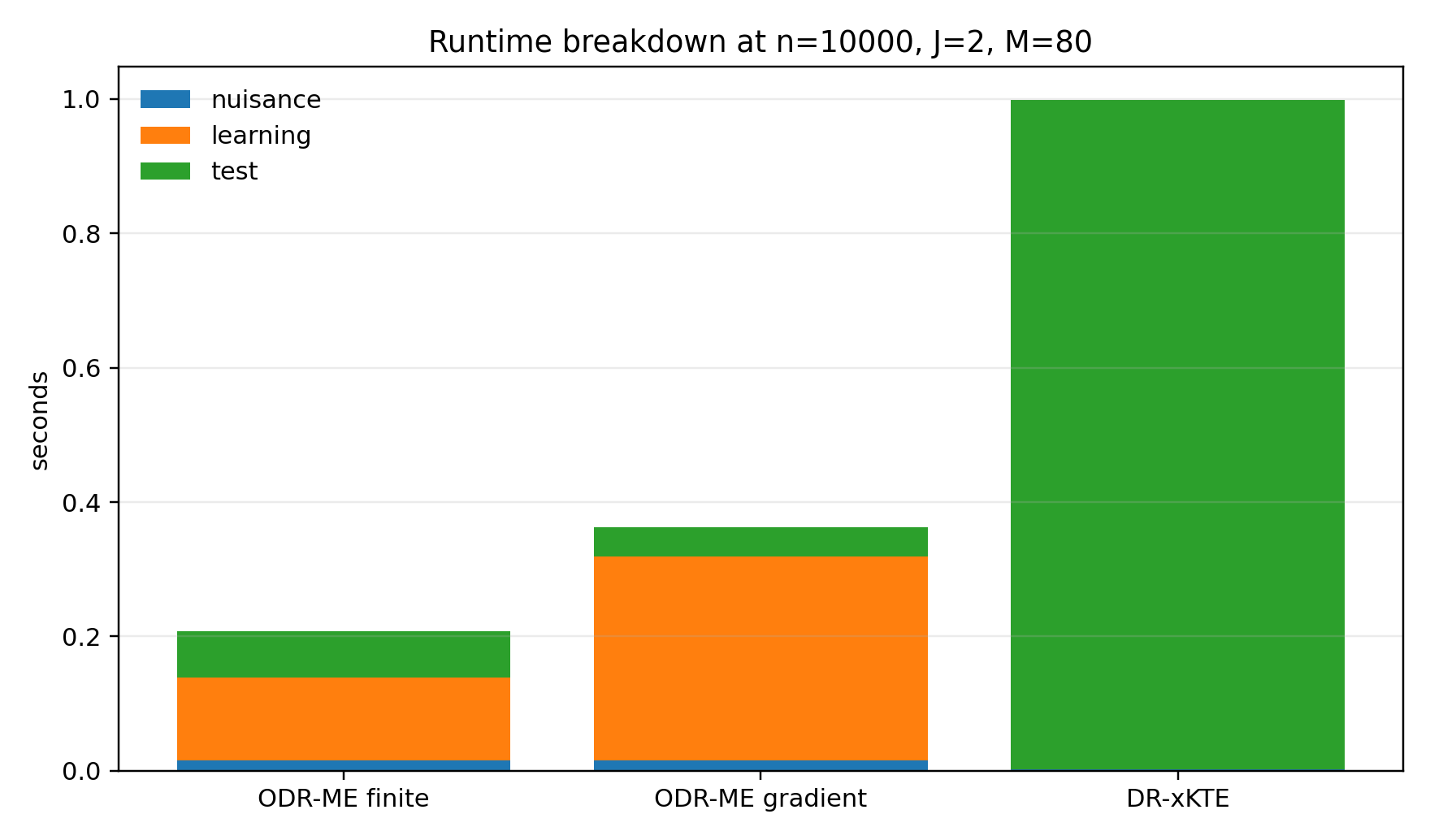}
\caption{Runtime breakdown at \(n=10000\), \(J=2\), and \(M=80\). The finite-dictionary DR-ME implementation keeps the final test low-dimensional and is faster than DR-xKTE, whose cost is dominated by the global kernel test. Gradient-based DR-ME spends most of its time in location learning.}
\label{fig:runtime_breakdown_n10000}
\end{figure}

Figure~\ref{fig:runtime_breakdown_n10000} shows the breakdown at the largest
sample size. The finite-dictionary implementation is the most attractive option:
it is faster than DR-xKTE at $n=10000$ and still returns learned interpretable
locations. DR-xKTE remains a strong global baseline, but its runtime is tied to a
global kernel computation over the evaluation sample. DR-ME-Grad is dominated by
its repeated gradient evaluations and is therefore less attractive for large
$n$.

Overall, this experiment supports a task-dependent implementation choice. For hypothesis testing, as in our main experiments, DR-ME is best implemented with a finite dictionary and small \(J\). This yields a practical low-dimensional test, scales favorably at larger sample sizes, and preserves interpretability through selected outcome locations. Continuous gradient learning is better viewed as an optional Euclidean-location variant: it is natural when the goal is to optimize interpretable locations in structured Euclidean outcome spaces, such as images, but it is not the primary scalable implementation for testing. DR-xKTE remains the main global-kernel baseline for power comparisons, while the runtime results show that finite-location testing can provide a computational advantage in the larger-\(n\) regime.

\begin{table}[t]
\centering
\caption{Total runtime in seconds as the number of locations $J$ varies, with $n=1200$, $M=80$, and three gradient steps for DR-ME-Grad. Results are averaged over five random seeds.}
\label{tab:runtime_total_vs_J}
\begin{tabular}{rrrr}
\toprule
$J$ & DR-ME-Dict & DR-ME-Grad & DR-xKTE  \\
\midrule
1 & 0.0133 & 0.0118 & 0.0043 \\
2 & 0.0158 & 0.0128 & 0.0058 \\
5 & 0.0187 & 0.0108 & 0.0044 \\
10 & 0.0320 & 0.0128 & 0.0046 \\
\bottomrule
\end{tabular}
\end{table}

\begin{table}[t]
\centering
\caption{Total runtime in seconds as the dictionary size $M$ varies, with $n=1200$, $J=2$, and three gradient steps for DR-ME-Grad. Results are averaged over five random seeds.}
\label{tab:runtime_total_vs_M}
\begin{tabular}{rrrr}
\toprule
$M$ & DR-ME-Dict & DR-ME-Grad & DR-xKTE  \\
\midrule
40 & 0.0114 & 0.0109 & 0.0052 \\
80 & 0.0139 & 0.0113 & 0.0039 \\
160 & 0.0184 & 0.0111 & 0.0043 \\
320 & 0.0260 & 0.0106 & 0.0045 \\
\bottomrule
\end{tabular}
\end{table}

\subsection{OCTMNIST mean-matched image-location experiment}
\label{app:octmnist_experiment}

We use OCTMNIST from MedMNIST v2 \citep{medmnistv2}, \citep{kermany2018identifying}. MedMNIST v2 is released under the Creative Commons Attribution 4.0 International license (CC BY 4.0). OCTMNIST is used to construct a semi-synthetic observational
experiment with full image outcomes. The main goal is to study whether a learned
finite-location witness can localize a causal distributional change in image
space, while avoiding the degenerate setup in which the same image serves both
as a covariate and as an outcome. We therefore generate low-dimensional
pre-treatment covariates synthetically, use them to induce confounding and
heterogeneity, and use OCTMNIST images only on the outcome side.

For each unit, we first sample synthetic covariates
\[
X_i\in\mathbb R^{d_X},
\qquad
X_i\sim N(0,I_{d_X}),
\]
with \(d_X=6\) in our implementation. Treatment is assigned observationally by a
clipped logistic propensity,
\[
A_i\mid X_i\sim \mathrm{Bernoulli}\{\pi_0(1\mid X_i)\},
\qquad
\pi_0(1\mid X_i)
=
\operatorname{clip}\!\left\{\sigma(\alpha^\top X_i),\,0.07,\,0.93\right\},
\]
so confounding is driven by the synthetic covariates rather than by the OCT
images themselves.

The outcome construction begins with baseline images
\[
B_i\in[0,1]^{28\times 28}
\]
sampled from the \emph{normal} OCTMNIST class. These images play the role of
untreated anatomy. To define a clinically interpretable treatment pattern, we
construct a fixed fluid-like template
\[
M\in[0,1]^{28\times 28}
\]
from OCTMNIST using the positive part of the smoothed mean difference between
the DME and normal classes, together with a mild central spatial prior. The DME
images are therefore used only to define this template; the causal sample itself
is built from normal baseline images.

Potential outcomes are generated by adding treatment-specific residual images to
the same baseline anatomy:
\[
Y_i(a)=B_i+R_i(a),
\qquad a\in\{0,1\}.
\]
The residuals are designed so that the two treatment arms have the same average
residual signal, but different residual distributions. Let
\[
q_i
=
q_{\min}+(q_{\max}-q_{\min})\,\sigma(\beta^\top X_i),
\]
where \(q_i\in(q_{\min},q_{\max})\) is a covariate-dependent probability of a
severe fluid event. We then sample
\[
S_i\mid X_i\sim \mathrm{Bernoulli}(q_i),
\]
and define mean-one multiplicative amplitude jitters
\[
J_{ia}=\exp(\tau U_{ia}-\tau^2/2),
\qquad
U_{ia}\sim N(0,1).
\]
The two residual laws are
\[
R_i(0)=q_i\theta J_{i0}M+\varepsilon_{i0},
\qquad
R_i(1)=S_i\theta J_{i1}M+\varepsilon_{i1},
\]
with independent Gaussian pixel noise
\[
\varepsilon_{i0},\varepsilon_{i1}\sim N(0,\sigma^2 I).
\]
Since \(\E[J_{ia}]=1\) and \(\E[S_i\mid X_i]=q_i\), the construction satisfies
\[
\E\{R_i(1)-R_i(0)\mid X_i\}=0.
\]
Thus the treatment effect is distributional rather than mean-based: the control
arm carries a mild diffuse fluid pattern, while the treated arm exhibits rare
but much more severe localized deviations.
\begin{figure}
    \centering
    \includegraphics[width=0.9\linewidth]{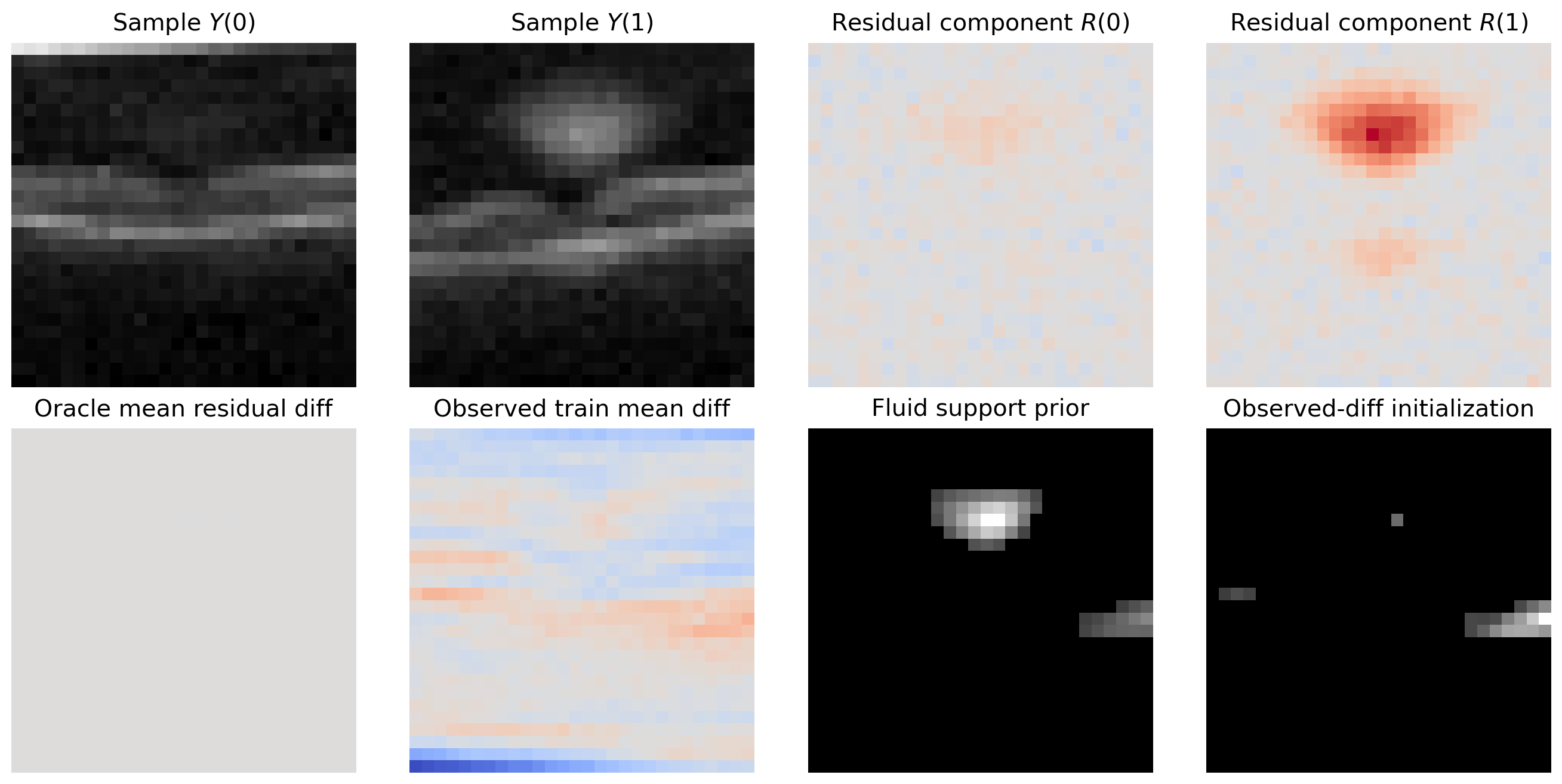}
    \caption{Additional diagnostics for the OCTMNIST mean-matched image-location experiment. Top row: sampled control and treated images \(Y(0)\), \(Y(1)\), together with their residual components \(R(0)\), \(R(1)\). Bottom row: oracle mean residual difference, observed treated-versus-control mean image difference on the training split, the fluid support prior, and the observed-difference initialization. The learned DR-ME image location shown in the main text is optimized from this initialization on the training split and evaluated only on the independent test split. After sample-level recentering, the oracle mean residual difference is zero up to numerical precision, so localization is not driven by a residual mean shift.}
    \label{fig:octmnist_full}
\end{figure}

In finite samples, the equality above holds exactly only in expectation. To make
the oracle mean residual difference visually and numerically null in the
realized sample, we additionally recenter the sampled residuals. Writing
\[
\bar\Delta_R
:=
\frac1n\sum_{i=1}^n \{R_i(1)-R_i(0)\},
\]
we replace
\[
R_i(0)\leftarrow R_i(0)+\frac12\bar\Delta_R,
\qquad
R_i(1)\leftarrow R_i(1)-\frac12\bar\Delta_R.
\]
After this recentering, the empirical oracle residual mean difference is zero up
to numerical precision, while the residual distributions remain different. This
makes the qualitative localization diagnostic sharper: any visually meaningful
learned location cannot be attributed to a first-moment residual contrast.

The observed outcome is
\[
Y_i=A_iY_i(1)+(1-A_i)Y_i(0).
\]
For inference we work directly with the full image outcomes, not with residuals.
We use a Gaussian image kernel
\[
k_Y(v,y)
=
\exp\!\left\{
-\frac{\|v-y\|_2^2}{2d_Y\ell_Y^2}
\right\},
\qquad
d_Y=28\cdot 28,
\]
where the factor \(d_Y\) makes \(\ell_Y\) interpretable as a per-pixel scale.
The propensity is estimated on \(I_\eta\) by logistic regression, and the
nuisance regressions
\[
m_a(x;v)=\E\{k_Y(v,Y)\mid A=a,X=x\}
\]
are estimated on \(I_\eta\) by ridge regression over a fixed nonlinear feature
basis of \(X\).

We learn a single image location \(v\) on the training split \(I_{\tr}\), so
\(J=1\). In this scalar case, the DR-ME location-learning criterion reduces to
the squared standardized doubly robust moment,
\[
v\mapsto
\hat{\mathcal P}_{\tau,\tr}(v)
=
\frac{\bar z_{\tr}^{\dr}(v)^2}
{\widehat{\Var}_{\tr}\{z^{\dr}(v)\}+\tau}.
\]
The final one-location DR-ME statistic is then computed only on the independent
test split \(I_{\te}\).

Because \(v\) lives in a continuous \(28\times 28\) image space, direct
first-order optimization is highly nonconvex and, without further structure,
often returns diffuse or anatomically implausible solutions. We therefore use a
small amount of optimization-side regularization to stabilize the search and to
favor interpretable image locations. First, the optimization is initialized at
the observed treated-versus-control mean image difference on the training split,
\[
v_{\mathrm{init}}
=
\bar Y_{\tr,A=1}-\bar Y_{\tr,A=0}.
\]
Second, each gradient step is smoothed by a Gaussian filter and lightly
anchored toward the initialization, which discourages unstable high-frequency
updates. These operations are used
only to stabilize the continuous optimization problem and to improve
interpretability; they do not use the test split and do not alter the final
split-sample DR-ME statistic once the location has been learned. In the plotted
figures, an additional display threshold is applied only for visualization, so
that faint background structure is suppressed and the dominant localized region
is easier to see.

This experiment should therefore be interpreted as a qualitative localization
diagnostic rather than as a benchmark for unconstrained image optimization. It
shows that, once the average residual contrast is removed, a learned finite
location can still localize the retinal region where the interventional image
laws differ. At the same time, the need for smoothness, support, and sparsity
aids highlights an interesting methodological direction: for structured outcomes
such as images, more principled optimization classes for \(v\) may be preferable
to unconstrained pixel-space search. Possible future directions include sparse
or wavelet-based parameterizations, total-variation or shape-constrained
regularization, learned dictionaries of candidate image locations, and other
structured search classes that preserve the split-sample DR-ME testing logic
while making location learning more stable and more interpretable.

\end{document}